\documentclass[11pt]{article}

\usepackage[margin=1in]{geometry}
\usepackage[utf8]{inputenc}
\usepackage{tgtermes}
\usepackage{newtxtext}
\usepackage{amsmath,mathtools}
\usepackage{amsthm}
\usepackage{newtxmath}
\usepackage{bm}
\usepackage{enumitem}
\usepackage{algorithm}
\usepackage{algpseudocode}
\usepackage{booktabs}
\usepackage{placeins}
\usepackage{float}
\usepackage{graphicx}
\usepackage{natbib}
\usepackage{fancyhdr}

\bibpunct[, ]{(}{)}{,}{a}{}{,}

\usepackage[hidelinks]{hyperref}
\usepackage[capitalize,noabbrev,sort&compress]{cleveref}

\newtheorem{theorem}{Theorem}
\newtheorem{lemma}{Lemma}
\newtheorem{proposition}{Proposition}

\theoremstyle{definition}
\newtheorem{assumption}{Assumption}
\newtheorem{definition}{Definition}
\newtheorem{example}{Example}

\newcommand{\Dir}{\mathrm{Dir}}

\DeclareMathOperator*{\argmax}{arg\,max}
\newcommand{\assref}[2]{Assumption~\ref{#1}\ref{#2}}
\DeclareMathOperator{\diag}{diag}

\newcommand{\papershorttitle}{Evolving Robustness--Exploration Trade-off in Online RL}
\newcommand{\paperauthorheader}{Song, Wang, and Zhou}
\newcommand{\papertitle}{Evolving Robustness--Exploration Trade-off in Online Reinforcement Learning via Quantile Bayesian Risk MDPs}
\newcommand{\authorentry}[2]{%
  \noindent
  \begin{tabular*}{\textwidth}{@{}l@{\extracolsep{\fill}}r@{}}
  \textbf{#1} & \texttt{\small #2}\\[-1pt]
  \multicolumn{2}{@{}l@{}}{\textit{School of Industrial and Systems Engineering}}\\
  \multicolumn{2}{@{}l@{}}{\textit{Georgia Institute of Technology}}\\
  \multicolumn{2}{@{}l@{}}{\textit{Atlanta, GA 30332, USA}}
  \end{tabular*}\par\vspace{8pt}
}
\newcommand{\makearxivtitle}{%
  \vspace*{2.2em}
  \begin{center}
  {\LARGE\bfseries \papertitle\par}
  \end{center}
  \vspace{1.7em}
  \authorentry{Meichen Song}{msong97@gatech.edu}
  \authorentry{Yuhao Wang}{yuhaowang@gatech.edu}
  \authorentry{Enlu Zhou}{enlu.zhou@isye.gatech.edu}
  \vspace{1.4em}
}

\begin{document}
\makearxivtitle

\begin{abstract}
In online reinforcement learning, data scarcity creates epistemic uncertainty that makes robustness important early in learning, whereas sufficient exploration is needed to learn the true-environment optimal policy.
We study this time-varying robustness--exploration trade-off through a quantile Bayesian risk-aware Markov decision process (BR-MDP), in which the quantile level controls how posterior uncertainty enters the Bellman backup.
We characterize this control through an asymptotic normality result for the difference between the quantile BR-MDP value and the value in the true environment. The result implies that upper/lower-tail quantiles induce optimism/pessimism towards epistemic uncertainty, and the magnitude of the optimism/pessimism decreases as data accumulate.
Building on this characterization, we propose an online Bayesian risk-aware algorithm with an adaptive quantile schedule that emphasizes robustness early and gradually encourages exploration of less-visited state--action pairs.
We establish sublinear Bayesian regret bounds with respect to both the true optimal value and the optimal BR-MDP robust value. Numerical experiments demonstrate strong performance in both exploration-demanding and exploration-costly environments.
\end{abstract}


\section{Introduction}
In online reinforcement learning (RL), an agent sequentially interacts with an unknown environment and uses collected data to estimate the unknown environment and update the policy used in subsequent interactions. Thus, each action affects both the immediate reward and the information available for future decisions. Limited data lead to epistemic uncertainty in estimating environment parameters \citep{DerKiureghianDitlevsen2009}. 
This uncertainty is central to the consideration of exploration--exploitation trade-off \citep{jaksch2010near,osband2013psrl,azar2017minimax,ma2025eubrl}: exploitation chooses actions that appear optimal under current estimates to pursue high estimated cumulative reward, whereas exploration collects information to reduce epistemic uncertainty. 

Although regions with higher uncertainty offer a greater incentive to explore, acting in such regions is also risky because the unreliable estimates induce estimated optimal policies that can perform poorly in the true environment. This risk is most pronounced early in learning, when data are scarce. 
It is particularly salient in high-stakes settings with limited interaction budgets \citep{dulacarnold2021realworld}, such as public health intervention problems \citep{liang2024bayesian} and inventory or service systems with costly trial-and-error decisions. Therefore, beyond the classical exploration--exploitation consideration, the robustness of the policy used in interactions (which we refer to as the interaction policy throughout the rest of the paper) is also a primary concern early in learning, as it hedges against the risk induced by acting under epistemic uncertainty.
As learning progresses and more data are collected, epistemic uncertainty decreases and such unreliable estimates become less likely. Consequently, the need for robustness becomes less essential, whereas exploring less-visited state-action pairs becomes more urgent for learning the optimal policy in the true environment.
This yields an intrinsic and time-varying robustness--exploration trade-off: robustness is valuable early in learning, but maintaining a fixed conservative attitude can later hinder the exploration needed to learn the true-environment optimal policy.
Thus, the interaction policy should adapt its treatment of epistemic uncertainty over time, being more robust when data are scarce and becoming less conservative as information accumulates.

To account for epistemic uncertainty, a widely used approach is robust and distributionally robust MDPs and reinforcement learning, which optimizes worst-case performance over uncertainty or ambiguity sets \citep{elghaoui2005rmdp,iyengar2005robust,xu2010drmdp,wiesemann2013robust}. However, much of this literature focuses on offline settings, where the historical dataset remains fixed throughout learning or data are accessed through a generative model or simulator \citep{panaganti2022generative,zhou2021droffline,panaganti2022offline,blanchet2023double}. These results show how robust policies can be learned when data are exogenously available, but do not address how an interaction policy should jointly manage robustness and exploration as data accumulate.

Recent online robust RL works are closer to the present setting \citep{wang2021onlinerobust,badrinath2021rlspi,dong2022onlinerobust,lu2024interactive,wang2023brql,ghosh2025orvit,wang2025online}. 
Some of these studies rely on an {external exploratory behavior policy} to interact with the true environment \citep{wang2021onlinerobust,badrinath2021rlspi,wang2023brql}. However, such a behavior policy does not account for robustness during interaction. 
Another line of work encourages exploration by adding explicit visit-count-dependent bonuses to robust Bellman backups before selecting the interaction policy \citep{dong2022onlinerobust,lu2024interactive,ghosh2025orvit,wang2025online}. These bonuses take larger values early in learning, so exploration can have a greater effect on action selection than the robust value estimate. These methods therefore do not directly ensure that the interaction policy is robust. Moreover, these methods aim to learn an {optimal robust policy}, whereas our goal is to learn the optimal policy in the true environment while ensuring that the interaction policy is robust, especially early in learning. This leaves open how the {interaction policy should adapt its risk attitude over time}---being more robust when data are scarce, but becoming less conservative as more data are collected so that the agent can explore enough to learn the optimal policy in the true environment.

The Bayesian risk-aware Markov Decision Process (BR-MDP) provides a starting point to this question. It models unknown parameters in the true environment through a posterior distribution and imposes risk measures on future rewards at each transition step, resulting in a data-adaptive stochastic model \citep{wu2018bro,lin2022brmdp,wang2023brql,lin2025approximate}. 
Recently, \citet{wang2025online} incorporated lower-tail CVaR as the risk measure in the BR-MDP, so that solving the resulting BR-MDP yields a risk-averse policy. Based on this formulation, they proposed an online Bayesian risk-averse algorithm to learn an optimal robust policy under the Bayesian risk criterion. However, its regret analysis also relies on an additional exploration bonus that is large when visit counts are small. As a result, the interaction policy may not have the desired robustness to epistemic uncertainty early in learning.

To rigorously characterize the robustness-exploration trade-off in online RL, we first study the BR-MDP formulation with a fixed $\alpha$-quantile as the risk measure, referred to as $\alpha$-quantile BR-MDP. Our theoretical analysis shows how the $\alpha$-quantile BR-MDP explicitly controls robustness to epistemic uncertainty while maintaining exploration. 
The optimal policy of the $\alpha$-quantile BR-MDP can be more robust or more exploratory, depending on the quantile level $\alpha$. Therefore, we propose adaptive quantile BR-MDP (AQ-BRMDP), which uses an adaptive quantile schedule to respond to the evolving robustness--exploration trade-off in online RL. Early in learning, the schedule sets the quantile level to emphasize lower-tail evaluations of the cumulative future rewards, yielding a more robust policy. As more data are collected and the posterior concentrates, the schedule gradually increases the quantile level, especially for less-visited state-action pairs that remain important for learning the optimal policy in the true environment, thereby encouraging exploration toward those pairs. In the discounted infinite-horizon setting, we implement this idea through pseudo-episode construction, as in \citet{xu2024continuing}, which partitions the interaction process into intervals with random lengths according to the discount factor. Specifically, at the beginning of each pseudo-episode, we update the posterior belief and quantile schedule, then solve the corresponding $\alpha$-quantile BR-MDP, and finally execute the resulting policy until the next update. This yields an implementable online procedure that adapts the treatment of epistemic uncertainty to the stage of learning and the collected data.

We summarized the main contributions of this paper as follows.

\begin{enumerate}
    \item We formulate the $\alpha$-quantile BR-MDP and show that the quantile level can explicitly control the trade-off between robustness and exploration.
    We characterize this trade-off through an analytical result that 
    the difference between the value function of the $\alpha$-quantile BR-MDP and the original value function is asymptotically normal.
    The magnitude of this mean increases as the quantile level moves farther towards the lower tail, inducing more robust policy, or farther into the upper tail, inducing more exploratory policy. It decreases as more data are collected, at a rate of $O(\frac{1}{\sqrt{N}})$,
    where $N$ is the total number of data points used to estimate the posterior distribution.

    \item We utilize the property above to design an online Bayesian risk-aware algorithm, AQ-BRMDP, that combines adaptive quantile scheduling with pseudo-episodic posterior updates in the discounted infinite-horizon setting.
    The quantile schedule is designed to depend on both visit counts across state-action pairs and the stage of learning, so AQ-BRMDP can adapt to the evolving robustness--exploration trade-off in online RL settings.

    \item Theoretically, we establish that AQ-BRMDP has Bayesian regret bounds of order $\widetilde{O}(\sqrt{T})$, where $T$ is the total number of interactions, with respect to two benchmarks: the optimal value in the true environment, and the  optimal robust value. Numerically, we evaluate the performance of AQ-BRMDP in two environments: one that requires sustained exploration, and the other in which exploratory mistakes are costly. We also implement an extension of the proposed algorithm to continuous-state spaces and evaluate its empirical performance in a continuous-state environment.
\end{enumerate}

\subsection{Related Work}\label{sec:related_work}

Classical online reinforcement learning studies the exploration--exploitation trade-off for efficiently learning the optimal policy in the true environment. A broad class of optimism-based methods addresses this trade-off through optimistic initialization, optimistic models, confidence sets, or explicit exploration bonuses \citep{kearns2002near,brafman2002rmax,strehl2006pac,strehl2008analysis,jaksch2010near,bartlett2009regal,azar2017minimax,he2021nearly,jin2018qlearning,dong2019qlearning}. Posterior-sampling methods instead introduce exploration through the randomness of models sampled from a Bayesian posterior and are typically analyzed through Bayesian regret \citep{osband2013psrl,abbasi2015lazypsrl,russo2018tutorial,xu2024continuing}. Related posterior-inference and posterior-quantile methods also use posterior information to direct exploration toward actions whose values remain uncertain \citep{tiapkin2022dirichlet,tarbouriech2023probinf,ma2025eubrl}. This literature primarily uses uncertainty to improve learning of the true-environment optimal policy; it does not explicitly control the interaction risk of exploratory actions when epistemic uncertainty is high. 

Robust and distributionally robust reinforcement learning originate from robust MDP formulations, where the objective is to optimize performance under worst-case models or over {ambiguity sets} \citep{elghaoui2005rmdp,iyengar2005robust,xu2010drmdp,wiesemann2013robust}. This line has developed into learning algorithms for generative-model and offline-data settings \citep{panaganti2022generative,zhou2021droffline,panaganti2022offline,blanchet2023double}, and more recently into online robust RL with interactive data collection \citep{wang2021onlinerobust,badrinath2021rlspi,dong2022onlinerobust,lu2024interactive,ghosh2025orvit}. 
These works primarily account for model misspecification and aim to learn robust-optimal policies, but they do not directly address how the {behavior policy} should balance robustness with the exploration needed to learn the optimal policy in the true environment.

Our work bridges these two lines of work: as robust RL, it accounts for epistemic uncertainty to improve the robustness of the interaction policy; as an online RL method designed for efficient exploration, it ultimately aims to learn the optimal policy in the true environment. The key idea of our work is to use an $\alpha$-quantile Bayesian risk criterion to interpolate between robustness early in learning and exploration needed later to learn the true-environment optimal policy, rather than treating robustness and exploration as separate mechanisms. 

{Our paper is built on Bayesian risk optimization and the BRMDP formulation. Bayesian risk optimization was introduced by \citet{zhouxie2015bro} as a flexible alternative to worst-case robustness in static (one-stage) stochastic optimization, and its statistical properties were established by \citet{wu2018bro}. This perspective was extended to sequential (multi-stage) decision making through BRMDP, spanning finite- and infinite-horizon formulations as well as offline and online learning settings \citep{lin2022brmdp,lin2025approximate,wang2023brql,wang2025online}. 
On a related note, a broader Bayesian RL literature studies how posterior uncertainty can be represented and used in sequential decision making \citep{ghavamzadeh2015survey}. In particular, Bayes-adaptive formulations augment the decision state with posterior beliefs and optimize Bayes-adaptive expected return \citep{duff2002bamdp,poupart2006analytic,ross2007bayesadaptive}. }

Finally, our paper should be distinguished from safe exploration. 
In this literature, risk is intrinsic to the true environment, i.e., the possibility that certain actions lead to inherently undesirable outcomes \citep{garcelon2020conservative,yamagata2024safe}. In contrast, the risk we consider does not arise from the environment itself, but from epistemic uncertainty: limited data can cause the agent to act on unreliable estimates of the environment. This risk diminishes as more data are collected.

\section{Bayesian Risk-Aware MDPs}
\label{sec:formulation}

\subsection{$\alpha$-quantile Bayesian Risk MDPs}
\paragraph{Unknown True Environment.}
We consider an infinite-horizon discounted Markov decision process (MDP)
$\mathcal{M}=(\mathcal S,\mathcal A,\gamma, r,P^c)$, where $\mathcal S$ and $\mathcal A$ are finite state and action spaces with $|\mathcal S|=S$ and $|\mathcal A|=A$, $\gamma\in(0,1)$ is the discount factor, $r:\mathcal S\times\mathcal A\to [0,1]$ is the reward function, and $P^c(s'\mid s,a)$ is the true but unknown transition kernel. For notational simplicity, we write $P^c_{s,a}(\cdot):=P^c(\cdot\mid s,a)$. 

Let $\Pi$ be the set of Markovian policies $\pi=\{\pi_t\}_{t\ge 0}$, where each $\pi_t(\cdot\mid s)$ is a probability distribution over $\mathcal A$ for every $s\in\mathcal S$. Given an initial state $s_0=s$, the objective is to maximize the expected total discounted reward
\begin{align}
\label{eq:true-objective}
    \sup_{\pi\in\Pi}
\mathbb E^\pi\left[\sum_{t=0}^\infty \gamma^t r(s_t,a_t) \,\middle|\, s_0=s\right],\quad \forall s\in\mathcal S,
\end{align}
where $a_t\sim \pi_t(\cdot\mid s_t)$ and $s_{t+1}\sim P^c(\cdot\mid s_t,a_t)$. {For a policy $\pi\in\Pi$, the objective function under a fixed policy $\pi$ in \eqref{eq:true-objective} defines the value function of $\pi$ at state $s$
\begin{equation}
V^\pi(s)
:=
\mathbb E^\pi\left[\sum_{t=0}^\infty \gamma^t r(s_t,a_t) \,\middle|\, s_0=s\right],
\qquad \forall s\in\mathcal S.
\end{equation}
Accordingly, the optimal value function is $V^*(s):=\sup_{\pi\in\Pi}V^\pi(s), $ $ \forall s\in\mathcal S,$ and a policy $\pi^*$ is optimal if $V^{\pi^*}(s)=V^*(s)$ for all $s\in\mathcal S$.}

{For tabular MDPs, there exists a stationary deterministic optimal policy \citep{puterman1994markov}. Thus, without loss of optimality, 
we henceforth restrict attention to stationary deterministic policies of the form $\pi:\mathcal S\to\mathcal A$.} For any such policy, the value function satisfies the Bellman equation, $\forall s\in\mathcal S,$
\begin{equation}
\label{eq:V-true-pi}
V^\pi(s)
=
r\bigl(s,\pi(s)\bigr)
+
\gamma\,\mathbb E_{P^c_{s,\pi(s)}}\!\left[V^\pi(s')\right],
\end{equation}
where $\mathbb E_P[V^\pi(s')] = \sum_{s'\in\mathcal S} P(s')V^\pi(s') = P^\top V^\pi$ is the expected value (EV) of the successor-state value function with respect to the distribution $P$. The optimal value function $V^*$ satisfies the Bellman optimality equation, $\forall s\in\mathcal S,$
\begin{equation}
\label{eq:V-true-opt}
V^*(s)
=
\max_{a\in\mathcal A}\left\{
r(s,a)
+
\gamma\,\mathbb E_{P^c_{s,a}}\!\left[V^*(s')\right]
\right\}.
\end{equation}
In practice, the transition kernel $P^c$ is unknown and must be learned from data, which introduces epistemic uncertainty. We assume throughout that the reward function $r$ is known, but the proposed approach can be extended to settings with unknown rewards.

\paragraph{Bayesian Modeling of the Transition Kernel.}
To model epistemic uncertainty in the unknown transition kernel $P^c$, we adopt a Bayesian approach. Specifically, we model each unknown transition vector $P^c_{s,a}$ by a Bayesian random vector $P_{s,a}$ and place independent Dirichlet conjugate priors on $\{P_{s,a}\}_{(s,a)\in\mathcal S\times\mathcal A}$. After observing a trajectory, the posterior of each $P_{s,a}$ remains Dirichlet by conjugacy, with a parameter vector determined by the prior and the transition counts along the observed trajectory. Details of the Dirichlet definition and posterior update are provided in Appendix~\ref{app:dirichlet-update}.

When analyzing the BR-MDP under the current posterior, we write $P_{s,a}\sim \Dir(\phi(s,a)),$ 
where $\phi(s,a)=(\phi(s,a,s'))_{s'\in\mathcal S}$ is the current posterior parameter vector. We also write $\phi=\{\phi(s,a)\}_{(s,a)\in\mathcal S\times\mathcal A}$ for the collection of current posterior parameters. 
The dependence of $\phi$ on the prior and the observed trajectory is suppressed for notational simplicity in this section.
\paragraph{Bayesian Risk Value Function.}
To account for epistemic uncertainty represented by the current posterior over transition kernels, we evaluate policies by applying a risk measure with respect to this posterior distribution. 

For a policy $\pi$, the value function of BR-MDP with posterior parameter $\phi$ is defined through the Bellman equation \citep{wang2023brql,wang2025online}, $\forall s\in\mathcal S,$
\begin{align}
V^\pi_{\phi,\alpha}(s)=
r\bigl(s,\pi(s)\bigr)+
\gamma\,\rho^\alpha_{\phi(s,\pi(s))}
\!\bigl(P^\top V^\pi_{\phi,\alpha}\bigr),
\qquad 
\label{eq:VR-pi}
\end{align}
where $\rho_{\phi(s,\pi(s))}^\alpha(\cdot)$ is a risk measure at level $\alpha\in(0,1)$ with respect to $\Dir(\phi(s,\pi(s)))$. Notably, the Bayesian risk value function $V^\pi_{\phi,\alpha}$ depends on the posterior parameter $\phi$. As new data are observed, $\phi$ is updated via \eqref{eq:update-of-phi}, which changes the posterior distribution of the transition kernel and the associated risk evaluation. Consequently, both the value function and the induced optimal policy evolve over time, reflecting the agent's updated belief about the environment.

As shown by \citet{wang2023brql,wang2025online}, the value function of BR-MDP also admits an equivalent nested representation:
\begin{align}
V^\pi_{\phi,\alpha}(s_0)
&=
r(s_0,a_0)
+\gamma \rho^\alpha_{\phi(s_0,a_0)}
\!\Bigl(
\mathbb E_{P^1}\Bigl[
r(s_1,a_1)
+
\gamma \rho^\alpha_{\phi(s_1,a_1)}
\!\Bigl(
\mathbb E_{P^2}
\bigl[r(s_2,a_2)+\cdots\bigr]
\Bigr)
\Bigr]
\Bigr).
\label{eq:nested V}
\end{align}
where, independently across stages, $P^t\sim \Dir(\phi(s_{t-1},a_{t-1}))$, $s_t\sim P^t$ for $t\ge 1$, and $a_t=\pi(s_t)$ for $t\ge 0$.
Moreover, the optimal value function can be computed via the following Bayesian risk Bellman optimality equation: $\forall s\in\mathcal S$,
\begin{align}
V^*_{\phi,\alpha}(s)
= \max_{a\in\mathcal{A}} \Bigl\{
r\bigl(s,a\bigr)
+
\gamma\,\rho^\alpha_{\phi(s,a)}
\!\bigl(P^\top V^*_{\phi,\alpha}\bigr)\Bigr\}.
\label{eq:VR-opt}
\end{align}
By solving \eqref{eq:VR-opt}, we can obtain an optimal policy, $\pi^*_{\phi,\alpha}$, that is deterministic and stationary.

We also define the corresponding Bellman operator under the policy $\pi$ and the optimal Bellman operator as, $\forall s\in\mathcal S,$ 
\begin{align}
(\mathcal T^\pi_{\phi,\alpha}V)(s)
&:=
r\bigl(s,\pi(s)\bigr)
+
\gamma\,\rho^\alpha_{\phi(s,\pi(s))}
\!\bigl(P^\top V\bigr),
\label{eq:TR-pi}\\
(\mathcal T^*_{\phi,\alpha}V)(s)
&:=
\max_{a\in\mathcal A}
\Bigl\{
r(s,a)
+
\gamma\,\rho^\alpha_{\phi(s,a)}
\!\bigl(P^\top V\bigr)
\Bigr\}.
\label{eq:TR}
\end{align}
Under structural conditions on the risk measure discussed by \citet{wang2023brql}, $\mathcal T^\pi_{\phi,\alpha}$ and $\mathcal T^*_{\phi,\alpha}$ are $\gamma$-contractions under $\|\cdot\|_\infty$ and therefore admit unique fixed points. 
\paragraph{$\alpha$-quantile BR-MDP.} We specialize the risk measure $\rho^\alpha(\cdot)$ to the $\alpha$-quantile, thereby obtaining an $\alpha$-quantile BR-MDP. We refer to $V^\pi_{\phi,\alpha}$ as the corresponding $\alpha$-quantile BR-MDP value function. Specifically, for a random variable $X$ and $\alpha\in(0,1)$, the $\alpha$-quantile of $X$ is defined as
$
    \rho^\alpha(X):=\inf\{z\in\mathbb R:\mathbb P(X\le z)\ge \alpha\}.
$
Under this choice, the quantile level $\alpha$ determines which part of the posterior distribution of $P_{s,a}^\top V$ is emphasized in the Bellman backup, thereby encoding a risk attitude toward epistemic uncertainty. 

To illustrate the effect of the quantile level $\alpha$, we define the one-step adjustment induced by 
replacing the true EV of successor states with its posterior $\alpha$-quantile as follows:
$$
b^{\alpha}_{\phi(s,a)}(V)
:=
\rho^\alpha_{\phi(s,a)}\!\bigl(P^\top V\bigr)
-\mathbb E_{P^c_{s,a}}[V(s')].
$$
The optimal Bellman operator can then be written as
$$(\mathcal T^*_{\phi,\alpha}V)(s)
=
\max_{a\in\mathcal A}
\Bigl\{
r(s,a)
+\gamma\,\mathbb E_{P^c_{s,a}}[V(s')]+
\gamma\,b^{\alpha}_{\phi(s,a)}(V)
\Bigr\},
\qquad \forall s\in\mathcal S.$$

This representation decomposes the Bellman backup into $r(s,a)+\gamma\,\mathbb E_{P^c_{s,a}}[V(s')]$, the backup under the true transition kernel for $(s,a)$, and the one-step adjustment $b^{\alpha}_{\phi(s,a)}(V)$ that captures both the risk attitude and the belief about  epistemic uncertainty in the transition kernel. 

We next briefly discuss how the risk level $\alpha$ and the posterior parameter $\phi$ together affect the optimal policy from a \textbf{Bayesian perspective}.
Conditional on the current posterior parameter $\phi(s,a)$, the unknown true transition vector $P^c_{s,a}$ is viewed as a draw from the posterior $\Dir(\phi(s,a))$ and $b^{\alpha}_{\phi(s,a)}(V)
=
\rho^\alpha_{\phi(s,a)}\!\bigl(P^\top V\bigr)
-(P^c_{s,a})^\top V.$ Therefore, by the definition of the $\alpha$-quantile, the posterior probability of $b^{\alpha}_{\phi(s,a)}(V) \ge 0$ is at least $\alpha$. 
For large $\alpha$, the Bellman backup in \eqref{eq:VR-opt} emphasizes the upper tail of the posterior distribution of $P_{s,a}^\top V$. In this case, $b^{\alpha}_{\phi(s,a)}(V)$ is positive with high posterior probability $\alpha$ and acts as an implicit bonus for exploration, reflecting an optimistic attitude toward epistemic uncertainty about the transition kernel.
Moreover, for such large $\alpha$, the magnitude of $b^{\alpha}_{\phi(s,a)}(V)$ increases with larger epistemic uncertainty (i.e., a more dispersed posterior distribution of $P_{s,a}$). Therefore, $b^{\alpha}_{\phi(s,a)}(V)$ acts as an exploration incentive for state-action pairs $(s,a)$ with high epistemic uncertainty about $P_{s,a}$.
Conversely, when $\alpha$ is small, $b^{\alpha}_{\phi(s,a)}(V)$ is negative with high posterior probability $1-\alpha$ and its magnitude increases with larger epistemic uncertainty. As a result, the one-step adjustment $b^{\alpha}_{\phi(s,a)}(V)$ acts as an implicit penalty for epistemic uncertainty about $P_{s,a}$ and thereby yields a policy that is more robust to epistemic uncertainty. 

Overall, when $b^{\alpha}_{\phi(s,a)}(V)$ is positive, it plays a role analogous to an optimistic exploration bonus in online RL methods such as UCBVI-$\gamma$ \citep{he2021nearly}, but it is induced by the posterior quantile rather than added as an explicit term derived by concentration inequalities. This adjustment can also be negative when the quantile level is chosen in the lower tail, therefore acting as a pessimistic penalty.

\subsection{Asymptotic Analysis of $\alpha$-quantile BR-MDP}
In the previous subsection, we introduced $\alpha$-quantile BR-MDP to account for epistemic uncertainty in the unknown transition kernel and discussed how the quantile level $\alpha$ reflects a risk attitude toward this uncertainty.
In this subsection, we further refine this interpretation by studying the asymptotic behavior of the gap between the $\alpha$-quantile BR-MDP value function $V^\pi_{\phi,\alpha}$ and the true value function $V^\pi$ as the Dirichlet posterior concentrates.
Let $P_\pi^c:=\bigl(P^c_{s,\pi(s)}\bigr)_{s\in\mathcal S}$ denote the true transition matrix induced by policy $\pi$, and let $\diag(x)$ denote the diagonal matrix with diagonal entries given by the entries of $x$.
We impose the following regularity conditions on the data-collection scheme.

\begin{assumption}
\label{ass:quantile_wc}
Fix a stationary deterministic policy $\pi$.
\begin{enumerate}[label=(\roman*)]
    \item 
    Let $N$ denote the total number of observed transitions. The data are collected along an on-policy trajectory $\{s_0,a_0,s_1,\ldots,s_{N-1},a_{N-1},s_N\}$, where $a_i=\pi(s_i)$ and $s_{i+1}\sim P^c_{s_i,a_i}$ for $i=0,\ldots,N-1$. Define the visit count of each state-action pair by $N(s,a):=\sum_{i=0}^{N-1}\mathbb I\{s_i=s,a_i=a\}$. For each $s\in\mathcal S$, there exists a constant $\bar n_s\in(0,1)$ such that
    $$
    \frac{N(s,\pi(s))}{N}\xrightarrow{\mathrm{a.s.}}\bar n_s,
    \qquad
    \sum_{s\in\mathcal S}\bar n_s=1.
    $$
    \label{ass:quantile_wc_1}

    \item For each state-action pair $(s,a)\in\mathcal S\times\mathcal A$, the prior on $P_{s,a}$ is a uniform
    Dirichlet prior with $\phi_0(s,a,s')=1$, $\forall s'\in\mathcal S$.\label{ass:quantile_wc_3}
\end{enumerate}
\end{assumption}
\assref{ass:quantile_wc}{ass:quantile_wc_1} is a trajectory-based sampling condition. It is weaker than the state-action-wise independent sampling condition used in \citet{wang2025online}, which assumes independent transition samples within each state-action pair and mutual independence across different state-action pairs. Here, the data are collected along a single on-policy trajectory, and the induced temporal dependence is handled in the proof by a martingale central limit theorem (CLT). The ratio condition ensures that each on-policy state-action pair $(s,\pi(s))$ receives a nonvanishing fraction of the total number of observations; it holds, for instance, when the Markov chain induced by $P_\pi^c$ is ergodic with a strictly positive stationary mass on every state.
\assref{ass:quantile_wc}{ass:quantile_wc_3} is used only for notational simplicity. 
More generally, any fixed Dirichlet prior with strictly positive parameters contributes only $O(1)$ pseudo-counts and therefore does not affect the limiting behavior.

\begin{theorem}[Asymptotic normality for $\alpha$-quantile BR-MDP]
\label{thm:quantile_weak}
Fix $\alpha\in(0,1)$ and let $z_\alpha:=\Phi^{-1}(\alpha)$, where $\Phi$ is the cdf of the
standard normal distribution. Under Assumption~\ref{ass:quantile_wc},
$$
    \sqrt N\,(I-\gamma P_\pi^c)\bigl(V^\pi_{\phi,\alpha}-V^\pi\bigr)
\Rightarrow
\mathcal N\!\Bigl(\gamma\lambda_\pi,\ \operatorname{diag}\bigl((\gamma\sigma_\pi)^2\bigr)\Bigr),
$$
where, for each $s\in\mathcal S$, $\sigma_\pi^2(s)
:=
\frac{1}{\bar n_s}\,
\operatorname{Var}_{s'\sim P^c_{s,\pi(s)}}\!\bigl[V^\pi(s')\bigr]$, and $\lambda_\pi(s):=z_\alpha\,\sigma_\pi(s)$.
\end{theorem}

The proof of Theorem~\ref{thm:quantile_weak} is deferred to Appendix~\ref{sec:Proof of Weak Convergence}. In particular, Theorem~\ref{thm:quantile_weak} yields the following expansion of the $\alpha$-quantile BR-MDP value function:
\begin{align}
V^\pi_{\phi,\alpha}
&=
V^\pi
+
(I-\gamma P_\pi^c)^{-1}
\times
\left(
\frac{\gamma\lambda_\pi}{\sqrt N}
+
\frac{\gamma\,\operatorname{diag}(\sigma_\pi)\,Z}{\sqrt N}
\right)
+
o_p(N^{-1/2}),
\label{eq:quantile_expansion}
\end{align}
where $\sigma_\pi:=\bigl(\sigma_\pi(s)\bigr)_{s\in\mathcal S}$, $o_p(\cdot)$ denotes little-$o$ notation in probability with respect to the randomness of the observed data, and
$Z\sim\mathcal N(0,I_{S})$ is a standard multivariate normal vector. 
For fixed $\alpha\in(0,1)$, \eqref{eq:quantile_expansion} yields a stochastic expansion of $V^\pi_{\phi,\alpha}$ around $V^\pi$, whose leading term decomposes into a deterministic bias term of order $N^{-1/2}$ and a Gaussian fluctuation term, with remainder $o_p(N^{-1/2})$.
More specifically, before Bellman propagation, i.e., before multiplying by
$(I-\gamma P_\pi^c)^{-1}$ in \eqref{eq:quantile_expansion}, the asymptotic mean of the one-step adjustment $b^{\alpha}_{\phi(s,\pi(s))}(V^\pi)$ is given by
\begin{align}
\label{eq:asy mean}
    \frac{\lambda_{\pi}(s)}{\sqrt{N}}
&=
z_\alpha\sqrt{\frac{\operatorname{Var}_{s'\sim P^c_{s,\pi(s)}}\!\bigl[V^\pi(s')\bigr]}{\bar n_s N}}.
\end{align}
The magnitude of \eqref{eq:asy mean} is larger for states that are sampled less frequently (i.e. small $\bar n_s$) or for which the next-state value $V^\pi(s')$ is more dispersed under the true transition kernel (i.e. large $\operatorname{Var}_{s'\sim P^c_{s,\pi(s)}}[V^\pi(s')]$).
The sign of \eqref{eq:asy mean} is determined by $z_\alpha$: when $\alpha<0.5$, the asymptotic mean of the one-step adjustment is negative, so it discourages visiting less frequently visited states on average, reflecting a conservative attitude toward epistemic uncertainty and inducing an underestimation of $V^\pi_{\phi,\alpha}$ after Bellman propagation; 
{when $\alpha=0.5$, $z_{0.5}=0$, so the deterministic bias term in the asymptotic expansion vanishes. This makes the $0.5$-quantile BR-MDP asymptotically risk-neutral, similar to applying the posterior expectation in \eqref{eq:VR-pi} under the normal approximation;}
when $\alpha>0.5$, $z_\alpha$ becomes positive, so the one-step adjustment provides a bonus to visit less frequently visited states on average, encouraging exploration, reflecting an optimistic attitude and inducing an overestimation due to such exploration incentives. 
Thus, Theorem~\ref{thm:quantile_weak} shows how the quantile level $\alpha$ of the quantile BR-MDP accounts for epistemic uncertainty in the asymptotic regime.

The matrix $(I-\gamma P_\pi^c)^{-1}$ then propagates this one-step adjustment through state transitions. For fixed $N$, the deterministic component in \eqref{eq:quantile_expansion}, $(I-\gamma P_\pi^c)^{-1}
{\gamma\lambda_\pi}/{\sqrt N}$, represents the mean shift of $V^\pi_{\phi,\alpha}$ away from $V^\pi$ induced by the $\alpha$-quantile in the Bellman backup. Since $\lambda_\pi = z_\alpha\sigma_\pi\to\pm\infty$ as $\alpha\uparrow 1$ or $\alpha\downarrow 0$, the magnitude of the mean shift increases as the quantile level moves farther into the upper tail for exploration or farther into the lower tail for robustness. On the other hand, for every fixed $\alpha\in(0,1)$, since
$\|(I-\gamma P_\pi^c)^{-1}\|_\infty\le (1-\gamma)^{-1}$, the deterministic component $(I-\gamma P_\pi^c)^{-1}
{\gamma\lambda_\pi}/{\sqrt N}$ is bounded
in sup-norm by
$\frac{\gamma}{1-\gamma}\,|z_\alpha|\,\|\sigma_\pi\|_\infty/\sqrt N$.
Hence, the mean shift induced by a fixed quantile level $\alpha$ shrinks automatically as more data are collected, which implies that the trade-off between robustness and exploration diminishes with more data.
As a result, for every fixed $\alpha\in(0,1)$, the difference
$V^\pi_{\phi,\alpha}-V^\pi$ is of order $O_p(N^{-1/2})$. 

Finally, the asymptotic covariance $\sigma^2_\pi$ in Theorem~\ref{thm:quantile_weak} does not depend on $\alpha$. Thus, $\alpha$ only controls the direction and magnitude of the mean shift, whereas the intrinsic variability is determined by the true transition dynamics and the policy used.

\subsection{Finite-sample Robustness of $\alpha$-quantile BR-MDP}
\label{sec:robustness lower bounds}

In the previous subsection, we introduced $\alpha$-quantile BR-MDP to account for epistemic uncertainty in the unknown transition kernel and discussed how the quantile level $\alpha$ reflects a risk attitude toward this uncertainty.
In this subsection, we further provide a robustness interpretation of $\alpha$-quantile BR-MDP.

With a slight abuse of notation, let $P=\bigl(P_{s,a}\bigr)_{(s,a)\in\mathcal S \times\mathcal{A}}$ be the random transition kernel obtained by independently drawing $P_{s,a}\sim \Dir(\phi(s,a))$ for each state-action pair. 
Let \(\mathbb P_\phi\) denote the posterior probability measure over the transition kernel \(P\) and $\rho^\alpha_{\phi}(\cdot)$ denote the $\alpha$-quantile with respect to the transition kernel $P$.
Given a realized transition kernel $P$ and a policy $\pi$, let $V^\pi_{P}$ denote the value function of $\pi$ under the transition kernel $P$.
For each state $s\in\mathcal S$, we define the posterior $\alpha$-quantile value of policy $\pi$ by
\begin{align}
{V^{\pi,\mathrm q}_{\phi,\alpha}(s)}
:=
\sup\Bigl\{
c\in\mathbb R:
\mathbb P_\phi\bigl(V^\pi_{P}(s)\ge c\bigr)\ge 1-\alpha
\Bigr\}.
\label{eq:posterior-lb}
\end{align}
Thus, \(V^{\pi,\mathrm q}_{\phi,\alpha}(s)\) is the largest value threshold that \(V^\pi_P(s)\) exceeds with posterior probability at least \(1-\alpha\). Smaller values of $\alpha$ correspond to a stronger robustness guarantee.

Assume that, for every fixed policy \(\pi\) and state \(s\), the posterior distribution of
\(V^\pi_P(s)\) is continuous, then $V^{\pi,\mathrm q}_{\phi,\alpha}(s)=\rho^\alpha_{\phi}\!\bigl(V^\pi_{P}(s)\bigr)$. Here, $\rho^\alpha_{\phi}\!\bigl(V^\pi_{P}(s)\bigr)$ applies the risk measure to the entire value function $V^\pi_P$, rather than applying the risk measure at each transition stage, as in the nested formulation of $V^\pi_{\phi,\alpha}$ in \eqref{eq:nested V}. Since the quantile functional is non-additive, {$V^{\pi,\mathrm q}_{\phi,\alpha}$} does not admit a Bellman equation in general. 
The hardness results of \citet{DelageMannor2010} imply that directly optimizing the corresponding posterior quantile objective, i.e. $\max_{\pi} V^{\pi,\mathrm{q}}_{\phi,\alpha}(s)$, is NP-hard under general uncertainty in the transition parameters. Even under independent Dirichlet transition priors, exact optimization remains difficult.
Instead, $\alpha$-quantile BR-MDP value function can serve as a tractable lower bound, as shown in the following proposition.

\begin{proposition}
\label{prop:hp-lb-no-subst}
For $\alpha\in(0,1)$, let $\bar\alpha := 1-(1-\alpha)^{1/S}$. Assume that, for every fixed policy \(\pi\) and state \(s\), the posterior distribution of
\(V^\pi_P(s)\) is continuous. For any fixed policy $\pi$,
$$
\mathbb P_\phi\Bigl(
V^\pi_{P}(s)\ge V^\pi_{\phi,\bar\alpha}(s),
\quad \forall s\in\mathcal S
\Bigr)
\ge 1-\alpha.
$$
Consequently, for every $s\in\mathcal S$, 
{$V^{\pi,\mathrm q}_{\phi,\alpha}(s)\ge V^\pi_{\phi,\bar\alpha}(s).$}
\end{proposition}

The proof of Proposition~\ref{prop:hp-lb-no-subst} is in Appendix~\ref{app:hp-lb-no-subst}.
As an immediate consequence, for any fixed state $s$, 
{$\max_{\pi} V^{\pi,\mathrm q}_{\phi,\alpha}(s)
\ge
\max_{\pi} V^\pi_{\phi,\bar\alpha}(s).$}
Thus, maximizing the $\bar\alpha$-quantile value function gives a conservative surrogate for maximizing the {posterior $\alpha$-quantile value}. 
{Proposition}~\ref{prop:hp-lb-no-subst}, together with \eqref{eq:posterior-lb}, shows that, for small $\alpha$, $\alpha$-quantile BR-MDP admits a finite-sample robustness interpretation: the posterior $\alpha$-quantile value $V^{\pi,\mathrm q}_{\phi,\alpha}$ provides a posterior high-probability ($1-\alpha$) performance guarantee, and the nested $\bar\alpha$-quantile BR-MDP value $V^\pi_{\phi,\bar\alpha}$ is a tractable lower bound for this robustness guarantee.

\section{Online RL via Adaptive Quantile Scheduling}
\label{sec:online-adaptive-quantile-scheduling}
Section~\ref{sec:formulation} showed that the quantile level in $\alpha$-quantile BR-MDP provides a flexible way to handle epistemic uncertainty captured by the posterior distribution for different purposes, namely robustness and exploration. 
We now turn to the online setting, where this flexibility becomes algorithmically useful for adapting to the evolving robustness--exploration trade-off over the course of learning. To illustrate such a trade-off, two examples that separately demonstrate the robustness and exploration sides of this trade-off are presented in Section~\ref{sec:robustness-exploration-tension}.
%
Section~\ref{sec:alpha} introduces an adaptive quantile schedule that tracks this evolving trade-off by varying the quantile level across learning stages and across state-action pairs.
Then, Section~\ref{sec:episode} embeds this mechanism into a pseudo-episode scheme for discounted infinite-horizon problems, enabling a fully online learning algorithm. Specifically, at the start of each pseudo-episode, we update the posterior, compute the quantile level, solve the corresponding $\alpha$-quantile BR-MDP, and then execute the resulting policy to collect more data until the next update. 

\subsection{Adaptive Quantile Scheduling}
\label{sec:alpha}

To track the evolving robustness--exploration trade-off identified in Section~\ref{sec:robustness-exploration-tension}, we let the quantile level vary across pseudo-episodes and across state-action pairs.

We index the pseudo-episodes by $k=1,2,\ldots$. At the beginning of the $k$th pseudo-episode, for each $(s,a)\in\mathcal S\times\mathcal A$, let $N_k(s,a,s')$ denote the number of transitions from $(s,a)$ to $s'$, and let $N_k(s,a):=\sum_{s'\in\mathcal{S}}N_k(s,a,s')$ denote the total number of visits to $(s,a)$. 
Under the prior $\Dir(1,\ldots,1)$, for each $(s,a)$, the posterior distribution of $P_{s,a}$ at pseudo-episode $k$ is denoted by $\Dir(\phi_k(s,a))$ and the posterior parameters are updated as follows:
\begin{align}
\phi_k(s,a,s')
&=1+N_k(s,a,s'),\quad
\phi_k(s,a)
=\bigl(\phi_k(s,a,s')\bigr)_{s'\in\mathcal S}.
\label{eq:update_phi_k}
\end{align}
Let $\phi_k=\{\phi_k(s,a)\}_{(s,a)\in\mathcal{S}\times\mathcal{A}}$  denote the collection of posterior parameters at the beginning of the $k$th pseudo-episode.
{To construct the adaptive quantile schedule, we first define the adjusted visit count and its average by $N_k^+(s,a):=\max\{N_k(s,a),1\}$, and $\bar N_k^+:=\frac{1}{SA}\sum_{(s,a)\in\mathcal S\times\mathcal A} N_k^+(s,a).$
The relative visit count of $(s,a)$ is then defined as
    $$r_k(s,a):=\frac{N_k^+(s,a)}{\bar N_k^+},~ (s,a)\in\mathcal S\times\mathcal A,
    $$ 
which compares the number of visits to $(s,a)$ with the average adjusted visit count over the
state-action space. We also define the scaling factor:
    $$
    g_k:=\frac{\ln(2k)}{\sqrt{k}},
    $$
whose scaling form is chosen in accordance with the regret guarantee established in Section~\ref{sec:regret}.}

Given a problem-specified robustness target parameter $\underline{\alpha}\in(0,1)$, which encodes the degree of lower-tail robustness that the decision maker deems sufficient when acting under epistemic uncertainty, and an algorithmic sensitivity parameter $\delta>0$, we define the adaptive quantile schedule as follows:
\begin{align}
\alpha_k(s,a)
&=
\max\left\{
1-\delta\,r_k(s,a)\,g_k,\ \underline{\alpha}
\right\}.
\label{eq:alphak}
\end{align}
{Before explaining the intuition behind the schedule \eqref{eq:alphak}, we extend the formulation of the $\alpha$-quantile BR-MDP with a common quantile level $\alpha$ across all state-action pairs to allow state-action-dependent quantile levels. Specifically, let $\alpha:\mathcal S\times\mathcal A\to(0,1)$ and replace $\rho^\alpha_{\phi(s_t,a_t)}$ in \eqref{eq:nested V} with $\rho^{\alpha(s_t,a_t)}_{\phi(s_t,a_t)}$. With a slight abuse of notation, we write $\alpha=(\alpha(s,a))_{(s,a)\in\mathcal S\times\mathcal A}$ and use $\rho^\alpha_{\phi(s,a)}$ as shorthand for $\rho^{\alpha(s,a)}_{\phi(s,a)}$. With this notation, the Bellman equations \eqref{eq:VR-pi}--\eqref{eq:VR-opt} and the Bellman operators \eqref{eq:TR-pi}--\eqref{eq:TR} retain exactly the same form in the state-action-dependent case.
Moreover, under the same structural conditions on the risk measure as in Section~\ref{sec:formulation}, the corresponding Bellman operators remain $\gamma$-contractions under $\|\cdot\|_\infty$ and therefore admit unique fixed points.}

The factor $g_k$ in \eqref{eq:alphak} depends only on $k$ and controls how far the quantile levels are pushed toward the lower tail across different pseudo-episodes. Early in learning, before visit counts become highly imbalanced across state-action pairs, $r_k(s,a)$ typically does not vary much across state-action pairs, 
whereas the common factor $g_k$ is relatively large in the early pseudo-episodes. Accordingly, the quantile levels $\alpha_k(s,a)$ remain in the lower-tail regime, so the Bellman backup of the corresponding quantile BR-MDP focuses on lower-tail performance of the EV of successor states and therefore yields a more robust policy as discussed in Section~\ref{sec:formulation}. 

This robustness is not uniform across state--action pairs. The factor $r_k(s,a)$ adjusts the quantile level according to relative visit counts. Before the floor $\underline{\alpha}$ becomes binding, a less-visited $(s,a)$ pair has a smaller $r_k(s,a)$ and thus a larger $\alpha_k(s,a)$. Consequently, in the Bellman backup for such $(s,a)$, the EV of successor states is evaluated at a less extreme lower quantile level.
This effect is not confined to the local state--action pair but propagates backward to states from which this less-visited pair can be reached through repeated Bellman backups. 

As $k$ increases, the common factor $g_k$ decreases and visit counts $N_k(s,a)$ can become imbalanced across $(s,a)$. Consequently, $r_k(s,a)g_k$ may become sufficiently small for less-visited $(s,a)$ pairs, causing the schedule to move the quantile level $\alpha_k(s,a)$ toward the upper tail and thereby encouraging exploration of these less-visited pairs, as discussed in Section~\ref{sec:formulation}. Moreover, in \eqref{eq:alphak}, $\delta$ controls the sensitivity of the schedule to relative visit counts.

\subsection{Pseudo-Episodes and Algorithm}
\label{sec:episode}
In finite-horizon MDPs, it is natural to update the posterior and recompute the policy at the beginning of each episode~\citep{osband2013psrl,osband2016without}. In the discounted infinite-horizon setting considered here, however, no natural episodes are available. Updating the posterior and recomputing the policy after every transition would be computationally demanding, whereas long deterministic artificial episodes may delay the use of newly collected data. We therefore partition the interaction stream into pseudo-episodes, as in \citet{xu2024continuing}. 

For each $t\ge 1$, after observing the transition at time $t$, we draw an independent pseudo-episode indicator $X_{t+1}\sim \mathrm{Bernoulli}(\gamma)$, and interpret $X_{t+1}=1$ as continuing the current pseudo-episode and $X_{t+1}=0$ as starting a new pseudo-episode at time $t+1$. We initialize $X_1=0$, so that time $1$ starts the first pseudo-episode. 
Then, the pseudo-episode length $L$ follows
a Geometric distribution with parameter $1-\gamma$ on $\{1,2,\ldots\}$, denoted by
$L\sim \mathrm{Geom}(1-\gamma)$, and is independent of the trajectory. With such a random length, the expectation of the accumulated reward with respect to $L$ is equal to the discounted accumulated reward, i.e. for any realized reward sequence
$\{r(s_i,a_i)\}_{i\ge0}$ indexed from the start of the pseudo-episode,
$$\mathbb E_{L\sim\mathrm{Geom}(1-\gamma)}
\left[\sum_{i=0}^{L-1} r(s_i,a_i)\right]
=
\sum_{i=0}^{\infty}\mathbb P(L\ge i+1)r(s_i,a_i)
=
\sum_{i=0}^{\infty}\gamma^i r(s_i,a_i).$$

Let $K_T$ denote the number of pseudo-episodes started by time $T$, and hence $\mathbb E[K_T]=1+(T-1)(1-\gamma)$. 
Hence, the expected number of pseudo-episodes, which is equivalent to the number of posterior updates and policy recomputations, grows on the order of $(1-\gamma)T$ rather than $T$.

Let $\mathcal H_t:=\{(s_\tau,a_\tau,s_{\tau+1})\}_{\tau=1}^{t-1}$ denote the interaction history available before time $t$. 
By incorporating the adaptive quantile schedule with the pseudo-episodic posterior updates, we obtain the Adaptive Quantile BR-MDP (AQ-BRMDP), summarized in Algorithm~\ref{alg:AQ-BRMDP}.
\begin{algorithm}[htbp]
\caption{AQ-BRMDP}
\label{alg:AQ-BRMDP}
\begin{algorithmic}[1]
\State \textbf{Input:} Discount factor $\gamma$, total learning time $T$, schedule parameters $\delta$ and $\underline{\alpha}$
\State Initialize $t\gets 1$, $k\gets 0$, $X_1\gets 0$, and $\mathcal H_1\gets\emptyset$
\While{$t\le T$}
    \If{$X_t=0$}
        \State $k\gets k+1$ and $t_k\gets t$
        \State Update the posterior $\phi_k$ using $\mathcal H_{t_k}$ according to \eqref{eq:update_phi_k}
        \State Compute $\alpha_k(s,a)$ for all $(s,a)\in\mathcal S\times\mathcal A$ via \eqref{eq:alphak}
        \State Compute $\pi_k$ by solving the $\alpha_k$-quantile BR-MDP under $\phi_k$ via value iteration (Algorithm~\ref{alg:brmdp-vi-k} in Appendix~\ref{sec:value-iteration-exact})
    \EndIf
    \State Take action $a_t\gets \pi_k(s_t)$ and observe the next state $s_{t+1}$
    \State Update the history $\mathcal H_{t+1}\gets \mathcal H_t\cup\{(s_t,a_t,s_{t+1})\}$
    \If{$t<T$}
        \State Sample $X_{t+1}\sim \mathrm{Bernoulli}(\gamma)$
    \EndIf
    \State $t\gets t+1$
\EndWhile
\end{algorithmic}
\end{algorithm}
\section{Regret Analysis}
\label{sec:regret}

Section~\ref{sec:online-adaptive-quantile-scheduling} introduced the adaptive quantile
schedule and its pseudo-episodic implementation to account for the evolving robustness--exploration trade-off in online learning. We now establish performance guarantees for AQ-BRMDP.
We analyze Bayesian regret relative to two complementary benchmarks.
The first benchmark is the true-optimal value $V^*$, which measures regret relative to the optimal value function in the true environment.
The second benchmark is the robust-optimal value in the $k$th pseudo-episode $V^*_{\phi_k,\underline{\alpha}}$, which measures regret relative to the optimal value function of the $\underline{\alpha}$-quantile BR-MDP under posterior parameters $\phi_k$.
In implementation, the quantiles are approximated by sampling transition kernels from the posterior distribution, and the corresponding quantile BR-MDP is solved in the simulator. Compared with the cost of interacting with the true environment, the computational cost of posterior sampling and solving the quantile BR-MDP is treated as negligible. Accordingly, we assume that the policy computed in each pseudo-episode can be made arbitrarily close to the exact optimal policy of the corresponding quantile BR-MDP, and we ignore the resulting approximation error in the regret analysis.


Throughout this section, $s_{k,i}$ denotes the state at the $i$th time step of pseudo-episode $k$, $\pi_k=\pi^*_{\phi_k,\alpha_k}$ is the policy executed in pseudo-episode $k$, $V_k:=V^*_{\phi_k,\alpha_k}$ denotes the optimal value of the corresponding adaptive $\alpha_k$-quantile BR-MDP, and \(L_k\) denotes the realized length of pseudo-episode \(k\) within the interaction horizon \(T\). Bayesian regret (BR) relative
to the true-optimal benchmark is defined by
\begin{align}
\label{eq:BRT}
BR(T)
:=
\mathbb E\!\left[
\sum_{k=1}^{K_T}\sum_{i=1}^{L_k}
\Bigl(
V^*(s_{k,i})-V^{\pi_k}(s_{k,i})
\Bigr)
\right],
\end{align}
where the expectation is taken over the prior distribution of the true transition kernel,
the trajectory randomness, and the randomness of the pseudo-episode lengths.
To quantify whether exploration incurs substantial loss relative to a conservative benchmark,
we also define Bayesian regret relative to the robust-optimal benchmark (BR-R) as follows:
\begin{align}
\label{eq:BRRT}
BR\text{-}R(T)\!
:=\!
\mathbb E\!\left[
\sum_{k=1}^{K_T}\sum_{i=1}^{L_k}
\Bigl(
V^*_{\phi_k,\underline{\alpha}}(s_{k,i})-V^{\pi_k}_{\phi_k,\underline{\alpha}}(s_{k,i})
\Bigr)
\right],
\end{align}
where the expectation is taken over the prior distribution of the true transition kernel,
the trajectory randomness, and the randomness of the pseudo-episode lengths. Let \(\mathcal F_t\) denote the sigma-field generated by the interaction history and the pseudo-episode indicators observed before time \(t\): $\mathcal F_t
:=
\sigma\!\left((s_\tau,a_\tau,s_{\tau+1})_{\tau=1}^{t-1},\,X_1,\ldots,X_t\right).$


\begin{lemma}[Dual-benchmark optimism]
\label[lemma]{lem:bayes_optimism}
Fix pseudo-episode $k$ with starting time $t_k$. Then:
\begin{enumerate}[label=(\roman*)]
    \item $\mathbb P\!\left(
    \forall s\in\mathcal S:\ V^*(s)\le V_k(s)
    \,\middle|\,
     \mathcal F_{t_k}
    \right)
    \ge
    1-\frac{\delta SA\ln(2k)}{\sqrt{k}}.$
    \item $V^*_{\phi_k,\underline{\alpha}}(s)\le V_k(s),
    \qquad \forall s\in\mathcal S.$
\end{enumerate}
\end{lemma}

We defer the proof of Lemma~\ref{lem:bayes_optimism} to Appendix~\ref{app:bayes_optimism}.
Part~(i) in Lemma~\ref{lem:bayes_optimism} shows that the true optimal value function $V^*$ can be upper bounded by $V_k$ with posterior probability tending to 1 as learning progresses (i.e., as $k$ increases).
Part~(ii) in Lemma~\ref{lem:bayes_optimism} shows that the robust-optimal value function $V^*_{\phi_k,\underline{\alpha}}$ can also be upper bounded by $V_k$ for every state $s$.

Define the optimistic event $\mathcal G_k:=\cap_{(s,a)\in\mathcal{S}\times\mathcal{A}}
\left\{
(P^c_{s,a})^\top V_k
\le
\rho^{\alpha_k}_{\phi_k(s,a)}(P^\top V_k)
\right\}$, which is the event that the BR-MDP EV overestimates the true EV. We show that on $\mathcal G_k$, $V^*(s)\le V_k(s)$ for all $s\in\mathcal S$ in Appendix~\ref{app:bayes_optimism}. Hence, we can decompose the upper bound for $BR(T)$ in \eqref{eq:BRT} before taking expectations as follows:
\begin{align}
&\sum_{k=1}^{K_T}\sum_{i=1}^{L_k}
\Bigl(
V^*(s_{k,i})-V^{\pi_k}(s_{k,i})
\Bigr)
\le
\sum_{k=1}^{K_T}\sum_{i=1}^{L_k}
\Bigl(
V_k(s_{k,i})-V^{\pi_k}(s_{k,i})
\Bigr)
+
\frac{1}{1-\gamma}
\sum_{k=1}^{K_T} L_k\,\mathbb I\{\mathcal G_k^c\}.
\label{eq:decomp-good-bad-rigorous-sa}
\end{align}
Using Lemma~\ref{lem:bayes_optimism} (ii), we can decompose the corresponding upper bound for $BR\text{-}R(T)$ in \eqref{eq:BRRT} as follows:
\begin{align}
&\sum_{k=1}^{K_T}\sum_{i=1}^{L_k}
\Bigl(
V_{\phi_k,\underline{\alpha}}^*(s_{k,i})
-V_{\phi_k,\underline{\alpha}}^{\pi_k}(s_{k,i})
\Bigr)
\le
\sum_{k=1}^{K_T}\sum_{i=1}^{L_k}
\Bigl(
V_k(s_{k,i})
-V_{\phi_k,\underline{\alpha}}^{\pi_k}(s_{k,i})
\Bigr).
\label{eq:decomp-BR-R}
\end{align}

Therefore, the main remaining terms to control are the value-difference terms in \eqref{eq:decomp-good-bad-rigorous-sa} and \eqref{eq:decomp-BR-R}, together with the non-optimistic-event term in \eqref{eq:decomp-good-bad-rigorous-sa}. The next lemma shows that the gap between the value of a fixed policy in the true environment $V^\pi$ and the value of $\alpha_k$-quantile BR-MDP under the posterior distribution with parameters $\phi_k$, $V^\pi_{\phi_k,\alpha_k}$, can be expressed as the cumulative one-step discrepancy between the $\alpha_k$-quantile BR-MDP Bellman backup and the corresponding Bellman backup under the true transition along the realized trajectory.

\begin{lemma}[Pseudo-episode value decomposition]
\label[lemma]{lem:value-decomposition-sum}
Fix pseudo-episode $k$ with starting time $t_k$, and fix a stationary deterministic policy
$\pi$. Let $\{(s_{k,i},a_{k,i})\}_{i\ge 1}$ be the infinite-horizon trajectory generated by
following $\pi$ from time $t_k$ in the true MDP, so that
$a_{k,i}:=\pi(s_{k,i})$ for all $i\ge 1$.
Define
$
E_{k,i}
:=
\rho^{\alpha_k}_{\phi_k(s_{k,i},a_{k,i})}
\!\bigl(P^\top V^\pi_{\phi_k,\alpha_k}\bigr)
-
(P^c_{s_{k,i},a_{k,i}})^\top V^\pi_{\phi_k,\alpha_k},$ $
 i\ge 1.
$

Then
\begin{align}
&\mathbb E\!\left[
\sum_{i=1}^{L_k}
\Bigl(
V^\pi_{\phi_k,\alpha_k}(s_{k,i})-V^\pi(s_{k,i})
\Bigr)
\,\middle|\,
\mathcal F_{t_k},\,P^c
\right]
=
\mathbb E\!\left[
\sum_{i=1}^{L_k} i\,\gamma\, E_{k,i}
\,\middle|\,
\mathcal F_{t_k},\,P^c
\right].
\label{eq:agg-vd}
\end{align}
Here the conditional expectation is taken over the trajectory randomness under $\pi$ in the
true MDP and the independent pseudo-episode lengths $\{L_k:k\ge 1\}$.
\end{lemma}

The proof of Lemma~\ref{lem:value-decomposition-sum} is deferred to Appendix~\ref{app:value-decomposition-sum}.
For the robust regret, the following lemma provides an analogous decomposition between the $\alpha_k$-quantile BR-MDP value function and the $\underline{\alpha}$-quantile BR-MDP value function.

\begin{lemma}
\label[lemma]{lemma:value-decomposition-sum-robust}
Fix pseudo-episode $k$ with starting time $t_k$, and fix a stationary deterministic policy
$\pi$. Let $\{(s_{k,i},a_{k,i})\}_{i\ge 1}$ be the infinite-horizon trajectory generated by
following $\pi$ from time $t_k$ in the true environment, so that
$a_{k,i}:=\pi(s_{k,i})$ for all $i\ge 1$.
Define
$
E^-_{k,i}
:=
(P^c_{s_{k,i},a_{k,i}})^\top V^\pi_{\phi_k,\underline{\alpha}}
-
\rho^{\underline{\alpha}}_{\phi_k(s_{k,i},a_{k,i})}
\!\bigl(P^\top V^\pi_{\phi_k,\underline{\alpha}}\bigr),
$
and
$
E_{k,i}
:=
\rho^{\alpha_k}_{\phi_k(s_{k,i},a_{k,i})}
\!\bigl(P^\top V^\pi_{\phi_k,\alpha_k}\bigr)
-
(P^c_{s_{k,i},a_{k,i}})^\top V^\pi_{\phi_k,\alpha_k}.
$
Then
\begin{align}
&\mathbb E\!\left[
\sum_{i=1}^{L_k}
\Bigl(
V^\pi_{\phi_k,\alpha_k}(s_{k,i})-V^\pi_{\phi_k,\underline{\alpha}}(s_{k,i})
\Bigr)
\,\middle|\,
\mathcal F_{t_k},\,P^c
\right]
=
\mathbb E\!\left[
\sum_{i=1}^{L_k}i\,\gamma\,
\bigl(E_{k,i}+E^-_{k,i}\bigr)
\,\middle|\,
\mathcal F_{t_k},\,P^c
\right].
\label{eq:agg-vd-floor}
\end{align}
Here the conditional expectation is taken over the trajectory randomness under $\pi$ in the
true environment and the independent pseudo-episode lengths $\{L_k:k\ge 1\}$.
\end{lemma}

Using \cref{lem:bayes_optimism,lem:value-decomposition-sum,lemma:value-decomposition-sum-robust} together with Dirichlet-posterior concentration in \cref{lem:dirichlet_quantile_BR-MDP}, we upper bound the Bayesian regret of AQ-BRMDP under both benchmarks.

\begin{theorem}
\label{thm:main-order}
Suppose that the schedule parameter $\delta$ in \eqref{eq:alphak} satisfies $\delta>0$.
If
$T\ge S^2A$, then
$$
    BR(T)\le
    \widetilde O\!\left(
    \frac{\gamma\sqrt{SA\,T\ln\frac{e}{1-\underline{\alpha}}}}{(1-\gamma)^2}
    +
    \frac{\delta\,SA\,\sqrt{T}}{(1-\gamma)^2}
    +
    \frac{SA}{(1-\gamma)^3}
    \right),
$$
and
$$
    BR\text{-}R(T)
    \le
    \widetilde O\!\left(
    \frac{\gamma\sqrt{SA\,T\ln\frac{1}{\min\{1-\underline{\alpha},\underline{\alpha}\}}}}{(1-\gamma)^2}
    +
    \frac{SA}{(1-\gamma)^3}
    \right),
$$
where $\widetilde O(\cdot)$ suppresses polylogarithmic factors in $S$, $A$, $T$, $1/\delta$, and $1/(1-\gamma)$.
\end{theorem}

Theorem~\ref{thm:main-order} formalizes that the adaptive schedule accounts for the evolving robustness--exploration trade-off described in Section~\ref{sec:online-adaptive-quantile-scheduling}. For fixed problem parameters, the Bayesian regret bound for AQ-BRMDP grows as $\widetilde O(\sqrt{T})$ in the total number of interactions $T$. The term $\delta SA\sqrt{T}/(1-\gamma)^2$ captures the additional regret associated with using a more robust interaction policy early in learning. Choosing a larger constant $\delta$ would strengthen early robustness, at the cost of a larger $\delta SA\sqrt{T}/(1-\gamma)^2$ term in the regret bound.
If the focus is on the standard exploration--exploitation trade-off rather than on prioritizing additional early-stage robustness, we can set $\delta=1/\sqrt{SA}$. 
This choice keeps the robustness-induced regret term at the same order as the first term, $\sqrt{SA T}/(1-\gamma)^2$, in its dependence on $S$, $A$, $T$, and $(1-\gamma)^{-1}$ up to the explicit factor $\gamma\sqrt{\ln(e/(1-\underline{\alpha}))}$ and polylogarithmic factors hidden in $\widetilde O(\cdot)$. The $T$-independent lower-order term $SA/(1-\gamma)^3$ comes from the finite-time analysis of the pseudo-episode construction for discounted infinite-horizon learning. When $T \ge SA\max\{S,1/(1-\gamma)^2\}$, this lower-order term is absorbed by the $\sqrt T$ terms, so the dominant scaling becomes $\widetilde O(\sqrt{SA T}/(1-\gamma)^2)$. 
Thus, with this choice of $\delta$, AQ-BRMDP achieves the same Bayesian regret rate as state-of-the-art Bayesian RL algorithms, such as the $\widetilde O(H\sqrt{SAT})$ bound for posterior sampling in RL (PSRL) \citep{osband2017psbetter}, up to logarithmic and horizon-dependent factors.

The difference in horizon dependence stems from the different definitions of Bayesian regret in finite-horizon episodic and discounted infinite-horizon settings. Finite-horizon episodic PSRL bounds $\widetilde O(H\sqrt{SAT})$ are typically expressed with a linear dependence on the horizon $H$ \citep{osband2017psbetter}, because regret is accumulated once per episode through the value gap at the initial state of each episode. 
In contrast, our discounted infinite-horizon regret accumulates a value gap at every interaction time $t$. 
Since each value gap itself represents a discounted future loss over an effective horizon of order $(1-\gamma)^{-1}$, this per-step regret criterion introduces an additional effective-horizon factor, leading to the $(1-\gamma)^{-2}$ dependence in our bound.

The Bayesian robust regret for AQ-BRMDP, $BR\text{-}R(T)$, indicates that the adaptive quantile schedule, while becoming less robust to encourage learning of the true-environment optimal policy, does not incur large cumulative regret relative to the robust-optimal value. It grows as $\widetilde O(\sqrt{SAT})$, 
and increases as $\underline{\alpha}$ approaches $0$ or $1$, following an order of $O\left(\sqrt{\ln\frac{1}{\min\{1-\underline{\alpha},\underline{\alpha}\}}}\right)$. 
This dependence is comparable to the
state-action and time dependence appearing in recent tabular online robust and distributionally
robust RL guarantees. For finite-horizon robust MDPs with $K$ episodes and horizon $H$,
\citet{dong2022onlinerobust} obtain frequentist robust-regret bounds of order
$\widetilde O(H^2S\sqrt{AK})$ for $(s,a)$-rectangular uncertainty sets. More recent online
$f$-divergence distributionally robust RL results obtain $\sqrt{SAK}$ dependence,
for example $\widetilde O(\sqrt{H^4(1+\sigma)SAK})$ for $\chi^2$ ambiguity sets, where $\sigma$ is the radius of the ambiguity set \citep{ghosh2025orvit}.
The comparison, however, should be interpreted carefully. Existing online robust RL results
typically establish frequentist robust regret bounds under a fixed but unknown
environment, with the benchmark given by the optimal policy for a worst-case value over an
ambiguity set. Such frequentist guarantees are
stronger than Bayesian regret guarantees when the model class and benchmark coincide, since a Bayesian regret bound is the prior average of the true environment regret bound.
In our setting, however, the BR-MDP formulation itself is posterior-based: epistemic
uncertainty is represented by the posterior distribution, and the Bellman backup is evaluated
through a posterior risk criterion. Bayesian regret is therefore the natural performance
measure, in the same spirit as PSRL, where regret is commonly analyzed after averaging over the prior distribution \citep{osband2013psrl,osband2017psbetter,xu2024continuing}.


\section{Experiments}
\label{sec:experiments}

We evaluate AQ-BRMDP in finite-state environments with two complementary settings: RiverSwim in Section~\ref{sec:exp-demanding}, which requires sustained exploration toward a distant rewarding state, and risky-branch FrozenLake in Section~\ref{sec:exp-costly}, where exploratory shortcut actions can lead to sticky holes. Appendix~\ref{app:sensitivity-analysis} reports the sensitivity analysis for the schedule parameter $\delta$ in AQ-BRMDP.
We also implement an extension of AQ-BRMDP to continuous-state spaces and evaluate its empirical performance in a continuous-state environment in Appendix~\ref{app:implementation-details}.

We compare AQ-BRMDP with two classes of algorithms: Continuing PSRL and fixed-quantile BR-MDPs, denoted by BRMDP-0.1, BRMDP-0.3, and BRMDP-0.5, which use constant quantile levels $\alpha=0.1,0.3,0.5$, respectively. Continuing PSRL uses the same pseudo-episode mechanism but samples one transition kernel from the current posterior and acts optimally in the sampled MDP. The fixed-quantile BR-MDP baselines use a constant quantile level throughout learning; smaller quantiles correspond to more robust policies as discussed in Section~\ref{sec:formulation}. For BRMDP-based methods, posterior quantiles in Bellman backups are approximated by empirical quantiles from posterior samples; the corresponding value-iteration procedure and Monte Carlo budgets are given in Appendix~\ref{sec:value-iteration}. Unless otherwise noted, curves report empirical means over $100$ independent runs with $95\%$ confidence intervals.

\subsection{Exploration-Demanding Environments}
\label{sec:exp-demanding}

We first consider RiverSwim-$n$ with $n\in\{6,10\}$. The state space is $\mathcal S_n=\{1,\ldots,n\}$, the initial state is $1$, and the action space is $\mathcal A=\{a_L,a_R\}$. The left action $a_L$ moves deterministically to the left neighboring state, and at state $1$ the next state remains $1$, i.e., $P^c(s-1\mid s,a_L)=1$ for $2\le s\le n$ and $P^c(1\mid1,a_L)=1$. The rewards satisfy $r(1,a_L,s')=0.005$, $r(n,a_R,n)=1$, and $r(s,a,s')=0$ otherwise. For the right action, the true transition kernel makes sustained rightward progress difficult: at interior states $2\le s\le n-1$, $P^c(s+1\mid s,a_R)=0.35$, $P^c(s\mid s,a_R)=0.60$, and $P^c(s-1\mid s,a_R)=0.05$; at the left end, $P^c(2\mid1,a_R)=0.60$ and $P^c(1\mid1,a_R)=0.40$; at the right end, $P^c(n\mid n,a_R)=0.60$ and $P^c(n-1\mid n,a_R)=0.40$. The agent maintains a Dirichlet posterior over each transition vector $P^c_{s,a}$. Thus, obtaining reward $1$ requires sustained rightward exploration through the chain under stochastic transitions. The ten-state version lengthens this path and makes sustained exploration harder. A schematic of RiverSwim-6 is provided in Appendix Figure~\ref{fig:finite-state-schematics}.

\paragraph{Results.}
Figure~\ref{fig:rs6-main} reports cumulative true regret, moving-average reward trajectories, and state-occupancy heatmaps for RiverSwim-6 and RiverSwim-10. The reward trajectory at time $t$ is the moving average $\bar r_t^{(w)}:=w^{-1}\sum_{\tau=t-w+1}^{t}r_\tau$ with a window size of $w=100$. The state-occupancy heatmap reports the empirical state-occupancy measure $\widehat d_T(s):=T^{-1}\sum_{t=0}^{T-1}\mathbf 1\{s_t=s\}$. In RiverSwim-10, the fixed-quantile BR-MDP baselines behave similarly and their curves largely overlap in the cumulative-regret and moving-average-reward panels.

We have the following observations.
\begin{enumerate}[label=(\arabic*)]
    \item \textbf{Sustained exploration by AQ-BRMDP.} Compared with Continuing PSRL, AQ-BRMDP achieves lower cumulative regret in both RiverSwim-6 and RiverSwim-10. Its moving-average reward increases more quickly early in learning, indicating that it reaches the right end of the chain earlier. The occupancy heatmaps also show that AQ-BRMDP spends more time in states near the right end of the chain, especially in RiverSwim-10. These observations indicate that AQ-BRMDP supports more effective sustained exploration than Continuing PSRL.

    \item \textbf{Effect of fixed quantile levels.} The fixed-quantile BR-MDP baselines show that small quantile levels can be overly conservative in exploration-demanding environments. In RiverSwim-6, BRMDP-0.1 remains concentrated near state $1$, its moving-average reward stays close to the local reward $0.005$, and its cumulative regret is the largest among the fixed-quantile baselines. BRMDP-0.5 reaches states near the right end of the chain more often and attains higher moving-average rewards. In RiverSwim-10, all three fixed-quantile baselines remain concentrated near the initial state and their cumulative regret grows approximately linearly.
\end{enumerate}
These results are consistent with the mechanism in Section~\ref{sec:robustness-exploration-tension}. A fixed lower-tail quantile can make the policy too conservative to move away from the initial region and collect the data needed to reach the right end of the chain. By contrast, the adaptive quantile schedule in AQ-BRMDP reduces this conservatism over learning and improves exploration toward distant high-reward states.

\begin{figure*}[!t]
\centering
\includegraphics[width=0.96\textwidth]{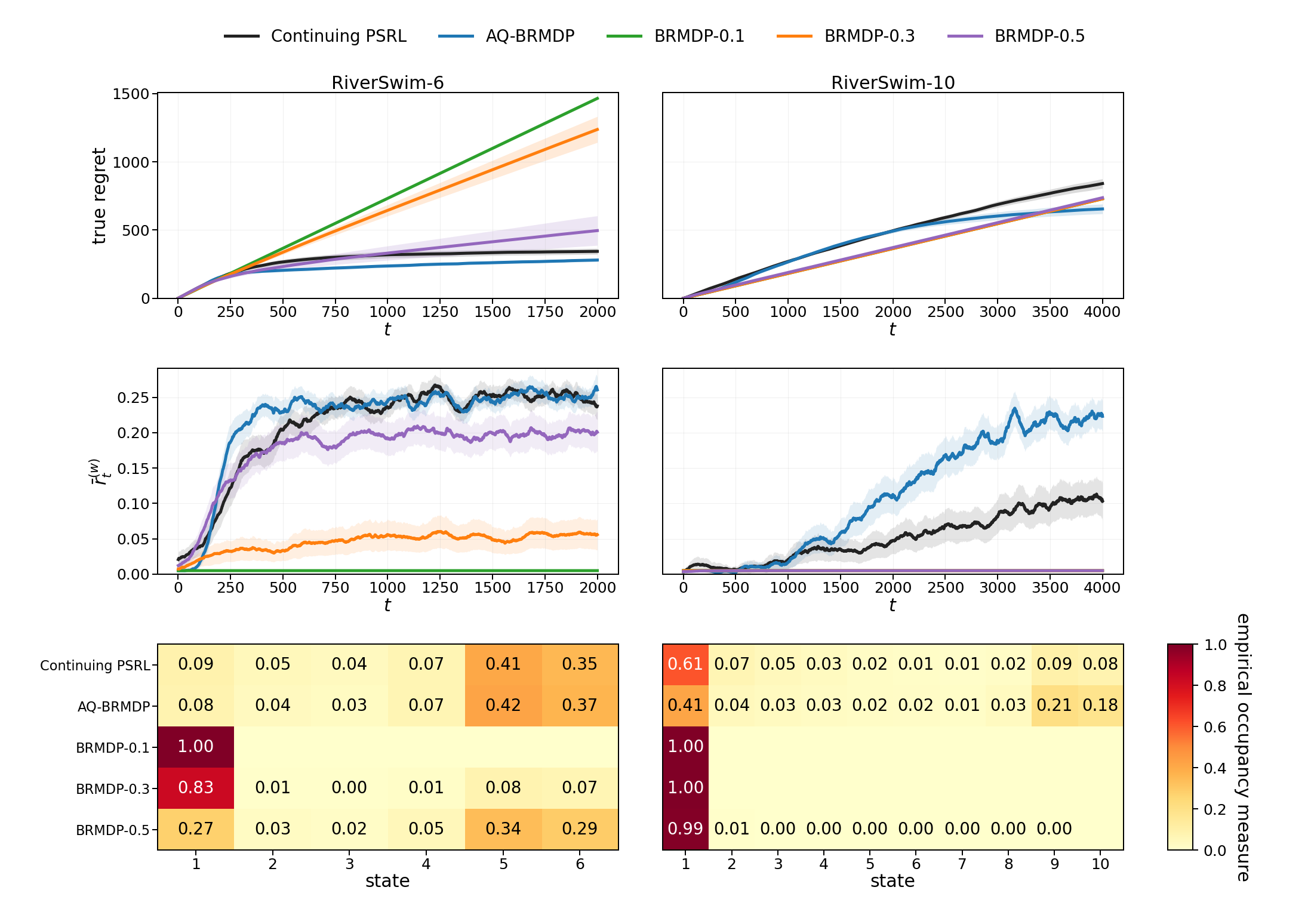}
\caption{Performance in RiverSwim-6 (left column) and RiverSwim-10 (right column). The first row shows cumulative true regret, the second row shows moving-average rewards with a window size of $100$, and the third row shows state-occupancy heatmaps.}
\label{fig:rs6-main}
\end{figure*}

\subsection{Exploration-Costly Environments}
\label{sec:exp-costly}

We next consider a risky-branch variant of FrozenLake. A description of this problem is provided in Appendix Figure~\ref{fig:finite-state-schematics}. The environment is a $4\times 4$ grid with states indexed by $\mathcal S=\{0,\ldots,15\}$, initial state $s_0=0$, goal state $s_G=15$, and hole states $H=\{5,7,11,12\}$. The action space is $\mathcal A=\{\mathrm{Left},\mathrm{Right},\mathrm{Up},\mathrm{Down}\}$. The reward is known: a transition to $s_G$ gives reward $1$, and all other transitions give reward $0$.
For each state that is neither a hole nor the goal, the transition is slippery. If the agent chooses an action $a$, then with probability $0.50$ the next state follows the intended move, and with probability $0.25$ for each perpendicular direction, it follows one of the two directions perpendicular to $a$. If a realized move leaves the grid, the agent remains in the current state. For $h\in H$ and action $a$, the agent moves in the intended direction with probability $p_h$ and remains at $h$ with probability $1-p_h$; we set $p_h=0.2$. After the agent reaches the goal, the next state is sampled from a uniform distribution supported on the non-hole, non-goal states.
The risky shortcut is at the state-action pair $(2,\mathrm{Down})$. Under the true transition kernel, taking action $\mathrm{Down}$ at state $2$ does not follow the default slippery transition rule. Instead,  $P^c(10\mid 2,\mathrm{Down})=1-\theta$ and $P^c(5\mid 2,\mathrm{Down})=\theta$.
Thus, this action can shorten the path to the goal when it succeeds, but with probability $\theta$ it sends the agent into a sticky hole. We consider $\theta=0.7$ and $\theta=0.9$. The agent treats the entire transition kernel $P^c$ as unknown and maintains an independent Dirichlet posterior over each $P^c_{s,a}$. 

In this environment, we also compute the posterior $0.1$-quantile value defined in \eqref{eq:posterior-lb} to evaluate the robustness of the policy executed during learning. For a time step $t$ in pseudo-episode $k$, let $\phi_t:=\phi_k$ and $\pi_t:=\pi_k$. The {posterior $0.1$-quantile value for the value of $\pi_t$ at the initial state $s_0$, $V^{\pi_t,\mathrm{q}}_{\phi_t,0.1}(s_0)$, means that when $\pi_t$ is deployed in the true environment, its value exceeds $V^{\pi_t,\mathrm{q}}_{\phi_t,0.1}(s_0)$ with posterior probability at least $0.9$.}
Hence, a larger $V^{\pi_t,\mathrm q}_{\phi_t,0.1}(s_0)$ indicates better robustness performance under the current posterior. As discussed in Appendix~\ref{app:posterior-lb-estimation}, we estimate $V^{\pi_t,\mathrm q}_{\phi_t,0.1}(s_0)$ via Monte Carlo sampling and report the approximate $\widehat V^{\pi_t,\mathrm q}_{\phi_t,0.1}(s_0)$ in Figure~\ref{fig:fl-quantile}.

\paragraph{Results.}

Figures~\ref{fig:fl-regret} and \ref{fig:fl-quantile} show the following patterns.
\begin{enumerate}[label=(\arabic*)]
     \item \textbf{Effective robustness--exploration trade-off.}
    {The cumulative robust-regret and true-regret curves together show that AQ-BRMDP achieves lower regret than Continuing PSRL under both criteria. This indicates that the adaptive quantile schedule can keep the policy robust relative to the optimal value of the $\underline{\alpha}$-quantile BR-MDP while also effectively learning the true-optimal policy.}

    \item \textbf{Reduced true regret of BRMDP-based policies in risky environments.}
    In both settings, AQ-BRMDP and the BRMDP-$\alpha$ baselines accumulate lower true regret than Continuing PSRL. This indicates that BRMDP-based policies are more effective in those risky-shortcut environments. The reason is that when the shortcut transition is not yet well estimated, a posterior sample may underestimate $P^c(5\,|\,2,\mathrm{Down})$, making policies that route the agent through state $2$ appear overly favorable. Under the true kernel, however, repeated use of action $\mathrm{Down}$ at state $2$ can move the agent to the sticky hole with high probability and produce larger regret.
    BRMDP-based policies instead conservatively evaluate the shortcut through posterior quantiles, and therefore tend to avoid the risky branch. 

    \item \textbf{Non-monotone effect of fixed quantile levels.} The BRMDP-$\alpha$ baselines reduce this early cost by evaluating lower-tail performance and therefore tend to avoid sticky holes and the risky branch. Among these baselines, BRMDP-0.3 generally yields lower regret than BRMDP-0.5, showing that some robustness is useful when the shortcut transition is not yet well estimated. However, BRMDP-0.1 can have higher regret than BRMDP-0.3, indicating that excessive conservatism may also cause large true regret. Thus, the performance of BRMDP-$\alpha$ is not monotone in the quantile level $\alpha$.

    \item \textbf{Posterior robustness under adaptive scheduling.} AQ-BRMDP generally attains the largest posterior $0.1$-quantile value in both settings, showing that its executed policies maintain stronger posterior robustness while still improving true regret relative to Continuing PSRL. 
\end{enumerate}

\begin{figure*}[!t]
\centering
\includegraphics[width=0.46\textwidth]{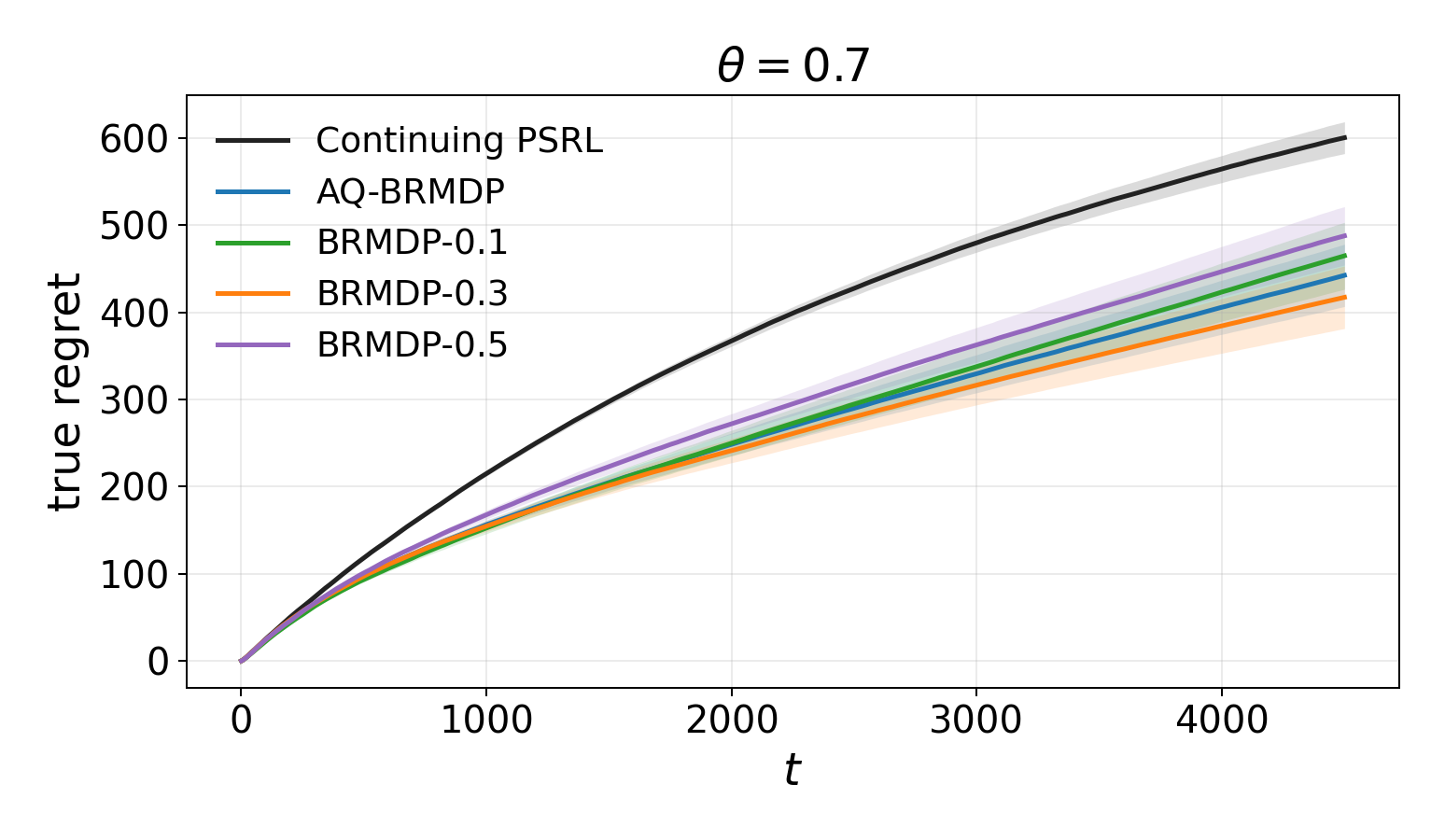}
\includegraphics[width=0.46\textwidth]{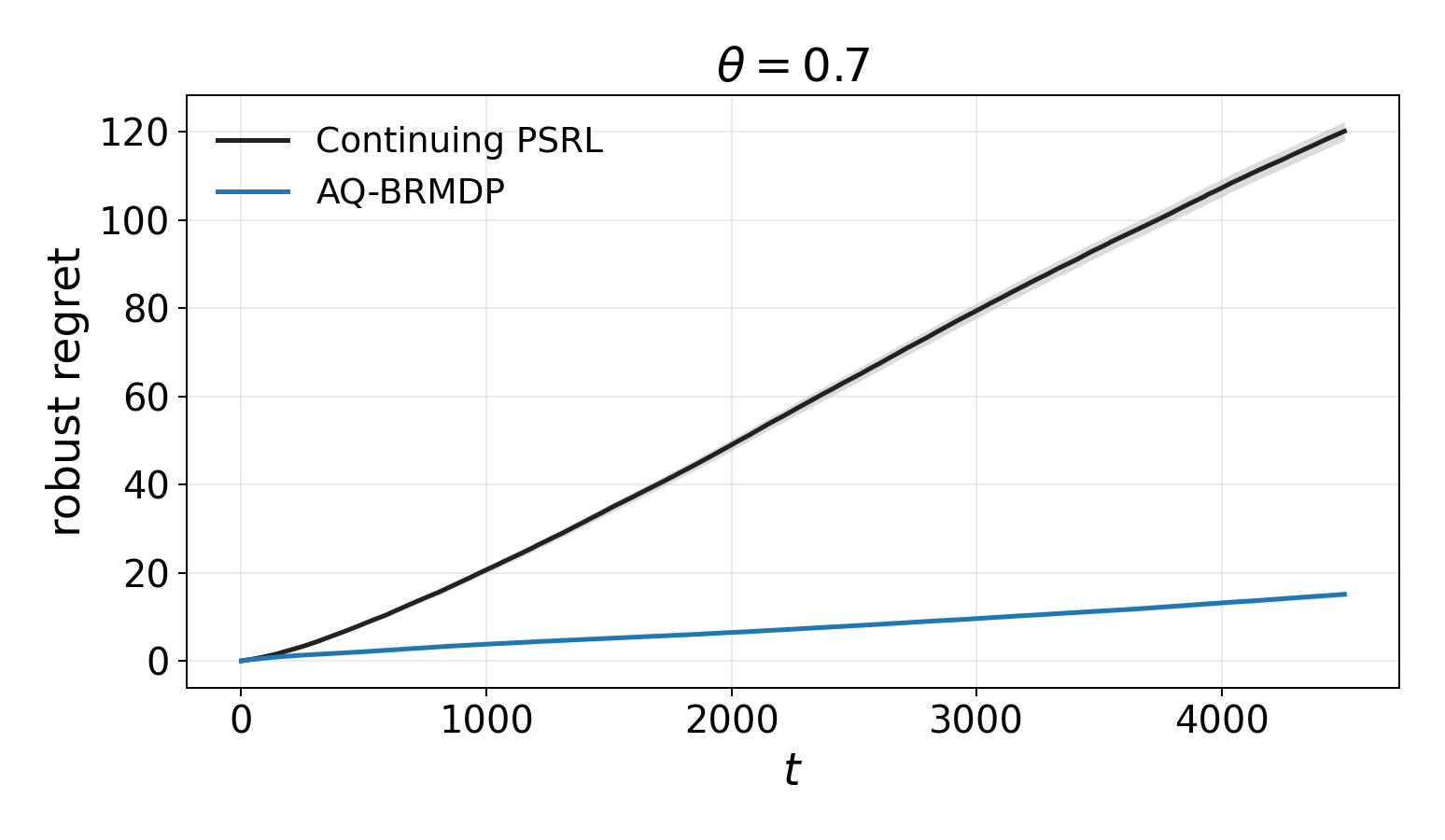}
\includegraphics[width=0.46\textwidth]{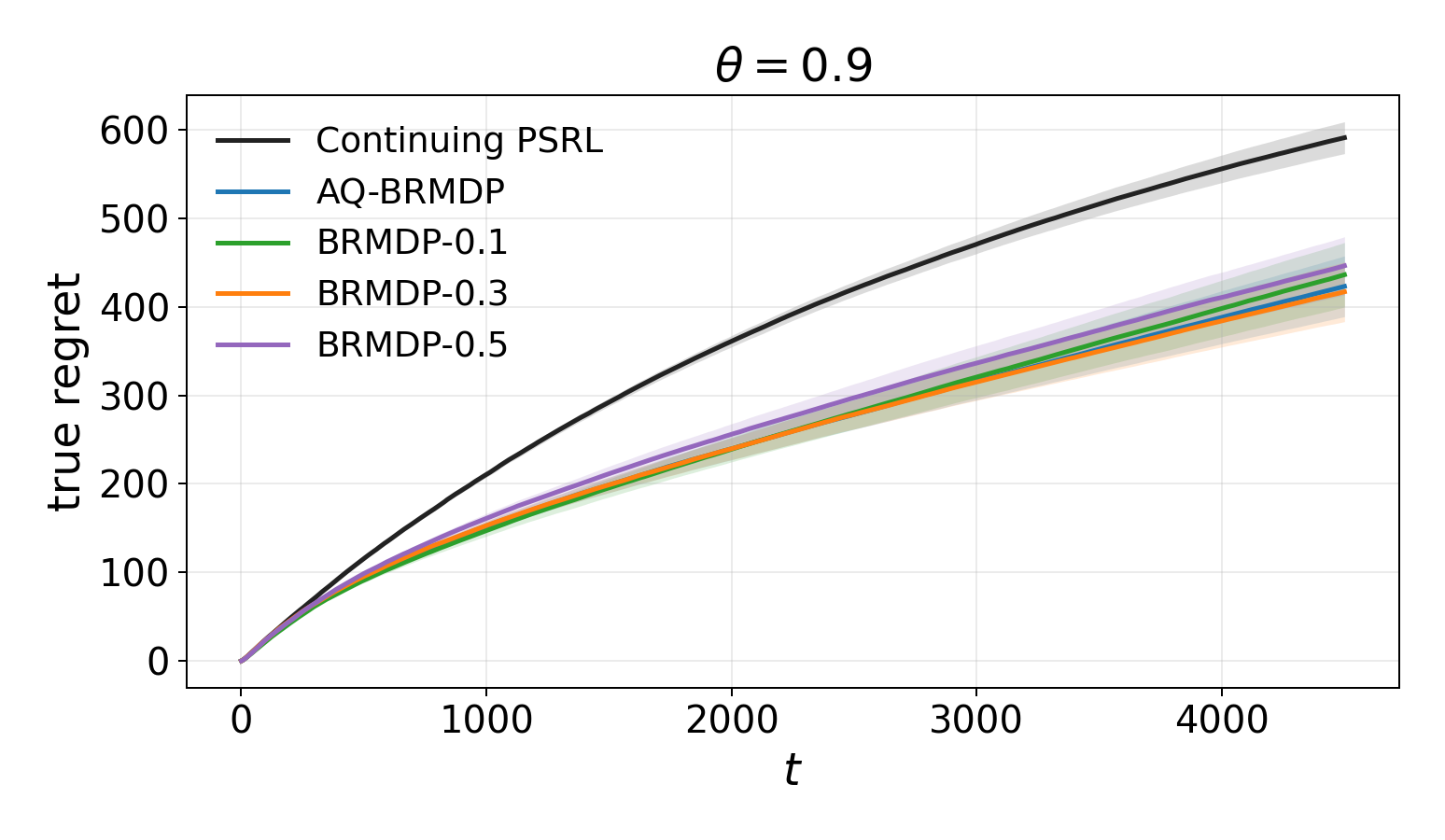}
\includegraphics[width=0.46\textwidth]{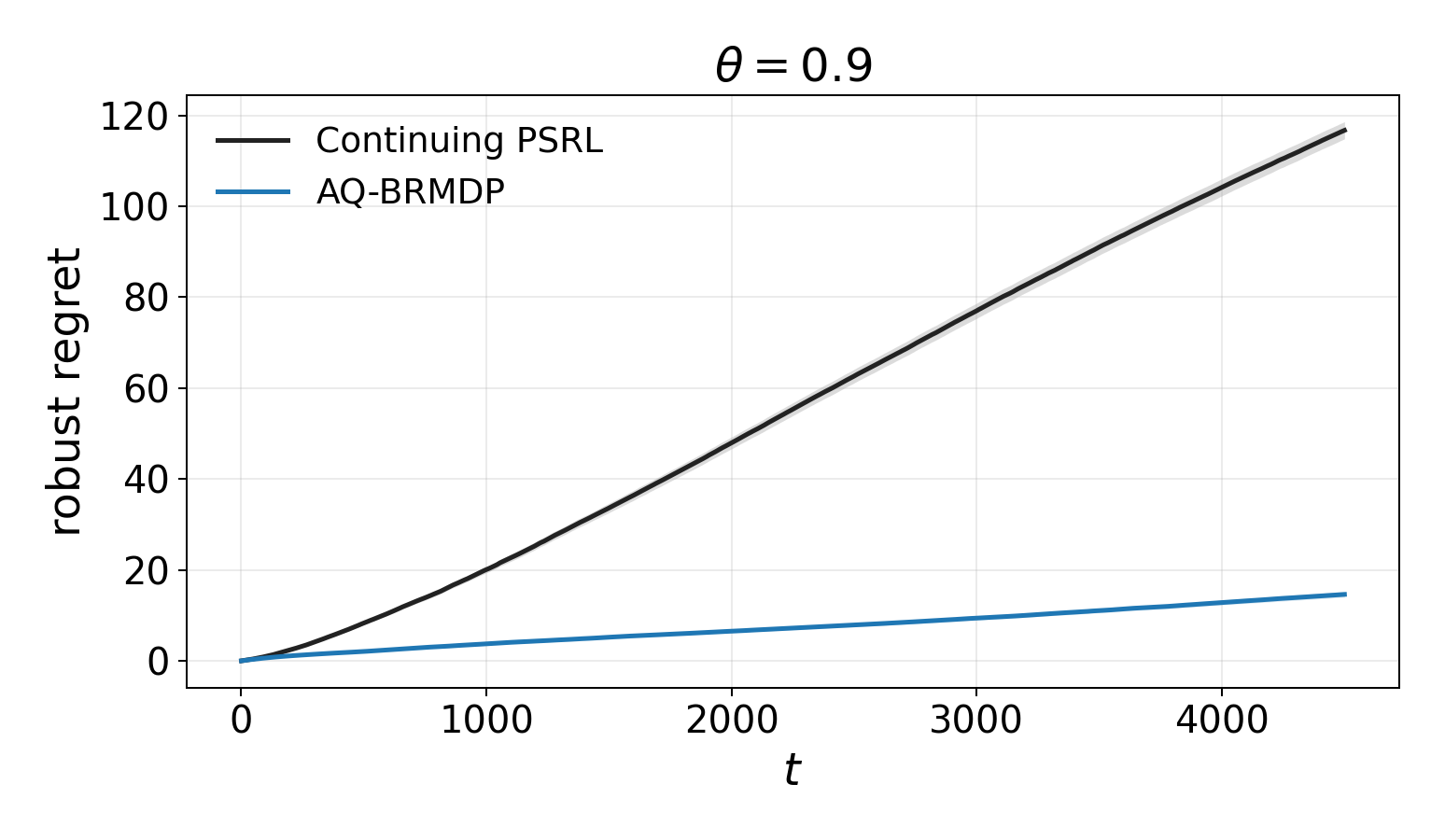}
\caption{Cumulative true regret (left column) and cumulative robust regret with $\underline{\alpha}=0.2$ (right column) in risky-branch FrozenLake for $\theta=0.7$ and $\theta=0.9$.}
\label{fig:fl-regret}
\end{figure*}

\begin{figure*}[!t]
\centering
\includegraphics[width=0.46\textwidth]{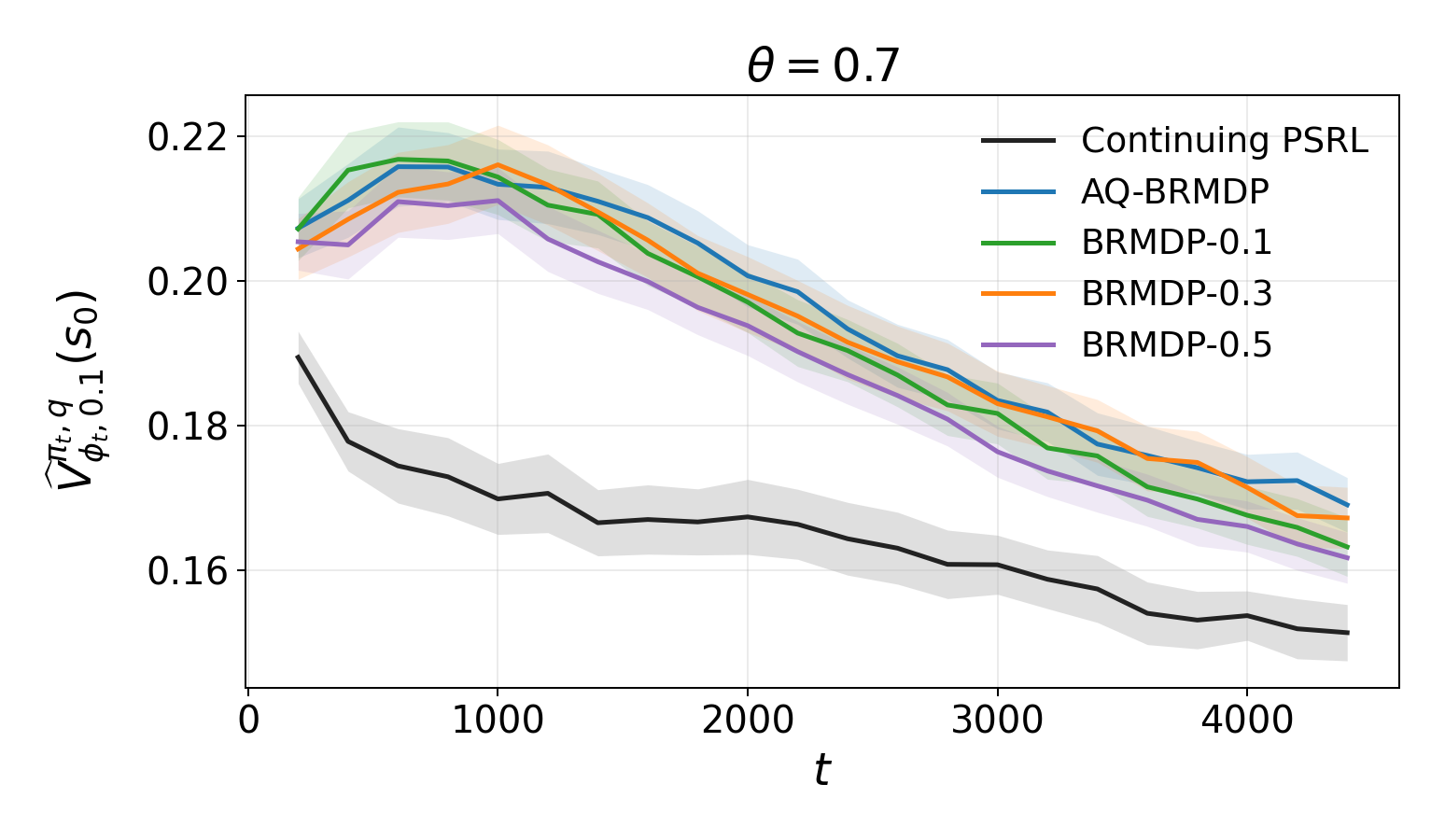}
\includegraphics[width=0.46\textwidth]{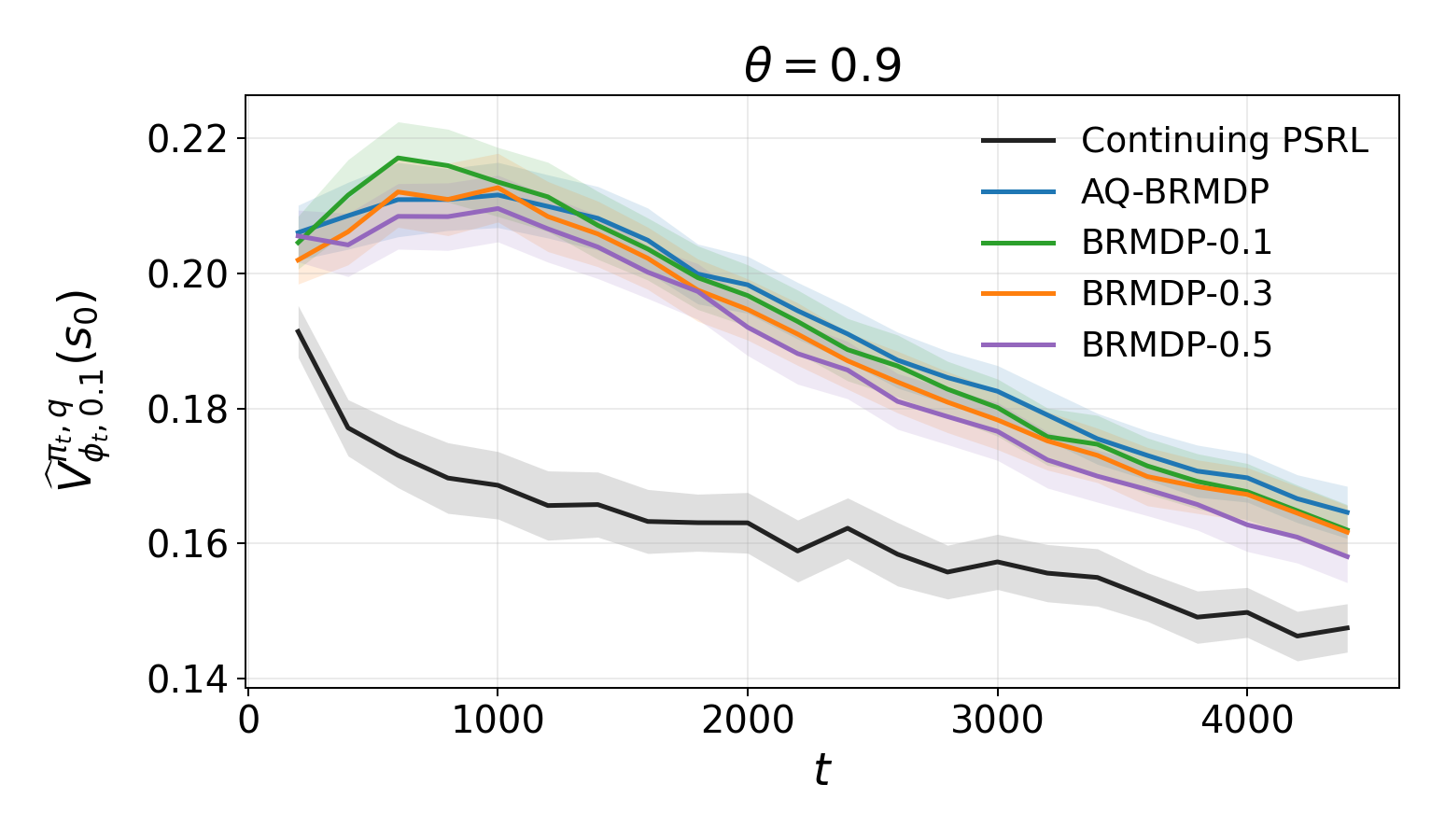}
\caption{Posterior $0.1$-quantile value $\widehat V^{\pi_t,\mathrm q}_{\phi_t,0.1}(s_0)$ in risky-branch FrozenLake for $\theta=0.7$ and $\theta=0.9$.}
\label{fig:fl-quantile}
\end{figure*}

\section{Conclusions}
This paper studies the $\alpha$-quantile BR-MDP from both modeling and algorithmic perspectives. From a modeling perspective, we formulate the $\alpha$-quantile BR-MDP, in which the quantile level provides a flexible way to adjust the risk attitude toward epistemic uncertainty. The formulation can induce robust or exploratory behavior. 
We further characterize this flexibility through an asymptotic normality result: for a fixed policy, the $\alpha$-quantile BR-MDP value function differs from the true value function by a mean term whose sign is determined by whether $\alpha$ is above or below $1/2$, and whose magnitude increases as $\alpha$ moves farther into either tail and shrinks as the posterior concentrates. The posterior quantile value result further provides a finite-sample robustness interpretation for lower-tail quantile evaluation.
From an algorithmic perspective, we develop AQ-BRMDP to account for the evolving robustness--exploration trade-off in online RL. The adaptive schedule varies with the pseudo-episode index and relative state-action visit counts, allowing the policy to be more conservative early in learning and less conservative when further exploration is needed. We theoretically prove a sublinear Bayesian regret bound and numerically demonstrate that AQ-BRMDP performs effectively in both exploration-demanding and exploration-costly environments. We also implement an extension of the proposed algorithm to continuous-state spaces and evaluate its empirical performance in a continuous-state environment.

Several directions remain open. 
The present analysis is developed for finite state and action spaces with direct transition parameterization and Dirichlet posteriors.
Extending the formulation and the theoretical guarantees to continuous-state and continuous-action spaces would broaden the scope of the approach. 
On the computational side, the method relies on repeated posterior sampling and solving the quantile BR-MDP at the beginning of each pseudo-episode. 
Although sampling transition kernels from the Dirichlet posterior is carried out in the simulator rather than through the more costly online interaction with the true environment, it can still materially increase computational cost. 
It would therefore be useful to understand how previously sampled transition kernels can be reused.

\section*{Acknowledgments}
This work was supported by the Air Force Office of Scientific Research (AFOSR) under Grant FA9550-25-1-0310 and the National Science Foundation under Award ECCS-2419562.

\bibliography{brrl}

\clearpage
\appendix
\renewcommand{\thesection}{EC.\arabic{section}}
\renewcommand{\theequation}{EC.\arabic{equation}}
\renewcommand{\thefigure}{EC.\arabic{figure}}
\renewcommand{\thetable}{EC.\arabic{table}}
\renewcommand{\thealgorithm}{EC.\arabic{algorithm}}
\makeatletter
\renewcommand{\theHsection}{EC.\arabic{section}}
\renewcommand{\theHsubsection}{EC.\arabic{section}.\arabic{subsection}}
\renewcommand{\theHsubsubsection}{EC.\arabic{section}.\arabic{subsection}.\arabic{subsubsection}}
\renewcommand{\theHequation}{EC.\arabic{equation}}
\renewcommand{\theHfigure}{EC.\arabic{figure}}
\renewcommand{\theHtable}{EC.\arabic{table}}
\renewcommand{\theHalgorithm}{EC.\arabic{algorithm}}
\providecommand{\theHALG@line}{}
\renewcommand{\theHALG@line}{EC.\arabic{algorithm}.\arabic{ALG@line}}
\makeatother
\setcounter{section}{0}
\setcounter{equation}{0}
\setcounter{figure}{0}
\setcounter{table}{0}
\setcounter{algorithm}{0}
\section*{Appendix: Proofs and Implementation Details}
\addcontentsline{toc}{section}{Appendix: Proofs and Implementation Details}

\setlength{\abovedisplayskip}{4pt plus 2pt minus 2pt}
\setlength{\belowdisplayskip}{4pt plus 2pt minus 2pt}
\setlength{\abovedisplayshortskip}{2pt plus 1pt minus 1pt}
\setlength{\belowdisplayshortskip}{3pt plus 1pt minus 1pt}
\setlength{\floatsep}{8pt plus 2pt minus 2pt}
\setlength{\textfloatsep}{8pt plus 2pt minus 2pt}
\setlength{\intextsep}{8pt plus 2pt minus 2pt}
\setlength{\dblfloatsep}{8pt plus 2pt minus 2pt}
\setlength{\dbltextfloatsep}{8pt plus 2pt minus 2pt}
\setlength{\abovecaptionskip}{3pt}
\setlength{\belowcaptionskip}{0pt}


\section{Dirichlet Posterior Update for the Transition Kernel}
\label{app:dirichlet-update}

Let $\Delta^{S}:=\{p\in\mathbb R_+^S:\sum_{s'\in\mathcal S}p(s')=1\}$ denote the probability simplex. For a parameter vector $\eta\in\mathbb R_{++}^S$, let $\Dir(\eta)$ denote the Dirichlet distribution on $\Delta^S$ with density
$$
f(p\mid \eta)\propto \prod_{s'\in\mathcal S}p(s')^{\eta(s')-1}.
$$
For each state-action pair $(s,a)$, we place the prior $P_{s,a}\sim \Dir(\phi_0(s,a)),$
where $\phi_0(s,a)=(\phi_0(s,a,s'))_{s'\in\mathcal S}$ is the prior parameter vector with $\phi_0(s,a,s')>0$. The prior parameter $\phi_0(s,a)$ can be chosen based on prior knowledge; in the absence of such information, a common choice is the uniform Dirichlet prior $\phi_0(s,a,s')=1$ for all $s'\in\mathcal S$.

Let $N(s,a,s')$ be the number of observed transitions from $(s,a)$ to $s'$, let $N(s,a):=\sum_{s'\in\mathcal S}N(s,a,s')$ denote the number of visits to $(s,a)$, and let $N:=\sum_{s,a} N(s,a)$ denote the total number of observations. Define the $\sigma$-field generated by the observed transitions as $\mathcal F_N=\sigma\{(s_i,a_i,s_{i+1}):i=0,\ldots,N-1\}$. By conjugacy, the posterior is
$$
P_{s,a}\mid \mathcal F_N \sim \Dir(\phi_N(s,a)),
$$
where
\begin{equation}
\label{eq:update-of-phi} \phi_N(s,a,s')=\phi_0(s,a,s')+N(s,a,s'), \qquad \forall s'\in\mathcal S.
\end{equation}
Thus, the posterior parameter $\phi_N$ is determined by the prior parameter $\phi_0$ and the transition counts induced by the observed trajectory. In the main text, when the current posterior is fixed, we suppress the dependence on $N$ and write $\phi(s,a)$ for $\phi_N(s,a)$.

\section{Proof of Asymptotic Normality}
\label{sec:Proof of Weak Convergence}

\begin{lemma}[Basic properties of the quantile functional]
\label{lem:quantile_properties}
Let $\rho^\alpha$ be the left $\alpha$-quantile functional, $\alpha\in(0,1)$.
Then for any real random variables $X,Y$:
\begin{enumerate}
    \item[(i)] for any constant $c\in\mathbb R$, $\rho^\alpha(X+c)=\rho^\alpha(X)+c$;
    \item[(ii)] for any constant $c>0$, $\rho^\alpha(cX)=c\,\rho^\alpha(X)$;
    \item[(iii)] if $X\ge Y$ almost surely, then $\rho^\alpha(X)\ge \rho^\alpha(Y)$;
    \item[(iv)] if $|X-Y|\le c$ almost surely for some $c\ge 0$, then $|\rho^\alpha(X)-\rho^\alpha(Y)|\le\|X-Y\|_\infty\le c$.
\end{enumerate}
\end{lemma}

Let $(\Omega,\mathcal F,\mathbb P)$ be the underlying probability space on which the observed transition process is defined.
For each $N$, let $\mathcal D_N:=\{s_0,a_0,s_1,\ldots,s_{N-1},a_{N-1},s_N\}$ be the trajectory data in \assref{ass:quantile_wc}{ass:quantile_wc_1}, and set $\mathcal F_N:=\sigma(\mathcal D_N)$. For the martingale argument below, write $\mathcal F_i:=\sigma(s_0,a_0,\ldots,s_i,a_i)$ for the history before observing $s_{i+1}$, $i=0,\ldots,N-1$.
Throughout this proof, $O_p(\cdot)$, $o_p(\cdot)$, and $\xrightarrow{p}$ are understood with respect to the randomness of the observed data under $\mathbb P$ as $N\to\infty$.

For any integer $m\ge 2$ and any probability vector $p\in\mathbb R_+^m$ with $\sum_{i=1}^m p_i=1$, define
$\Sigma(p)\ :=\ \mathrm{diag}(p)-pp^\top.$
{Define the transition counts} $N(s,a,s')
:=
\sum_{i=0}^{N-1}
\mathbb I\{s_i=s,a_i=a,s_{i+1}=s'\},$ and then $N(s,a)=\sum_{s'\in\mathcal S}N(s,a,s').$
For each $(s,a)$ {with $N(s,a)>0$}, define $\widetilde P_N(s,a)(s')=\frac{N(s,a,s')}{N(s,a)}.$
{Under the uniform Dirichlet prior, the posterior parameter satisfies $\phi(s,a,s')=N(s,a,s')+1$.}

Here and below, for a real-valued posterior random variable $X_N$ and a real-valued random variable $X$ with cdf $F$, we write
\begin{align*}
X_N\mid {\mathcal F_N} \Rightarrow X \quad\text{in }\mathbb P\text{-probability}
\end{align*}
to mean conditional weak convergence in probability. More precisely, for every bounded continuous test function
$\psi:\mathbb R\to\mathbb R$,
\begin{align*}
\mathbb E[\psi(X_N)\mid {\mathcal F_N}] - \mathbb E[\psi(X)] \xrightarrow{\mathbb P}0.
\end{align*}
The conditional expectation is taken with respect to the posterior distribution given {$\mathcal F_N$}, whereas the convergence in probability is with respect to the randomness of the observed data.

Lemma~\ref{lem:dirichlet_posterior_clt} is a fixed-dimensional Dirichlet special case of the Bernstein--von Mises theorem for discrete probability distributions
\citep{boucheron2009bernstein}; we restate it in the form needed here. 

\begin{lemma}[Posterior CLT for a Dirichlet transition vector]
\label{lem:dirichlet_posterior_clt}
For $(s,a)\in\mathcal S\times\mathcal A$,
suppose $P_{s,a}\mid {\mathcal F_N} \sim \Dir\bigl(\phi(s,a)\bigr).$
Then, conditionally on {$\mathcal F_N$}, as $N(s,a)\to\infty$,
\begin{align*}
\sqrt{N(s,a)}\bigl(P_{s,a}-\widetilde P_N(s,a)\bigr) \Rightarrow \mathcal N\!\bigl(0,\Sigma(P^c_{s,a})\bigr)
\end{align*}
in $\mathbb P$-probability.
\end{lemma}

\begin{proof}
    The Bernstein--von Mises (BvM) theorem applies directly to the coordinates
\(J_+:=\{x:P^c_{s,a}(x)>0\}\), on which the true transition vector lies in the
relative interior of the simplex \citep{boucheron2009bernstein}. For \(x\in J_0:=\{x:P^c_{s,a}(x)=0\}\), no
transition to \(x\) is observed almost surely, so \(N(s,a,x)=0\). Since the prior
parameters are fixed and positive, $\mathbb E[P_{s,a}(x)\mid\mathcal F_N]
=
O(N(s,a)^{-1}),$
and hence \(P_{s,a}(x)=O_p(N(s,a)^{-1})\). Therefore
\[
\sqrt{N(s,a)}
\bigl(P_{s,a}(x)-\widetilde P_N(s,a)(x)\bigr)
=
\sqrt{N(s,a)}P_{s,a}(x)
\xrightarrow{p}0.
\]
Combining the BvM limit on \(J_+\) with this degenerate limit on \(J_0\)
gives the stated Gaussian limit with covariance
\(\Sigma(P^c_{s,a})\), possibly singular.

\end{proof}

In the following lemma, we show that conditional weak convergence of random variables implies convergence of the corresponding conditional quantiles.
\begin{lemma}
\label{lem:conditional_weak_to_quantile}
Let $(Z_N)_{N\ge1}$ be real-valued random variables on $(\Omega,\mathcal F,\mathbb P)$, and let
$(\mathcal F_N)_{N\ge1}$ be a sequence of sub-$\sigma$-fields of $\mathcal F$.
For each $N$, define the conditional cdf of $Z_N$ given $\mathcal F_N$ by
\begin{align*}
F_N(x):= {\mathbb P(Z_N\le x\mid \mathcal F_N)},\qquad x\in\mathbb R,
\end{align*}
Assume that there exists a real-valued random variable $Z$ with cdf $F$ such that $Z_N\mid {\mathcal F_N} \Rightarrow Z
\text{ in }\mathbb P\text{-probability}$.
%
For $\alpha\in(0,1)$, define
\begin{align*}
q_N:=\inf\{x\in\mathbb R:F_N(x)\ge \alpha\}, \qquad q_\alpha:=\inf\{x\in\mathbb R:F(x)\ge \alpha\}.
\end{align*}
If $F$ is continuous and strictly increasing on a neighborhood of $q_\alpha$, then
\begin{align*}
q_N\xrightarrow{p}q_\alpha.
\end{align*}
In particular, if $F=\Phi$ is the standard normal cdf, then
$q_\alpha=\Phi^{-1}(\alpha)=z_\alpha$, and hence $q_N\xrightarrow{p}z_\alpha$.
\end{lemma}

\begin{proof}
Fix a continuity point $x$ of $F$ and $\delta>0$. For $\varepsilon>0$, define
\begin{align*}
\psi^-_{x,\varepsilon}(t):=
\begin{cases}
1, & t\le x-\varepsilon,\\
1-\dfrac{t-(x-\varepsilon)}{\varepsilon}, & x-\varepsilon<t<x,\\
0, & t\ge x,
\end{cases}
\qquad
\psi^+_{x,\varepsilon}(t):=
\begin{cases}
1, & t\le x,\\
1-\dfrac{t-x}{\varepsilon}, & x<t<x+\varepsilon,\\
0, & t\ge x+\varepsilon.
\end{cases}
\end{align*}
Then $\psi^-_{x,\varepsilon}\le \mathbf 1\{t\le x\}\le \psi^+_{x,\varepsilon}$, so
\begin{align*}
{\mathbb E[\psi^-_{x,\varepsilon}(Z_N)\mid \mathcal F_N] \le F_N(x) \le \mathbb E[\psi^+_{x,\varepsilon}(Z_N)\mid \mathcal F_N].}
\end{align*}
By the assumed conditional weak convergence,
\begin{align*}
{\mathbb E[\psi^\pm_{x,\varepsilon}(Z_N)\mid \mathcal F_N] \xrightarrow{p} \mathbb E[\psi^\pm_{x,\varepsilon}(Z)].}
\end{align*}
Since $F$ is continuous at $x$,
$\mathbb E[\psi^-_{x,\varepsilon}(Z)]\uparrow F(x)$ and
$\mathbb E[\psi^+_{x,\varepsilon}(Z)]\downarrow F(x)$ as $\varepsilon\downarrow0$.
Choose $\varepsilon>0$ so that
\begin{align*}
F(x)-\frac{\delta}{2} \;\le\; \mathbb E[\psi^-_{x,\varepsilon}(Z)] \;\le\; F(x) \;\le\; \mathbb E[\psi^+_{x,\varepsilon}(Z)] \;\le\; F(x)+\frac{\delta}{2}.
\end{align*}
Then
\begin{align*}
\mathbb P\!\left(F_N(x)<F(x)-\delta\right) \le \mathbb P\!\left( \mathbb E[\psi^-_{x,\varepsilon}(Z_N)\mid\mathcal F_N] < \mathbb E[\psi^-_{x,\varepsilon}(Z)]-\frac{\delta}{2} \right)\to0,
\end{align*}
and similarly,
\begin{align*}
\mathbb P\!\left(F_N(x)>F(x)+\delta\right) \le \mathbb P\!\left( \mathbb E[\psi^+_{x,\varepsilon}(Z_N)\mid\mathcal F_N] > \mathbb E[\psi^+_{x,\varepsilon}(Z)]+\frac{\delta}{2} \right)\to0.
\end{align*}
Hence
\begin{align*}
F_N(x)\xrightarrow{p}F(x).
\end{align*}
For $\varepsilon>0$ small enough that
$F(q_\alpha-\varepsilon)<\alpha<F(q_\alpha+\varepsilon)$, which is possible because
$F$ is continuous and strictly increasing on a neighborhood of $q_\alpha$, we have
\begin{align*}
F_N(q_\alpha-\varepsilon)\xrightarrow{p}F(q_\alpha-\varepsilon), \qquad F_N(q_\alpha+\varepsilon)\xrightarrow{p}F(q_\alpha+\varepsilon).
\end{align*}
Hence, $\mathbb P(F_N(q_\alpha-\varepsilon)<\alpha<F_N(q_\alpha+\varepsilon))\rightarrow 1.$
By the definition of the left quantile, this implies
$\mathbb P(q_\alpha-\varepsilon<q_N\le q_\alpha+\varepsilon)\rightarrow1$. Therefore
$q_N\xrightarrow{p}q_\alpha$.
\end{proof}

\begin{lemma}[Posterior quantile expansion for bounded linear forms]
\label{lem:quantile_expansion}
For $(s,a)\in\mathcal S\times\mathcal A$, suppose $N(s,a)\to\infty$. Let $v_N$ be an $\mathcal F_N$-measurable random vector such that
$\|v_N\|_\infty\le B$ almost surely for some constant $B<\infty$ and assume that \(v_N\xrightarrow{p}v\) for some deterministic vector \(v\in\mathbb R^S\).
Define
{$\sigma^{2}(s,a;v_N):=v_N^\top \Sigma(P^c_{s,a})v_N$}.
Then
\begin{align*}
\rho_{\phi(s,a)}^\alpha\!\bigl(P^\top v_N\bigr) = \widetilde P_N(s,a)^\top v_N + \frac{z_\alpha}{\sqrt{N(s,a)}}\,\sigma(s,a;v_N) + o_p\bigl(N(s,a)^{-1/2}\bigr).
\end{align*}
\end{lemma}

\begin{proof}
Conditionally on $\mathcal F_N$, the vector $v_N$ is deterministic.
If \(\sigma(s,a;v)=0\), then \(v\) is constant on the support of \(P^c_{s,a}\), and Lemma~\ref{lem:dirichlet_posterior_clt} together with \(v_N\xrightarrow{p}v\) implies \(P^\top v_N-\widetilde P_N(s,a)^\top v_N=o_p(N(s,a)^{-1/2})\); moreover \(\sigma(s,a;v_N)\to0\), so the variance term is also \(o_p(N(s,a)^{-1/2})\).
Thus it remains to consider the case \(\sigma(s,a;v)>0\).

By Lemma~\ref{lem:dirichlet_posterior_clt} and the continuous mapping theorem,
\begin{align*}
Z_N:= \frac{\sqrt{N(s,a)}\bigl(P^\top v_N-\widetilde P_N(s,a)^\top v_N\bigr)} {\sigma(s,a;v_N)} \mid \mathcal F_N \Rightarrow \mathcal N(0,1)
\end{align*}
in $\mathbb P$-probability.
Let $F_N(x):=\mathbb P(Z_N\le x\mid \mathcal F_N)$, $x\in\mathbb R$, denote the posterior conditional cdf of $Z_N$ given $\mathcal F_N$.
Applying Lemma~\ref{lem:conditional_weak_to_quantile} with $F=\Phi$, we obtain
\begin{align*}
\rho^\alpha(Z_N)=z_\alpha+o_p(1).
\end{align*}
Now write
\begin{align*}
P^\top v_N = \widetilde P_N(s,a)^\top v_N + \frac{\sigma(s,a;v_N)}{\sqrt{N(s,a)}}\,Z_N.
\end{align*}
Using Lemma~\ref{lem:quantile_properties}(i)--(ii),
\begin{align*}
\rho_{\phi(s,a)}^\alpha\!\bigl(P^\top v_N\bigr) = \widetilde P_N(s,a)^\top v_N + \frac{\sigma(s,a;v_N)}{\sqrt{N(s,a)}} \bigl(z_\alpha+o_p(1)\bigr).
\end{align*}
Finally, since
$\sigma(s,a;v_N)^2=v_N^\top\Sigma(P^c_{s,a})v_N$
$\le \|v_N\|_\infty^2\le B^2$ almost surely,
we have $\sigma(s,a;v_N)=O_p(1)$, and therefore
\begin{align*}
\frac{\sigma(s,a;v_N)}{\sqrt{N(s,a)}}\,o_p(1) = o_p\bigl(N(s,a)^{-1/2}\bigr).
\end{align*}
This proves the claim.
\end{proof}

\begin{lemma}[Martingale CLT for empirical transition frequencies]
\label{lem:posterior_mean_clt}
{Under Assumption~\ref{ass:quantile_wc},}
\begin{align*}
\left( \sqrt N \bigl(\widetilde P_N(s,\pi(s))-P^c_{s,\pi(s)}\bigr)^\top V^\pi \right)_{s\in\mathcal S} \Rightarrow \mathcal N\!\left( 0,\, \diag\bigl(\sigma_\pi^2(s)\bigr)_{s\in\mathcal S} \right),
\end{align*}
{where}
\begin{align*}
\sigma_\pi^2(s) = \frac{1}{\bar n_s} (V^\pi)^\top \Sigma(P^c_{s,\pi(s)}) V^\pi.
\end{align*}
{Moreover, for each $s\in\mathcal S$, $\|\widetilde P_N(s,\pi(s))-P^c_{s,\pi(s)}\|_\infty=O_p(N(s,\pi(s))^{-1/2})=O_p(N^{-1/2})$.}
\end{lemma}

\begin{proof}
For each $s\in\mathcal S$, define
\begin{align*}
 D_{i+1}^{s} := \mathbb I\{s_i=s,a_i=\pi(s)\} \left( V^\pi(s_{i+1}) - (P^c_{s,\pi(s)})^\top V^\pi \right), \qquad i=0,\ldots,N-1. 
\end{align*}
By \assref{ass:quantile_wc}{ass:quantile_wc_1}, $\mathbb E_{s_{i+1}\sim P^c_{s,\pi(s)}}[D_{i+1}^{s}\mid\mathcal F_i]=0$. Hence $M_N^s:=\sum_{i=0}^{N-1}D_{i+1}^{s}$ is a sum of martingale differences.
For $s,\tilde s\in\mathcal S$, the conditional covariance satisfies
\begin{align*}
 \sum_{i=0}^{N-1} \mathbb E\!\left[ D_{i+1}^{s}D_{i+1}^{\tilde s} \mid \mathcal F_i \right] = \mathbb I\{s=\tilde s\} N(s,\pi(s)) (V^\pi)^\top\Sigma(P^c_{s,\pi(s)})V^\pi. 
\end{align*}
Indeed, if $s\neq \tilde s$, the two indicators cannot both be one at the same time. Dividing by $N$ and using \assref{ass:quantile_wc}{ass:quantile_wc_1},
\begin{align*}
 \frac{1}{N} \sum_{i=0}^{N-1} \mathbb E\!\left[ D_{i+1}^{s}D_{i+1}^{\tilde s} \mid \mathcal F_i \right] \xrightarrow{\mathrm{a.s.}} \mathbb I\{s=\tilde s\} \bar n_s (V^\pi)^\top\Sigma(P^c_{s,\pi(s)})V^\pi. 
\end{align*}
Since the state space is finite and $V^\pi$ is bounded, the increments are uniformly bounded, so the conditional Lindeberg condition holds. The multivariate martingale CLT gives
\begin{align*}
 \left( \frac{M_N^s}{\sqrt N} \right)_{s\in\mathcal S} \Rightarrow \mathcal N\!\left( 0,\, \diag\left( \bar n_s (V^\pi)^\top\Sigma(P^c_{s,\pi(s)})V^\pi \right)_{s\in\mathcal S} \right). 
\end{align*}
Finally, note that $ M_N^s = N(s,\pi(s)) \bigl(\widetilde P_N(s,\pi(s))-P^c_{s,\pi(s)}\bigr)^\top V^\pi. $
Therefore,
\begin{align*}
 \sqrt N \bigl(\widetilde P_N(s,\pi(s))-P^c_{s,\pi(s)}\bigr)^\top V^\pi = \frac{N}{N(s,\pi(s))} \frac{M_N^s}{\sqrt N}.
\end{align*}
Slutsky's theorem and \assref{ass:quantile_wc}{ass:quantile_wc_1} imply the stated joint convergence.

It remains to show the rate bound. For each $s,x\in\mathcal S$, define
\begin{align*}
 D_{i+1}^{s,x} := \mathbb I\{s_i=s,a_i=\pi(s)\} \left( \mathbb I\{s_{i+1}=x\} - P^c(x\mid s,\pi(s)) \right). 
\end{align*}
The same martingale argument gives $\sum_{i=0}^{N-1}D_{i+1}^{s,x}=O_p(N^{1/2})$. Since
\begin{align*}
 \widetilde P_N(s,\pi(s))(x)-P^c(x\mid s,\pi(s)) = \frac{1}{N(s,\pi(s))} \sum_{i=0}^{N-1}D_{i+1}^{s,x}, 
\end{align*}
and $N(s,\pi(s))\asymp N$ by \assref{ass:quantile_wc}{ass:quantile_wc_1}, we obtain $\widetilde P_N(s,\pi(s))(x)-P^c(x\mid s,\pi(s))=O_p(N^{-1/2})$. The state space is finite, so the same bound holds in sup norm.
\end{proof}

Next, we are ready to prove Theorem \ref{thm:quantile_weak}.
\begin{proof}[Proof of Theorem~\ref{thm:quantile_weak}]
Let $\Delta_N:={V_{\phi,\alpha}^{\pi}} -V^\pi.$
For each state $s\in\mathcal S$, we define the posterior quantile map, for notational simplicity, $q_{N,s}(v) := \rho_{\phi(s,\pi(s))}^{\alpha}\!\bigl(P^\top v\bigr).$

Recall that ${V_{\phi,\alpha}^{\pi}} $ and $V^\pi$ are the unique fixed points of \eqref{eq:VR-pi} and \eqref{eq:V-true-pi}, respectively. Hence we have the two fixed-point equations:
\begin{align*}
{V_{\phi,\alpha}^{\pi}} (s) = r(s,\pi(s))+\gamma q_{N,s}({V_{\phi,\alpha}^{\pi}} ), \qquad V^\pi(s) = r(s,\pi(s))+\gamma (P^c_{s,\pi(s)})^\top V^\pi.
\end{align*}
For each $s\in\mathcal S$,
\begin{align*}
\Delta_N(s) &= \gamma\Big(q_{N,s}({V_{\phi,\alpha}^{\pi}} )-(P^c_{s,\pi(s)})^\top V^\pi\Big)\\
&= \gamma\Big(q_{N,s}(V^\pi)-(P^c_{s,\pi(s)})^\top V^\pi\Big) +\gamma (P^c_{s,\pi(s)})^\top \Delta_N\\
&\quad +\gamma\Big( q_{N,s}({V_{\phi,\alpha}^{\pi}} )-q_{N,s}(V^\pi) -(P^c_{s,\pi(s)})^\top \Delta_N \Big).
\end{align*}
Therefore, componentwise,
\begin{equation}\label{eq:main_identity_quantile_clean}
\bigl((I-\gamma P_\pi^c)\Delta_N\bigr)(s)
= \underbrace{\gamma\Big(
q_{N,s}(V^\pi)-(P^c_{s,\pi(s)})^\top V^\pi
\Big)}_{B_N(s)}+\underbrace{\gamma\Big(
q_{N,s}({V_{\phi,\alpha}^{\pi}} )-q_{N,s}(V^\pi)
-(P^c_{s,\pi(s)})^\top \Delta_N
\Big)}_{R_N(s)}.
\end{equation}

We next identify the weak limit of $B_N$.
Fix $s\in\mathcal S$ and $a=\pi(s)$.
Define $\widetilde\sigma_\pi^{2}(s)
:=
(V^\pi)^\top \Sigma(P^c_{s,a})V^\pi.$
Lemma~\ref{lem:quantile_expansion} with $v_N\equiv V^\pi$ yields
\begin{equation}\label{eq:q_at_Vpi_expansion}
q_{N,s}(V^\pi) = { \widetilde P_N(s,a)^\top V^\pi + \frac{z_\alpha}{\sqrt{N(s,a)}}\,\widetilde\sigma_\pi(s) } + o_p\bigl(N(s,a)^{-1/2}\bigr),
\end{equation}

{Combining \eqref{eq:q_at_Vpi_expansion} with \assref{ass:quantile_wc}{ass:quantile_wc_1},}
\begin{align*}
\sqrt N \Big( q_{N,s}(V^\pi)-(P^c_{s,\pi(s)})^\top V^\pi \Big) = \sqrt N \bigl(\widetilde P_N(s,\pi(s))-P^c_{s,\pi(s)}\bigr)^\top V^\pi + z_\alpha \sqrt{\frac{N}{N(s,\pi(s))}}\, \widetilde\sigma_\pi(s) + o_p(1).
\end{align*}
By Lemma~\ref{lem:posterior_mean_clt}, the first term on the right-hand side converges jointly to a centered normal vector. Since $N(s,\pi(s))/N\xrightarrow{\mathrm{a.s.}}\bar n_s$, Slutsky's theorem gives
\begin{align*}
\left( \sqrt N \Big( q_{N,s}(V^\pi)-(P^c_{s,\pi(s)})^\top V^\pi \Big) \right)_{s\in\mathcal S} \Rightarrow \mathcal N\!\left( \lambda_\pi,\, \diag\bigl(\sigma_\pi^2(s)\bigr)_{s\in\mathcal S} \right),
\end{align*}
{where}
\begin{align*}
\sigma_\pi^2(s) = \frac{(\widetilde\sigma_\pi(s))^2}{\bar n_s} = \frac{1}{\bar n_s} (V^\pi)^\top \Big( \diag(P^c_{s,\pi(s)}) - P^c_{s,\pi(s)}(P^c_{s,\pi(s)})^\top \Big)V^\pi, \qquad \lambda_\pi(s):=z_\alpha\sigma_\pi(s).
\end{align*}
{Therefore,}
\begin{equation}\label{eq:BN_joint_limit}
\sqrt N\,B_N \Rightarrow \mathcal N\!\Bigl( \gamma\lambda_\pi, \diag\bigl((\gamma\sigma_\pi)^2\bigr) \Bigr).
\end{equation}

We now derive the rate of $\Delta_N$ itself.
From the fixed-point equations {\eqref{eq:VR-pi} and \eqref{eq:V-true-pi}},
\begin{align*}
\Delta_N(s) = \gamma\Big( q_{N,s}({V_{\phi,\alpha}^{\pi}} )-q_{N,s}(V^\pi) \Big) + B_N(s).
\end{align*}
Taking sup norms and using Lemma~\ref{lem:quantile_properties}(iv), $\|\Delta_N\|_\infty
\le
\gamma\|\Delta_N\|_\infty+\|B_N\|_\infty.$
Hence $\|\Delta_N\|_\infty
\le
\frac{1}{1-\gamma}\|B_N\|_\infty.$
Since \eqref{eq:BN_joint_limit} implies $\sqrt N\,B_N$ is tight,
\begin{align}
\label{eq:DeltaN_rate_bound} \|\Delta_N\|_\infty=O_p(N^{-1/2}).
\end{align}

It remains to show that the remainder $R_N$ is negligible. 
Fix again $s\in\mathcal S$ and set $a=\pi(s)$.
Define, for any $v\in\mathbb R^{S}$, $\sigma^2(s,a;v):=v^\top \Sigma(P^c_{s,a})v$.
By \eqref{eq:DeltaN_rate_bound}, \(V^\pi_{\phi,\alpha}\xrightarrow{p}V^\pi\), so the convergence condition in Lemma~\ref{lem:quantile_expansion} is satisfied.
Applying Lemma~\ref{lem:quantile_expansion} twice, first with $v_N={V_{\phi,\alpha}^{\pi}}$,
then with $v_N\equiv V^\pi,$ we have 
\begin{align}
q_{N,s}({V_{\phi,\alpha}^{\pi}} )-q_{N,s}(V^\pi) &= \widetilde P_N(s,a)^\top \Delta_N \notag\\
&\quad +\frac{z_\alpha}{\sqrt{N(s,a)}} \Big( \sigma(s,a;{V_{\phi,\alpha}^{\pi}} ) - \sigma(s,a;V^\pi) \Big) + o_p\bigl(N(s,a)^{-1/2}\bigr). \label{eq:q_difference_clean}
\end{align}
Substituting \eqref{eq:q_difference_clean} into \eqref{eq:main_identity_quantile_clean},
\begin{align}
R_N(s) &= \gamma\bigl(\widetilde P_N(s,a)-P^c_{s,a}\bigr)^\top \Delta_N \notag\\
&\quad +\frac{\gamma z_\alpha}{\sqrt{N(s,a)}} \Big( \sigma(s,a;{V_{\phi,\alpha}^{\pi}} ) - \sigma(s,a;V^\pi) \Big) + o_p\bigl(N(s,a)^{-1/2}\bigr). \label{eq:RN_expansion_clean}
\end{align}

We estimate the middle term in \eqref{eq:RN_expansion_clean}.
Let $A^c(s,a):=\Sigma(P^c_{s,a})$.
Since $A^c(s,a)$ is positive semidefinite, $\sigma(s,a;v)=\|A^c(s,a)^{1/2}v\|_2.$
Hence,
\begin{align*}
&\Big| \sigma(s,a;{V_{\phi,\alpha}^{\pi}} ) - \sigma(s,a;V^\pi) \Big|\\
&\qquad = \Big| \|A^c(s,a)^{1/2}{V_{\phi,\alpha}^{\pi}} \|_2 - \|A^c(s,a)^{1/2}V^\pi\|_2 \Big|\\
&\qquad \le \|A^c(s,a)^{1/2}({V_{\phi,\alpha}^{\pi}} -V^\pi)\|_2 \le \|A^c(s,a)^{1/2}\|_{\mathrm{op}}\|\Delta_N\|_2.
\end{align*}
Since $A^c(s,a)\preceq \diag(P^c_{s,a})\preceq I,$ $\|A^c(s,a)\|^{1/2}_{\mathrm{op}}\le 1.$
Therefore
\begin{align*}
\Big| \sigma(s,a;{V_{\phi,\alpha}^{\pi}} ) - \sigma(s,a;V^\pi) \Big| \le \|A^c(s,a)\|^{1/2}_{\mathrm{op}}\|\Delta_N\|_2 \le \sqrt{S}\,\|\Delta_N\|_\infty = O_p(N^{-1/2}).
\end{align*}
After division by $\sqrt{N(s,a)}$, the middle term in \eqref{eq:RN_expansion_clean} becomes $O_p(N^{-1}).$

For the first term in \eqref{eq:RN_expansion_clean},
Lemma~\ref{lem:posterior_mean_clt} gives
\begin{align*}
\|\widetilde P_N(s,a)-P^c_{s,a}\|_\infty = O_p(N(s,a)^{-1/2}) = O_p(N^{-1/2}),
\end{align*}
while \eqref{eq:DeltaN_rate_bound} gives $\|\Delta_N\|_\infty=O_p(N^{-1/2}).$
Hence
\begin{align*}
\bigl(\widetilde P_N(s,a)-P^c_{s,a}\bigr)^\top \Delta_N = O_p(N^{-1}).
\end{align*}
Since also $o_p\bigl(N(s,a)^{-1/2}\bigr)=o_p(N^{-1/2}),$ 
we conclude from \eqref{eq:RN_expansion_clean} that $R_N(s)=o_p(N^{-1/2})$, $\forall s\in\mathcal S.$
Because $S<\infty$,
\begin{equation}\label{eq:RN_small_clean}
\sqrt N\,R_N\xrightarrow{p}0.
\end{equation}

Finally, multiply \eqref{eq:main_identity_quantile_clean} by $\sqrt N$:
\begin{align*}
\sqrt N\,(I-\gamma P_\pi^c)\Delta_N = \sqrt N\,B_N+\sqrt N\,R_N.
\end{align*}
By \eqref{eq:BN_joint_limit}, \eqref{eq:RN_small_clean}, and Slutsky's theorem,
\begin{align*}
\sqrt N\,(I-\gamma P_\pi^c)\bigl({V_{\phi,\alpha}^{\pi}} -V^\pi\bigr) \Rightarrow \mathcal N\!\Bigl( \gamma\lambda_\pi, \diag\bigl((\gamma\sigma_\pi)^2\bigr) \Bigr).
\end{align*}
Since $P_\pi^c$ is row-stochastic, $\|\gamma P_\pi^c\|_\infty\le \gamma<1,$ so $(I-\gamma P_\pi^c)^{-1}$ exists and is bounded.
Applying the continuous mapping theorem to the linear map $(I-\gamma P_\pi^c)^{-1}$ yields
\begin{align*}
\sqrt N\,\bigl({V_{\phi,\alpha}^{\pi}} -V^\pi\bigr)
\Rightarrow
\mathcal N\!\Bigl(
(I-\gamma P_\pi^c)^{-1}\gamma\lambda_\pi,
\,
(I-\gamma P_\pi^c)^{-1}
\diag\bigl((\gamma\sigma_\pi)^2\bigr)
(I-\gamma P_\pi^c)^{-T}
\Bigr).
\end{align*}
This completes the proof.
\end{proof}
\section{Proof of Proposition~\ref{prop:hp-lb-no-subst}}
\label{app:hp-lb-no-subst}
\begin{proof}
Fix a policy $\pi$. By the definition of the left $\bar\alpha$-quantile, for each
$s\in\mathcal S$,
\begin{align*}
{\mathbb P_\phi}\Bigl( P_{s,\pi(s)}^\top V^\pi_{\phi,\bar\alpha} \ge \rho^{\bar\alpha}_{\phi(s,\pi(s))}\!\bigl(P^\top V^\pi_{\phi,\bar\alpha}\bigr) \Bigr) \ge 1-\bar\alpha.
\end{align*}
Using the Bellman equation for $V^\pi_{\phi,\bar\alpha}$, this implies
\begin{align*}
{\mathbb P_\phi}\Bigl( r(s,\pi(s)) + \gamma\,P_{s,\pi(s)}^\top V^\pi_{\phi,\bar\alpha} \ge V^\pi_{\phi,\bar\alpha}(s) \Bigr) \ge 1-\bar\alpha.
\end{align*}
Because the posterior rows $\left(P_{s,\pi(s)}\right)_{s\in\mathcal S}$ are independent across states {under \(\mathbb P_\phi\)},
the above events are independent. Hence,
\begin{align*}
{\mathbb P_\phi}\Bigl( r(s,\pi(s)) + \gamma\,P_{s,\pi(s)}^\top V^\pi_{\phi,\bar\alpha} \ge V^\pi_{\phi,\bar\alpha}(s), \quad \forall s\in\mathcal S \Bigr) \ge (1-\bar\alpha)^{S} = 1-\alpha.
\end{align*}
On this event,
\begin{align*}
V^\pi_{\phi,\bar\alpha}(s) \le r(s,\pi(s)) + \gamma\,P_{s,\pi(s)}^\top V^\pi_{\phi,\bar\alpha}, \qquad \forall s\in\mathcal S.
\end{align*}
Let $\mathcal T^\pi_{P}$ denote the standard Bellman operator under policy $\pi$ and
realized transition matrix $P$:
\begin{equation*}
(\mathcal T^\pi_{P}V)(s) := r(s,\pi(s)) + \gamma\,P_{s,\pi(s)}^\top V, \qquad \forall s\in\mathcal S.
\end{equation*}

The preceding event implies $V^\pi_{\phi,\bar\alpha}\le \mathcal T^\pi_{P}V^\pi_{\phi,\bar\alpha}$ componentwise. By monotonicity of $\mathcal T^\pi_{P}$, $V^\pi_{\phi,\bar\alpha}\le (\mathcal T^\pi_{P})^m V^\pi_{\phi,\bar\alpha},$ $\forall m\ge 1.$
Since $\mathcal T^\pi_{P}$ is a $\gamma$-contraction under $\|\cdot\|_\infty$, its iterates
converge to the unique fixed point $V^\pi_{P}$. Letting $m\to\infty$ yields $V^\pi_{\phi,\bar\alpha}\le V^\pi_{P}$ componentwise on the above event. Then the claim follows
immediately from the definition of 
{$V^{\pi,\mathrm q}_{\phi,\alpha}(s)$}
in \eqref{eq:posterior-lb}.
\end{proof}

\section{Evolving Robustness--Exploration Trade-off in Online RL:  Illustrative Examples}

\label{sec:robustness-exploration-tension}
In online RL, the balance between robustness and exploration evolves over the course of learning. When the posterior distribution is still dispersed and epistemic uncertainty is relatively large across all state--action pairs, robustness induced by lower-tail evaluation (lower quantile level) can be beneficial. As more data are collected, the remaining epistemic uncertainty about $P^c_{s,a}$ becomes concentrated in less-visited state--action pairs, which remain critical for learning the optimal policy in the true environment. In this regime, a higher quantile level encourages exploration of these state--action pairs that have higher epistemic uncertainty. The examples in this subsection separately illustrate these two mechanisms and show why a fixed lower-tail rule, i.e., choosing actions according to the optimal policy of the $\alpha$-quantile BR-MDP with a fixed $\alpha$, fails to learn the optimal policy.

We begin with a regime in which the balance tilts toward robustness. 
For any transition kernel $P$, let $V^*_P$ denote the optimal value function under $P$. Given a posterior distribution $\varphi$ over transition kernels, let $\bar P_\varphi:= \mathbb E_{P\sim\varphi}[P]$ denote the posterior-mean kernel. We denote by $\pi^*_{\varphi,\alpha}$ an optimal risk-aware policy of the $\alpha$-quantile BR-MDP under the posterior $\varphi$, and by $\pi^*_{\bar P_\varphi}$ an optimal risk-neutral policy under the posterior-mean kernel $\bar P_\varphi$.
The latter is obtained by replacing $P^c$ with $\bar P_\varphi$ in \eqref{eq:V-true-opt} and solving the resulting Bellman optimality equation. 
To make precise when robustness is needed in this regime, we introduce the notion of posterior downside exposure in the following definition.
\begin{definition}[Posterior downside exposure]
\label{def:posterior-downside}
Fix a posterior distribution $\varphi$ over transition kernels, an initial state $s_0$, a policy $\pi$, and a threshold $\Lambda>0$. Define the set of transition kernels under which the regret of policy $\pi$ is at least the threshold $\Lambda$ as follows:
$$
\mathcal D(\pi,\Lambda)
:=
\Bigl\{
P:\;
V^*_P(s_0)-V^\pi_P(s_0)\ge \Lambda
\Bigr\}.
$$
We say that $\pi$ has $(\beta,\Lambda)$-\emph{posterior downside exposure} under $\varphi$ if
$
\varphi\!\left(\mathcal D(\pi,\Lambda)\right)\ge \beta.
$
\end{definition}

Definition~\ref{def:posterior-downside} measures how much posterior mass is placed on transition kernels under which the regret of policy $\pi$ is at least $\Lambda$ from the initial state $s_0$.

\begin{example}[A posterior downside-exposure example]
\label{ex:robust-necessity}
Fix a discount factor $\gamma\in(0,1)$, a constant $c\in(0,\gamma)$, 
and a loss level $L>0$. Consider the discounted MDP in 
Figure~\ref{fig:robust-necessity-schematic}, where the reward is known to the 
agent but the transition kernel is unknown. The current posterior is supported 
on two kernels,
\[
\varphi=\mu\,\delta_{P^G}+(1-\mu)\,\delta_{P^B}, 
\qquad \mu\in(0,1),
\]
where $\delta_x$ denotes the Dirac measure at $x$, and $P^G$ and $P^B$ differ 
only in the transition following the risky action $a_R$ at the initial state 
$s_0$, as shown in the figure.
\end{example}

\begin{figure}[htbp]
\centering
\includegraphics[width=0.6\textwidth]{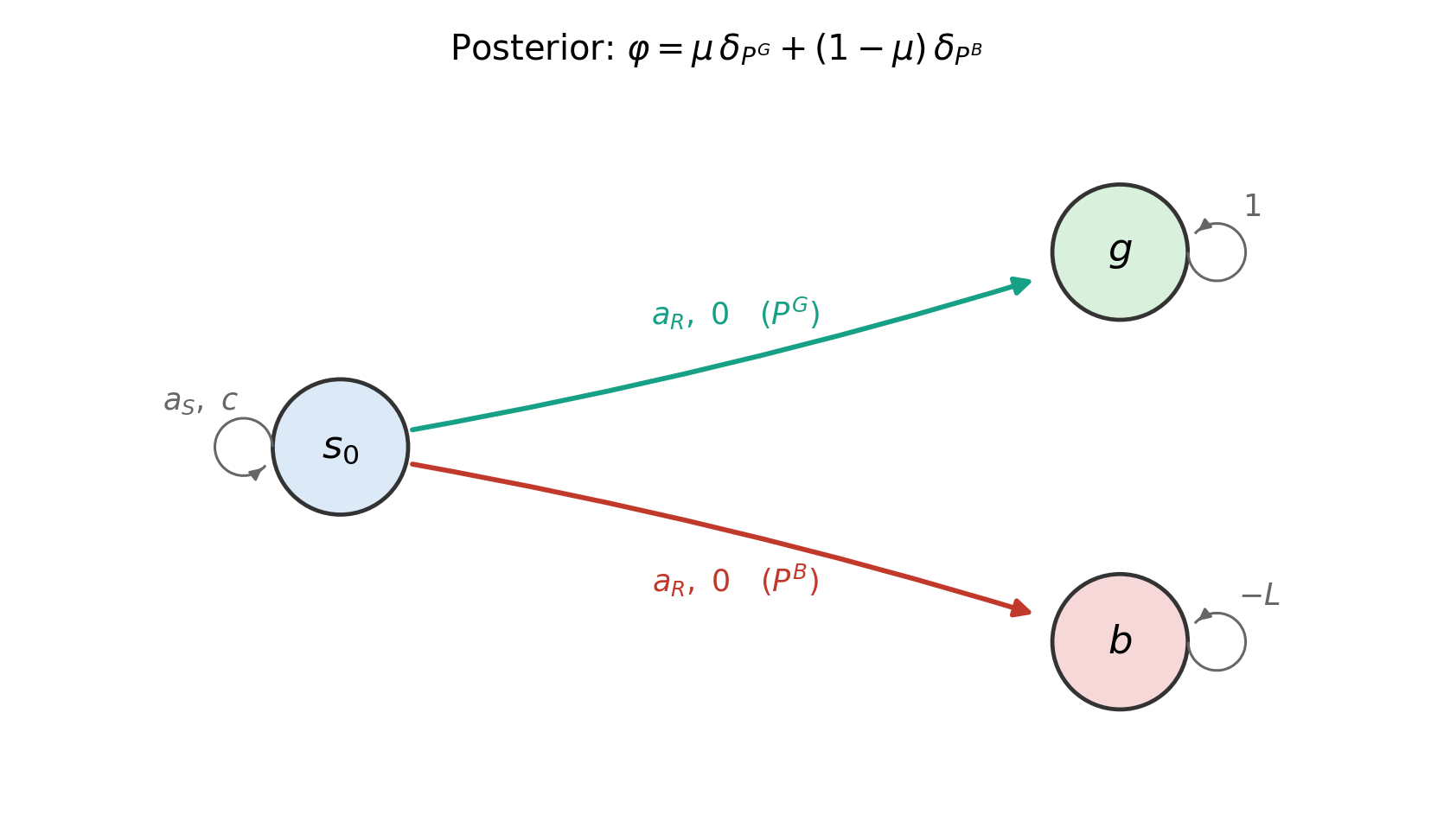}
\caption{Schematic of Example~\ref{ex:robust-necessity}. From the initial state $s_0$, the safe action $a_S$ yields reward $c$ and returns to $s_0$. The risky action $a_R$ yields reward $0$. Under $P^G$, it moves to the absorbing state $g$, which yields reward $1$ at every subsequent step; under $P^B$, it moves to the absorbing state $b$, which yields reward $-L$ at every subsequent step.}
\label{fig:robust-necessity-schematic}
\end{figure}


\begin{proposition}
\label{prop:robust-necessity}
In Example~\ref{ex:robust-necessity}, if $\frac{\gamma(1-L)}{2}<c<\gamma(\mu-(1-\mu)L)$, $\mu>1/2$, and $\alpha\le 1-\mu$, then $\pi^*_{\bar P_\varphi}(s_0)=a_R$ and $\pi^*_{\varphi,\alpha}(s_0)=a_S$. Moreover, for every $\Lambda$ satisfying
$\frac{\gamma-c}{1-\gamma}<\Lambda\le \frac{c+\gamma L}{1-\gamma},$
the optimal policy under the posterior-mean kernel $\pi^*_{\bar P_\varphi}$ has $(1-\mu,\Lambda)$-posterior downside exposure, whereas the regret of $\pi^*_{\varphi,\alpha}$ is below the threshold $\Lambda$ with posterior probability one, i.e., 
$$
\varphi\!\left(\mathcal D_\varphi(\pi^*_{\bar P_\varphi},\Lambda)\right)=1-\mu,
\qquad
\varphi\!\left(\mathcal D_\varphi(\pi^*_{\varphi,\alpha},\Lambda)\right)=0.
$$
\end{proposition}

The proof of Proposition~\ref{prop:robust-necessity} is deferred to 
Appendix~\ref{app:robust-necessity}. The proposition highlights a failure mode 
of posterior-mean planning under model uncertainty. Since the posterior-mean 
kernel averages over the possible transition models, it assign a higher posterior-mean value to $a_R$ than to $a_S$; 
However, this averaged comparison hides the fact that, with posterior probability $1-\mu$, the selected action $a_R$ leads to the bad absorbing state $b$ and receives reward $-L$ thereafter. Under those posterior models, the posterior-mean policy suffers regret above the threshold $\Lambda$ specified in Proposition~\ref{prop:robust-necessity}. In contrast, when $\alpha\le 1-\mu$, the $\alpha$-quantile BR-MDP is sensitive to this lower-tail posterior outcome: it penalizes $a_R$ for its performance on the models where $a_R$ leads to $b$, and therefore selects the safe action $a_S$.


We next consider a regime in which the balance shifts toward exploration. An action can matter because after observing the transition it generates, the posterior distribution can become substantially more concentrated, or even collapse to a point mass at the true transition kernel. This becomes important once the main issue is no longer the large posterior downside exposure defined in Definition~\ref{def:posterior-downside} but whether taking the action substantially reduces the remaining epistemic uncertainty that matters for subsequent decisions. The next example illustrates this situation by contrasting an explorative action, whose transition reveals the kernel, with a safe action that leaves the posterior unchanged. Let $\varphi^{(s,a,s')}$ denote the updated posterior distribution after observing $(s,a,s')$.

\begin{example}[An informative exploration example]
\label{ex:probing}
Fix a discount factor $\gamma\in(0,1)$ and a constant $c\in(0,\gamma)$. Consider the discounted MDP in Figure~\ref{fig:probing-schematic}, where the reward is known to the agent but the transition kernel is unknown. The current posterior is supported on two kernels,
$$
\varphi=\mu\,\delta_{P^G}+(1-\mu)\,\delta_{P^B}, \qquad \mu\in(0,1),
$$
where $P^G$ and $P^B$ differ only along the branch reached after taking the explorative action shown in the figure.
\end{example}

\begin{figure}[htbp]
\centering
\includegraphics[width=0.6\textwidth]{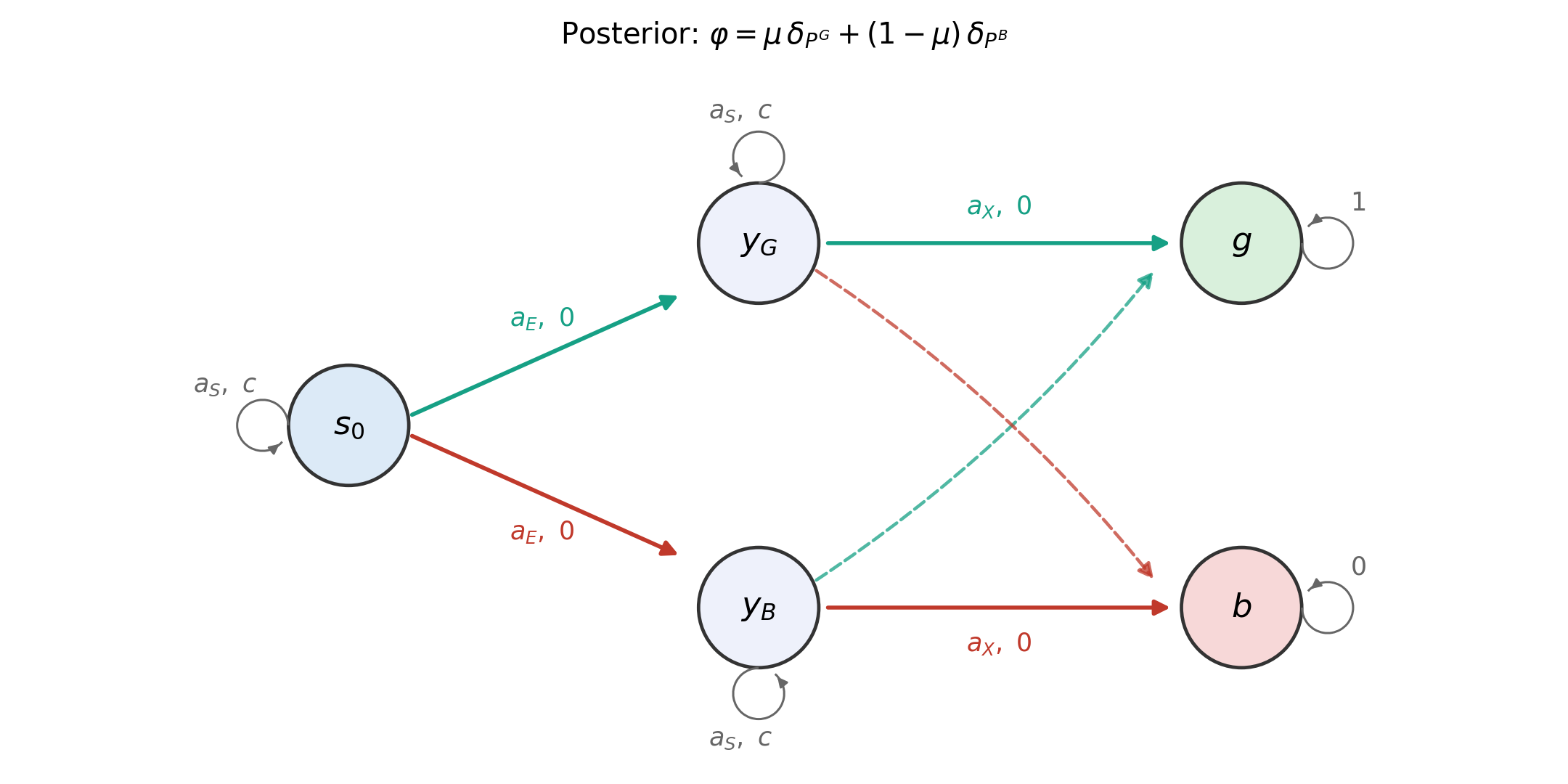}
\caption{Schematic of Example~\ref{ex:probing}. From the initial state $s_0$, the safe action $a_S$ yields reward $c$ and returns to $s_0$, whereas the exploratory action $a_E$ yields reward $0$ and moves to $y_G$ under $P^G$ and to $y_B$ under $P^B$. At each diagnostic state $y\in\{y_G,y_B\}$, action $a_S$ yields reward $c$ and stays at $y$, whereas action $a_X$ yields reward $0$ and moves to the absorbing state $g$ under $P^G$ and to the absorbing state $b$ under $P^B$. State $g$ yields reward $1$ at every subsequent step, and state $b$ yields reward $0$ at every subsequent step. All other transitions are identical under $P^G$ and $P^B$. Hence, the transition observed after $(s_0,a_E)$ reveals the kernel, whereas $(s_0,a_S)$ reveals no new information.}
\label{fig:probing-schematic}
\end{figure}

This example isolates a situation in which an action is valuable because of the information it reveals: after taking $a_E$ at $s_0$, the next-state observation reveals the kernel, so $\varphi^{(s_0,a_E,y_G)}=\delta_{P^G},$ and $\varphi^{(s_0,a_E,y_B)}=\delta_{P^B}.$ The optimal subsequent action can then be selected according to the updated posterior: take $a_X$ at $y_G$ and $a_S$ at $y_B$. Therefore, taking $a_E$ at $s_0$ produces an observation that fully resolves the remaining uncertainty relevant to subsequent decisions.
By contrast, taking the safe action $a_S$ at $s_0$ does not change the posterior: $P^G(s_0|s_0,a_S)=P^B(s_0|s_0,a_S)=1$, so observing $(s_0,a_S,s_0)$ cannot provide additional information for distinguishing $P^G$ from $P^B$, and hence $\varphi^{(s_0,a_S,s_0)}=\varphi.$
 
The next proposition shows that a fixed lower-tail rule can nevertheless avoid this informative action indefinitely. Once the agent chooses $a_S$, the posterior never changes, so the same decision rule continues to choose $a_S$ at every subsequent visit to $s_0$. If the true kernel is $P^G$, this yields linear regret.

\begin{proposition}\label{prop:robust-hurts-exploration}
In Example~\ref{ex:probing}, assume that $\alpha\le 1-\mu$ and $c<\gamma^2$. Suppose that after each realized transition the agent updates its posterior and then acts according to an optimal policy of the $\alpha$-quantile BR-MDP, $\pi_t:=\pi^*_{\varphi_t,\alpha}$ at time step $t$. If the initial state is $s_0$ and the prior distribution $\varphi_0=\varphi$, then
$$
\varphi_t=\varphi,
\qquad
a_t=\pi_t(s_t)=a_S
\qquad\text{and}\qquad
s_{t+1}=s_0
\qquad
\forall t\ge 0.
$$
If the true kernel is $P^G$, then under the cumulative regret criterion $R_T
:=
\sum_{t=0}^{T-1}
\Bigl(
V^*_{P^G}(s_t)-V^{\pi_t}_{P^G}(s_t)
\Bigr)$,
one has
$
R_T
=
T\,\frac{\gamma^2-c}{1-\gamma}.
$
\end{proposition}

The proof of Proposition~\ref{prop:robust-hurts-exploration} is deferred to Appendix~\ref{app:robust-hurts-exploration}.
It shows how a fixed lower-tail rule can create what we call a self-confirming trap: under the current posterior, the optimal policy of the $\alpha$-quantile BR-MDP is suboptimal in the true environment, yet following this policy generates no new information to update the posterior. As a result, re-solving the same $\alpha$-quantile BR-MDP under the unchanged posterior and fixed quantile level $\alpha$ selects the same action again.
When the true kernel is $P^G$, the fixed lower-tail rule continues to choose $a_S$ and the agent therefore never discovers the higher future value available after probing, which leads to linear regret.

Taken together, these results show that the robustness induced by the lower-tail quantile can be desirable when posterior downside exposure is substantial due to high epistemic uncertainty, but that using a fixed lower-tail rule can later hinder informative exploration and create a self-confirming trap with linear regret. The reason is that the trade-off between robustness and exploration evolves over learning.

\subsection{Proof of Proposition \ref{prop:robust-necessity}}
\label{app:robust-necessity}
\begin{proof}
Let $R_{\bar P}
:=
\frac{\gamma(\mu-(1-\mu)L)}{1-\gamma}.$
Under the posterior-mean kernel $\bar P_\varphi$, if action $a_R$ is taken at $s_0$, the value obtained after the transition is $R_{\bar P}
=
\mu\frac{\gamma}{1-\gamma}
+
(1-\mu)\frac{-\gamma L}{1-\gamma}.$
The optimal value at $s_0$ under $\bar P_\varphi$ satisfies
\begin{align*}
V^*_{\bar P_\varphi}(s_0) = \max\left\{ c+\gamma V^*_{\bar P_\varphi}(s_0), \, R_{\bar P} \right\}.
\end{align*}
Since $\gamma(\mu-(1-\mu)L)>c$, we have $R_{\bar P}>\frac{c}{1-\gamma}.$
Equivalently, $c+\gamma R_{\bar P}<R_{\bar P}.$
Hence the fixed point is $V^*_{\bar P_\varphi}(s_0)=R_{\bar P},$
and the unique optimal action at $s_0$ is $a_R$. Therefore $\pi^*_{\bar P_\varphi}(s_0)=a_R.$

Next consider the lower-tail $\alpha$-quantile BR-MDP. Since $g$ and $b$ are absorbing, $V(g)=\frac{1}{1-\gamma},$ and $V(b)=\frac{-L}{1-\gamma}.$
Under $a_R$, the posterior distribution of the next-state value is $\frac{1}{1-\gamma}$ with probability $\mu$, $\frac{-L}{1-\gamma}$ with probability $1-\mu$.
Since $\alpha\le 1-\mu$, the left $\alpha$-quantile is $\frac{-L}{1-\gamma}.$
Thus the BRMDP value of taking $a_R$ at $s_0$ is $R_\alpha
:=
-\frac{\gamma L}{1-\gamma}.$
The BRMDP optimal value at $s_0$ therefore satisfies
\begin{align*}
V^*_{\varphi,\alpha}(s_0) = \max\left\{ c+\gamma V^*_{\varphi,\alpha}(s_0), \, R_\alpha \right\}.
\end{align*}
Because $c>0$ and $L>0$, $\frac{c}{1-\gamma}>-\frac{\gamma L}{1-\gamma}=R_\alpha.$
Hence the fixed point is $V^*_{\varphi,\alpha}(s_0)=\frac{c}{1-\gamma},$ and the optimal action is $a_S$. Therefore $\pi^*_{\varphi,\alpha}(s_0)=a_S.$

We now compute the regret under each realized kernel. Since $c<\gamma(\mu-(1-\mu)L)\le\gamma$, under $P^G$, taking $a_R$ is optimal. Hence $V^*_{P^G}(s_0)=\frac{\gamma}{1-\gamma}.$
The posterior-mean policy chooses $a_R$, so its regret under $P^G$ is zero. The
BRMDP policy chooses $a_S$, so
\begin{align*}
V^*_{P^G}(s_0)-V^{\pi^*_{\varphi,\alpha}}_{P^G}(s_0) = \frac{\gamma}{1-\gamma} - \frac{c}{1-\gamma} = \frac{\gamma-c}{1-\gamma}.
\end{align*}
Under $P^B$, taking $a_S$ is optimal because $c>0$ and $L>0$. Hence $V^*_{P^B}(s_0)=\frac{c}{1-\gamma}.$
The BRMDP policy chooses $a_S$, so its regret under $P^B$ is zero. The posterior-mean
policy chooses $a_R$, so
\begin{align*}
V^*_{P^B}(s_0)-V^{\pi^*_{\bar P_\varphi}}_{P^B}(s_0) = \frac{c}{1-\gamma} - \left(-\frac{\gamma L}{1-\gamma}\right) = \frac{c+\gamma L}{1-\gamma}.
\end{align*}
Finally, the condition $c>\frac{\gamma(1-L)}{2}$ is equivalent to $\frac{\gamma-c}{1-\gamma}<\frac{c+\gamma L}{1-\gamma},$ so the stated interval for $\Lambda$ is nonempty. For any $\frac{\gamma-c}{1-\gamma}
<
\Lambda
\le
\frac{c+\gamma L}{1-\gamma},$
the posterior-mean policy has regret at least $\Lambda$ exactly under $P^B$, whose
posterior probability is $1-\mu$. Therefore $\varphi\!\left(\mathcal D(\pi^*_{\bar P_\varphi},\Lambda)\right) = 1-\mu.$

On the other hand, the BRMDP policy has regret $\frac{\gamma-c}{1-\gamma}<\Lambda$ under $P^G$, and regret $0$ under $P^B$. Hence its regret is below $\Lambda$ with
posterior probability one, and $\varphi\!\left(\mathcal D(\pi^*_{\varphi,\alpha},\Lambda)\right) = 0.$
\end{proof}



\subsection{Proof of Proposition \ref{prop:robust-hurts-exploration}}
\label{app:robust-hurts-exploration}

\begin{proof}
Under the $\alpha$-quantile BR-MDP criterion with $\alpha\le 1-\mu$, taking $a_X$ at either diagnostic state yields next-state value $1/(1-\gamma)$ with posterior mass $\mu$ and $0$ with posterior mass $1-\mu$. Therefore its left $\alpha$-quantile is $0$, whereas repeatedly taking the safe self-loop yields value $c/(1-\gamma)$. Hence
\begin{align*}
V^*_{\varphi,\alpha}(y_G)=V^*_{\varphi,\alpha}(y_B)=\frac{c}{1-\gamma},
\end{align*}
so taking $a_E$ at $s_0$ yields value $\gamma c/(1-\gamma)$. By contrast, repeatedly taking the safe self-loop at $s_0$ yields value $c/(1-\gamma)$, and delaying $a_E$ by $k\ge 1$ safe steps yields
\begin{align*}
\frac{c(1-\gamma^k)+\gamma^{k+1}c}{1-\gamma}<\frac{c}{1-\gamma}.
\end{align*}
Thus $\pi^*_{\varphi,\alpha}(s_0)=a_S$.

Now suppose that after each realized transition the agent updates its posterior and then chooses $\pi_t:=\pi^*_{\varphi_t,\alpha}$. Taking $a_S$ at $s_0$ always produces the observation $(s_0,a_S,s_0)$, whose likelihood is identical under $P^G$ and $P^B$. Hence Bayes updating leaves the posterior unchanged. Starting from $S_0=s_0$ and $\varphi_0=\varphi$, an induction gives $\varphi_t=\varphi$ and $\pi_t(s_0)=a_S$ for all $t\ge 0$; since $a_S$ returns to $s_0$, we also have $S_t=s_0$ for all $t$. In particular, the agent never probes.

Under $P^G$, taking $a_E$ at $s_0$ and then $a_X$ at $y_G$ yields value $\gamma^2/(1-\gamma)$, whereas repeatedly taking the safe self-loop yields value $c/(1-\gamma)$. Since $c<\gamma^2$, probing strictly dominates staying safe, and delaying $a_E$ by $k\ge 1$ safe steps yields $\frac{c(1-\gamma^k)+\gamma^{k+2}}{1-\gamma}<\frac{\gamma^2}{1-\gamma}.$
Hence $V^*_{P^G}(s_0)=\gamma^2/(1-\gamma)$. On the other hand, each $\pi_t$ chooses $a_S$ at $s_0$, so $V^{\pi_t}_{P^G}(s_0)=c/(1-\gamma)$. Because $S_t=s_0$ for all $t$, each summand in $R_T$ equals $(\gamma^2-c)/(1-\gamma)$, and therefore $R_T=T\,\frac{\gamma^2-c}{1-\gamma}.$
\end{proof}

\section{Value Iteration for $\alpha_k$-quantile BR-MDP}
\label{sec:value-iteration-exact}
At the beginning of pseudo-episode $k$, after computing the posterior parameter collection
$\phi_k$ and the adaptive quantile schedule $\alpha_k$, we solve the corresponding BR-MDP
by value iteration. 
\begin{algorithm}[H]
\caption{Value Iteration for the $\alpha_k$-quantile BR-MDP in Pseudo-Episode $k$}
\label{alg:brmdp-vi-k}
\small
\begin{algorithmic}[1]
\State \textbf{Input:} Posterior parameter collection $\phi_k$, quantile schedule $\alpha_k$, optimal value function in pseudo-episode $k-1$ $V_{k-1}$, tolerance $\varepsilon_{\mathrm{VI}}$, maximum iterations $M_{\mathrm{VI}}$
\State Initialize $V^{(0)}\leftarrow V_{k-1}$ 
\For{$m=0,1,\ldots,M_{\mathrm{VI}}-1$}
    \ForAll{$s\in\mathcal S$}
        \State
        \begingroup
        \setlength{\abovedisplayskip}{1pt plus 1pt minus 1pt}
        \setlength{\belowdisplayskip}{1pt plus 1pt minus 1pt}
        \setlength{\abovedisplayshortskip}{0pt plus 1pt}
        \setlength{\belowdisplayshortskip}{1pt plus 1pt minus 1pt}
        \begin{align*}
V^{(m+1)}(s) \gets \max_{a\in\mathcal A} \left\{ r(s,a) + \gamma\, \rho^{\alpha_k}_{\phi_k(s,a)} \!\left(P^\top V^{(m)}\right) \right\}. 
        \end{align*}
        \endgroup
    \EndFor
    \If{$\|V^{(m+1)}-V^{(m)}\|_\infty\le \varepsilon_{\mathrm{VI}}$}
        \State \textbf{break}
    \EndIf
\EndFor
\State Set $V_k\gets V^{(m+1)}$
\State For each $s\in\mathcal S$, choose $\pi_k(s)\in\argmax_{a\in\mathcal A}\left\{r(s,a)+\gamma\,\rho^{\alpha_k}_{\phi_k(s,a)}\!\left(P^\top V_k\right)\right\}$
\State \Return $V_k$, $\pi_k$
\end{algorithmic}
\end{algorithm}

\section{Proofs of Regret Analysis}
For a fixed interaction horizon $T$, the last pseudo-episode may be truncated. 
{In the proof, with a slight abuse of notation, let $L_k$ denote the full geometric length of pseudo-episode $k$.}
{Equivalently, if the last pseudo-episode is truncated by the fixed horizon, we continue it only for the purpose of the proof under the same policy $\pi_{K_T}$.}
{Since the added regret summands are nonnegative, this no-truncation convention provides an upper bound on the actual regret.}
Thus,
\begin{align}
\label{ieq:BRT} BR(T) \le \mathbb E\!\left[ \sum_{k=1}^{K_T}\sum_{i=1}^{L_k} \Bigl( V^*(s_{k,i})-V^{\pi_k}(s_{k,i}) \Bigr) \right].
\end{align}
{It remains to bound the right-hand side under this convention.}
\subsection{Proof of Lemma \ref{lem:bayes_optimism}}
\label{app:bayes_optimism}
\begin{proof}
From the Bayesian perspective,
\begin{align*}
P^c_{s,a}\mid \mathcal F_{t_k}\sim \Dir(\phi_k(s,a)), \qquad \forall (s,a)\in\mathcal S\times\mathcal A,
\end{align*}
{where $\phi_k(s,a)=(\phi_k(s,a,s'))_{s'\in\mathcal S}$ denotes the Dirichlet parameter vector.}
For each $(s,a)\in\mathcal S\times\mathcal A$, define {the event} $\mathcal G_k(s,a)
:=
\left\{
(P^c_{s,a})^\top V_k
\le
\rho^{\alpha_k}_{\phi_k(s,a)}(P^\top V_k)
\right\}.$
By the definition of the $\alpha$-quantile,
\begin{align*}
\mathbb P\!\left(\mathcal G_k(s,a)\mid \mathcal F_{t_k}\right) = \mathbb P\!\left( (P^c_{s,a})^\top V_k \le \rho^{\alpha_k}_{\phi_k(s,a)}(P^\top V_k) \ \middle|\ \mathcal F_{t_k} \right) \ge \alpha_k(s,a).
\end{align*}
Let $\mathcal G_k:=\bigcap_{(s,a)\in\mathcal S\times\mathcal A}\mathcal G_k(s,a).$
Then
\begin{align*}
\mathbb P\!\left(\mathcal G_k\mid \mathcal F_{t_k}\right) &\ge 1-\sum_{(s,a)\in\mathcal S\times\mathcal A} \mathbb P\!\left(\mathcal G_k^c(s,a)\mid \mathcal F_{t_k}\right)
\ge 1-\sum_{(s,a)\in\mathcal S\times\mathcal A}\bigl(1-\alpha_k(s,a)\bigr)\\
&\ge 1-\delta\frac{\ln(2k)}{\sqrt{k}} \sum_{(s,a)\in\mathcal S\times\mathcal A}\frac{N^+_k(s,a)}{\bar N^+_k}
= 1-\frac{\delta SA\ln(2k)}{\sqrt{k}},
\end{align*}
where the third line follows from \eqref{eq:alphak} and the last equality uses $\bar N^+_k=\frac{1}{SA}\sum_{(s,a)\in\mathcal S\times\mathcal A}N^+_k(s,a).$

It remains to show that on the event $\mathcal G_k$ we have $V^*\le V_k$ pointwise.
Define
\begin{align*}
Q^*(s,a):=r(s,a)+\gamma\,(P^c_{s,a})^\top V^*, \qquad Q_k(s,a):=r(s,a)+\gamma\,\rho^{\alpha_k}_{\phi_k(s,a)}(P^\top V_k),
\end{align*}
so that $V^*(s)=\max_{a\in\mathcal A}Q^*(s,a)$ and $V_k(s)=\max_{a\in\mathcal A}Q_k(s,a).$
Let $\Delta(s):=V^*(s)-V_k(s)$ and $\Delta_{\max}:=\max_{x\in\mathcal S}\Delta(x).$
For each $s\in\mathcal S$, choose $a_s^*\in\arg\max_{a\in\mathcal A}Q^*(s,a)$. Then
\begin{align*}
\Delta(s) &= V^*(s)-V_k(s) \le Q^*(s,a_s^*)-Q_k(s,a_s^*)\\
&= \gamma\Big( (P^c_{s,a_s^*})^\top V^* - \rho^{\alpha_k}_{\phi_k(s,a_s^*)}(P^\top V_k) \Big)\\
&= \gamma\Big( (P^c_{s,a_s^*})^\top (V^*-V_k) + (P^c_{s,a_s^*})^\top V_k - \rho^{\alpha_k}_{\phi_k(s,a_s^*)}(P^\top V_k) \Big)\\
&\le \gamma \Delta_{\max},
\end{align*}
where the last inequality uses $(P^c_{s,a_s^*})^\top (V^*-V_k)\le \Delta_{\max}$ and, on the event $\mathcal G_k$, $(P^c_{s,a_s^*})^\top V_k
-
\rho^{\alpha_k}_{\phi_k(s,a_s^*)}(P^\top V_k)
\le 0.$
Hence $\Delta_{\max}\le \gamma \Delta_{\max},$ which implies $\Delta(s)\le\Delta_{\max}\le 0$.
That is, $V^*(s)\le V_k(s)$ for all $s\in\mathcal S$ on the event $\mathcal G_k$.
Combining this implication with $\mathbb P\!\left(\mathcal G_k\mid {\mathcal F_{t_k}}\right)
\ge
1-\frac{\delta SA\ln(2k)}{\sqrt{k}}$ {proves part~(i).}

For part~(ii), let $V^-:=V^*_{\phi_k,\underline{\alpha}}.$
Since $\alpha_k(s,a)\ge \underline{\alpha}$ for all $(s,a)\in\mathcal S\times\mathcal A$ and the $\alpha$-quantile is nondecreasing in $\alpha$, we have
\begin{align*}
\big({\mathcal T^*_{\phi_k,\underline{\alpha}}}V\big)(s) \le \big({\mathcal T^*_{\phi_k,\alpha_k}}V\big)(s), \qquad \forall V,\ \forall s\in\mathcal S.
\end{align*}
Therefore,
\begin{align*}
V^- = \mathcal T_{\phi_k,\underline{\alpha}}^*V^- \le {\mathcal T^*_{\phi_k,\alpha_k}}V^- \le ({\mathcal T^*_{\phi_k,\alpha_k}})^nV^-, \qquad \forall n\ge 1,
\end{align*}
where the last inequality follows from the monotonicity of ${\mathcal T^*_{\phi_k,\alpha_k}}$.
Since ${\mathcal T^*_{\phi_k,\alpha_k}}$ is a $\gamma$-contraction, its iterates converge to its unique fixed point $V_k$. Letting $n\to\infty$ yields $V^-\le V_k$, that is, $V^*_{\phi_k,\underline{\alpha}}(s)\le V_k(s)$ for all $s\in\mathcal S$.
This completes the proof.
\end{proof}

\subsection{Proof of Lemma \ref{lem:value-decomposition-sum}}
\label{app:value-decomposition-sum}
\begin{proof}
Define $\Delta(s):=V_{\phi_k,\alpha_k}^{\pi} (s)-V^{\pi}(s)$ and $\Delta_{k,i}:=\Delta(s_{k,i}).$
By the Bellman equations for $V_{\phi_k,\alpha_k}^{\pi}$ and $V^\pi$, for each $i\ge 1$,
\begin{align*}
\Delta_{k,i}
&=
\Bigl[
r(s_{k,i},a_{k,i})
+
\gamma\,
\rho_{\phi_k(s_{k,i},a_{k,i})}^{\alpha_k}
\!\bigl(P^\top V_{\phi_k,\alpha_k}^{\pi} \bigr)
\Bigr]
-
\Bigl[
r(s_{k,i},a_{k,i})
+
\gamma\,(P^c_{s_{k,i},a_{k,i}})^\top V^\pi
\Bigr] \\
&=
\gamma\,E_{k,i}
+
\gamma\,(P^c_{s_{k,i},a_{k,i}})^\top
\bigl(V_{\phi_k,\alpha_k}^{\pi} -V^\pi\bigr).
\end{align*}
Since
\begin{align*}
\mathbb E\!\left[\Delta_{k,i+1}\,\middle|\,s_{k,i},a_{k,i},\mathcal F_{t_k},P^c\right] = (P^c_{s_{k,i},a_{k,i}})^\top \Delta,
\end{align*}
the tower property gives
\begin{equation}
\mathbb E\!\left[\Delta_{k,i}\,\middle|\,\mathcal F_{t_k},P^c\right]
=
\gamma\,
\mathbb E\!\left[E_{k,i}\,\middle|\,\mathcal F_{t_k},P^c\right]
+
\gamma\,
\mathbb E\!\left[\Delta_{k,i+1}\,\middle|\,\mathcal F_{t_k},P^c\right].
\label{eq:delta-recursion}
\end{equation}

Because rewards are bounded in $[0,1]$, both $V_{\phi_k,\alpha_k}^{\pi}$ and $V^\pi$ are bounded by
$(1-\gamma)^{-1}$ in sup norm. Hence $|\Delta_{k,i}| \le \frac{1}{1-\gamma}$ and $|E_{k,i}| \le \frac{1}{1-\gamma}$ for all $i\ge 1.$
Iterating \eqref{eq:delta-recursion} for $n\ge 1$ yields
\begin{align*}
\mathbb E\!\left[\Delta_{k,i}\,\middle|\,\mathcal F_{t_k},P^c\right]
=
\sum_{h=0}^{n-1}\gamma^{h+1}
\mathbb E\!\left[E_{k,i+h}\,\middle|\,\mathcal F_{t_k},P^c\right]
+
\gamma^n
\mathbb E\!\left[\Delta_{k,i+n}\,\middle|\,\mathcal F_{t_k},P^c\right].
\end{align*}
The last term converges to zero as $n\to\infty$ because $\gamma\in(0,1)$ and $\Delta_{k,i+n}$ is uniformly bounded.
Therefore,
\begin{equation}
\mathbb E\!\left[\Delta_{k,i}\,\middle|\,\mathcal F_{t_k},P^c\right] = \sum_{h=0}^{\infty}\gamma^{h+1} \mathbb E\!\left[E_{k,i+h}\,\middle|\,\mathcal F_{t_k},P^c\right]. \label{eq:delta-expansion}
\end{equation}

Next use the pseudo-episode construction. Conditional on $(\mathcal F_{t_k},P^c)$, the random length $L_k$
is independent of the MDP trajectory and satisfies $\mathbb P\!\left(L_k\ge i \,\middle|\, \mathcal F_{t_k},P^c\right)=\gamma^{i-1},$ $i\ge 1,$
because $L_k$ is geometric with success probability $1-\gamma$ on $\{1,2,\dots\}$.
{Since $|\Delta_{k,i}|\le(1-\gamma)^{-1}$, we have $\mathbb E[\sum_{i\ge1}\mathbf 1\{L_k\ge i\}|\Delta_{k,i}|\,|\,\mathcal F_{t_k},P^c]\le(1-\gamma)^{-2}<\infty$. Hence Fubini's theorem and conditional independence give}
\begin{align}
\mathbb E\!\left[ \sum_{i=1}^{L_k}\Delta_{k,i} \,\middle|\, \mathcal F_{t_k},P^c \right] &= \sum_{i=1}^{\infty} \mathbb E\!\left[ \mathbf 1\{L_k\ge i\}\Delta_{k,i} \,\middle|\, \mathcal F_{t_k},P^c \right] 
= \sum_{i=1}^{\infty}\gamma^{i-1} \mathbb E\!\left[ \Delta_{k,i} \,\middle|\, \mathcal F_{t_k},P^c \right]. \label{eq:lhs-expand}
\end{align}
Substituting \eqref{eq:delta-expansion} into \eqref{eq:lhs-expand} and exchanging the order of summation,
which is justified by absolute summability, gives
\begin{align}
\mathbb E\!\left[
\sum_{i=1}^{L_k}\Delta_{k,i}
\,\middle|\,
\mathcal F_{t_k},P^c
\right]
&=
\sum_{i=1}^{\infty}\gamma^{i-1}
\sum_{h=0}^{\infty}\gamma^{h+1}
\mathbb E\!\left[
E_{k,i+h}
\,\middle|\,
\mathcal F_{t_k},P^c
\right] =
\sum_{t=1}^{\infty}
t\,\gamma^{t}
\mathbb E\!\left[
E_{k,t}
\,\middle|\,
\mathcal F_{t_k},P^c
\right].
\label{eq:middle-series}
\end{align}
{Similarly, $\mathbb E[\sum_{t\ge1}\mathbf 1\{L_k\ge t\}\gamma t |E_{k,t}|\,|\,\mathcal F_{t_k},P^c]<\infty$, so Fubini's theorem and conditional independence give}
\begin{align}
\mathbb E\!\left[
\sum_{t=1}^{L_k}t\,\gamma\,E_{k,t}
\,\middle|\,
\mathcal F_{t_k},P^c
\right]
&=
\sum_{t=1}^{\infty}
t\,\gamma\,
\mathbb E\!\left[
\mathbf 1\{L_k\ge t\}E_{k,t}
\,\middle|\,
\mathcal F_{t_k},P^c
\right] =
\sum_{t=1}^{\infty}
t\,\gamma^{t}
\mathbb E\!\left[
E_{k,t}
\,\middle|\,
\mathcal F_{t_k},P^c
\right].
\label{eq:rhs-expand}
\end{align}
Comparing \eqref{eq:middle-series} and \eqref{eq:rhs-expand} proves \eqref{eq:agg-vd}.
\end{proof}

\subsection{Proof of Lemma \ref{lemma:value-decomposition-sum-robust}}
\label{app:value-decomposition-sum-robust}
\begin{proof}
Apply Lemma~\ref{lem:value-decomposition-sum} with the risk profile $\alpha_k$. This gives
\begin{equation}
\mathbb E\!\left[
\sum_{i=1}^{L_k}
\Bigl(
V_{\phi_k,\alpha_k}^{\pi} (s_{k,i})-V^{\pi}(s_{k,i})
\Bigr)
\,\middle|\,
\mathcal F_{t_k},\,P^c
\right]
=
\mathbb E\!\left[
\sum_{i=1}^{L_k}i\,\gamma\, E_{k,i}
\,\middle|\,
\mathcal F_{t_k},\,P^c
\right].
\label{eq:agg-vd-plus}
\end{equation}
Next, repeat the same argument as in Lemma~\ref{lem:value-decomposition-sum} for the constant risk level $\underline{\alpha}$. The corresponding one-step discrepancy is
\begin{align*}
\rho_{\phi_k(s_{k,i},a_{k,i})}^{\underline{\alpha}} \!\bigl(P^\top {V_{\phi_k,\underline{\alpha}}^{\pi}} \bigr) - (P^c_{s_{k,i},a_{k,i}})^\top {V_{\phi_k,\underline{\alpha}}^{\pi}} = -E^-_{k,i}.
\end{align*}
Therefore,
\begin{equation}
\mathbb E\!\left[
\sum_{i=1}^{L_k}
\Bigl(
  {V_{\phi_k,\underline{\alpha}}^{\pi}} (s_{k,i})-V^{\pi}(s_{k,i})
\Bigr)
\,\middle|\,
\mathcal F_{t_k},\,P^c
\right]
=
\mathbb E\!\left[
\sum_{i=1}^{L_k}i\,\gamma\, (-E^-_{k,i})
\,\middle|\,
\mathcal F_{t_k},\,P^c
\right].
\label{eq:agg-vd-minus}
\end{equation}
Subtracting \eqref{eq:agg-vd-minus} from \eqref{eq:agg-vd-plus} and using linearity of conditional expectation,
we obtain
\begin{equation*}
\mathbb E\!\left[
\sum_{i=1}^{L_k}
\Bigl(
V_{\phi_k,\alpha_k}^{\pi} (s_{k,i})-{V_{\phi_k,\underline{\alpha}}^{\pi}} (s_{k,i})
\Bigr)
\,\middle|\,
\mathcal F_{t_k},\,P^c
\right] =
\mathbb E\!\left[
\sum_{i=1}^{L_k}i\,\gamma\, 
\bigl(E_{k,i}+E^-_{k,i}\bigr)
\,\middle|\,
\mathcal F_{t_k},\,P^c
\right],
\end{equation*}
which is exactly \eqref{eq:agg-vd-floor}.
\end{proof}

\subsection{Proof of Theorem~\ref{thm:regret}}

\begin{lemma}[Dirichlet posterior quantile deviation for BR-MDP]
\label[lemma]{lem:dirichlet_quantile_BR-MDP}
Fix pseudo-episode $k$ and $(s,a)\in\mathcal S\times\mathcal A$.
Conditional on $\mathcal F_{t_k}$, let $P\sim \Dir(\phi_k(s,a)),$ 
{where $\phi_k(s,a)=(\phi_k(s,a,s'))_{s'\in\mathcal S}$ is the Dirichlet parameter vector.}
Let
$\phi_{k,0}(s,a):=\sum_{s'\in\mathcal S}\phi_k(s,a,s')
= N_k(s,a)+S$
{denote the scalar total concentration parameter.}
Define the posterior mean $\bar P_k(\cdot\mid s,a):=\phi_k(s,a)/\phi_{k,0}(s,a).$
Then for any $\mathcal F_{t_k}$-measurable vector
$V\in\left[0,\frac{1}{1-\gamma}\right]^{S}$ and any $\alpha\in(0,1)$,
\begin{align}
\rho_{\phi_k(s,a)}^\alpha(P^\top V)-\bar P_k(\cdot\mid s,a)^\top V &\le \frac{1}{1-\gamma} \sqrt{ \frac{2}{{\phi_{k,0}(s,a)}} \ln\!\left(\frac{1}{1-\alpha}\right) }, \label{eq:dirichlet_quantile_upper_BR-MDP}\\
\bar P_k(\cdot\mid s,a)^\top V-\rho_{\phi_k(s,a)}^\alpha(P^\top V) &\le \frac{1}{1-\gamma} \sqrt{ \frac{2}{{\phi_{k,0}(s,a)}} \ln\!\left(\frac{1}{\alpha}\right) }. \label{eq:dirichlet_quantile_lower_BR-MDP}
\end{align}
\end{lemma}

\begin{proof}
If $V$ is constant, then $P^\top V=\bar P_k(\cdot\mid s,a)^\top V$ almost surely, and both
\eqref{eq:dirichlet_quantile_upper_BR-MDP}--\eqref{eq:dirichlet_quantile_lower_BR-MDP}
are trivial. Hence we only consider the non-constant case.

If $S=1$, then $p\equiv \bar P_k(\cdot\mid s,a)\equiv 1$, so the result is again trivial.
Thus it remains to consider the case $S\ge 2$. In this case, $\phi_{k,0}(s,a)=N_k(s,a)+S\ge S\ge 2.$
Define $Y:=(1-\gamma)P^\top V$ and $\mu:=(1-\gamma)\bar P_k(\cdot\mid s,a)^\top V.$
By \citet[Lemma~B.1, Lemma~B.4 and the proof of Corollary~B.2]{agrawal2023optimistic}, we have the one-sided Gaussian tail bounds
\begin{gather*}
\mathbb P\!\left(Y-\mu\ge t \,\middle|\, \mathcal F_{t_k}\right) \le \exp\!\left(-\frac{\phi_{k,0}(s,a)t^2}{2}\right),\\
\mathbb P\!\left(\mu-Y\ge t \,\middle|\, \mathcal F_{t_k}\right) \le \exp\!\left(-\frac{\phi_{k,0}(s,a)t^2}{2}\right),
\end{gather*}
for all $t>0$.
Taking $t=(1-\gamma)\epsilon$ in these bounds yields, for every $\epsilon>0$,
\begin{align}
\mathbb P\!\left( P^\top V-\bar P_k(\cdot\mid s,a)^\top V \ge \epsilon \,\middle|\, \mathcal F_{t_k} \right) &\le\exp\! \left( -\frac{\phi_{k,0}(s,a)(1-\gamma)^2\epsilon^2}{2} \right), \label{eq:dirichlet_upper_tail_V_cited}\\
\mathbb P\!\left( \bar P_k(\cdot\mid s,a)^\top V-P^\top V \ge \epsilon \,\middle|\, \mathcal F_{t_k} \right) &\le \exp\!\left( -\frac{\phi_{k,0}(s,a)(1-\gamma)^2\epsilon^2}{2} \right). \label{eq:dirichlet_lower_tail_V_cited}
\end{align}

We now prove \eqref{eq:dirichlet_quantile_upper_BR-MDP}. Let $\epsilon_+
:=
\frac{1}{1-\gamma}
\sqrt{
\frac{2}{\phi_{k,0}(s,a)}
\ln\!\left(\frac{1}{1-\alpha}\right)
}.$
Then by \eqref{eq:dirichlet_upper_tail_V_cited},
\begin{align*}
\mathbb P\!\left( P^\top V \le \bar P_k(\cdot\mid s,a)^\top V+\epsilon_+ \,\middle|\, \mathcal F_{t_k} \right) \ge \alpha.
\end{align*}
By the definition of the left $\alpha$-quantile,
\begin{align*}
\rho_{\phi_k(s,a)}^\alpha(P^\top V) \le \bar P_k(\cdot\mid s,a)^\top V+\epsilon_+,
\end{align*}
which proves \eqref{eq:dirichlet_quantile_upper_BR-MDP}.

Next, let $\epsilon_-
:=
\frac{1}{1-\gamma}
\sqrt{
\frac{2}{\phi_{k,0}(s,a)}
\ln\!\left(\frac{1}{\alpha}\right)
}.$
Since $V$ is non-constant and the Dirichlet parameter vector has strictly positive components, $\phi_k(s,a,s')\ge 1$ for all $s'\in\mathcal S$, $P^\top V$ has a continuous distribution with a strictly increasing CDF on its support.
By \eqref{eq:dirichlet_lower_tail_V_cited},
\begin{align*}
\mathbb P\!\left( P^\top V \le \bar P_k(\cdot\mid s,a)^\top V-\epsilon_- \,\middle|\, \mathcal F_{t_k} \right) \le \alpha.
\end{align*}
Therefore, the preceding bound and the continuity and strict monotonicity of the CDF imply
\begin{align*}
\rho_{\phi_k(s,a)}^\alpha(P^\top V) \ge \bar P_k(\cdot\mid s,a)^\top V-\epsilon_-,
\end{align*}
which proves \eqref{eq:dirichlet_quantile_lower_BR-MDP}.
\end{proof}

\begin{lemma}\label{lem:mk-constant}
Let $\gamma\in(0,1)$ and let $L_1,\dots,L_K$ be i.i.d.\ geometric random variables on $\{1,2,\dots\}$ with $\mathbb{P}(L_k=\ell)=(1-\gamma)\gamma^{\ell-1}.$
Define $u:=\max\left\{e,\frac{K}{(1-\gamma)^2}\right\}$, $m:=\left\lceil \frac{1}{1-\gamma}\left(\log u+\log\log u+2\right)\right\rceil$, and
\begin{align*}
\mathfrak T:= \frac{4\gamma}{1-\gamma} \left(\log u+\log\log u+4\right)^2.
\end{align*}
Then
\begin{align*}
\mathbb E\!\left[ \frac{1}{1-\gamma}\sum_{k=1}^K\sum_{i=1}^{L_k} i\,\gamma \,\mathbb I\{L_k>m\} \right] \le \mathfrak T.
\end{align*}
\end{lemma}

\begin{proof}
Let $q:=1-\gamma$, and let $L\sim\mathrm{Geom}(q)$ on $\{1,2,\dots\}$. By linearity of expectation,
\begin{equation*}
\mathbb E\!\left[
\frac{1}{1-\gamma}\sum_{k=1}^K\sum_{i=1}^{L_k}
\gamma i\,\mathbb I\{L_k>m\}
\right]
=
\frac{K\gamma}{q}
\mathbb E\!\left[
\frac{L(L+1)}{2}\mathbb I\{L>m\}
\right]\le
\frac{K\gamma}{q}
\mathbb E\!\left[
L^2\mathbb I\{L>m\}
\right].
\end{equation*}
By the memoryless property of the geometric distribution, conditional on $\{L>m\}$ we can write $L=m+\widetilde L$, where $\widetilde L\sim\mathrm{Geom}(q)$. Hence
\begin{align*}
\mathbb E\!\left[L^2\mathbb I\{L>m\}\right] = \gamma^m\mathbb E\!\left[(m+\widetilde L)^2\right] \le \gamma^m\left(m^2+\frac{2m}{q}+\frac{2}{q^2}\right).
\end{align*}
Since $\gamma\le e^{-q}$, $qm\ge \log u+\log\log u+2$, and $K/q^2\le u$, we have $\gamma^m\le e^{-2}/(u\log u)$. Also $qm\le \log u+\log\log u+3$. Therefore,
\begin{equation*}
\frac{K\gamma}{q} \mathbb E\!\left[ L^2\mathbb I\{L>m\} \right] \le \frac{\gamma e^{-2}}{q\log u} \left[ (\log u+\log\log u+3)^2 +2(\log u+\log\log u+3) +2 \right]\le \mathfrak T.
\end{equation*}
This proves the claim.
\end{proof}

With a slight abuse of notation, for each time $t$, let $N_t(s,a)$ denote the number of visits to $(s,a)$ before time $t$.

\begin{theorem}\label{thm:regret}
Fix a deterministic interaction horizon $T$. Let
$u_T:=\max\left\{e,\frac{T}{(1-\gamma)^2}\right\},$ $
m_T:=\left\lceil \frac{1}{1-\gamma}\left(\log u_T+\log\log u_T+2\right)\right\rceil,$ $\mathfrak T_T:=
\frac{4\gamma}{1-\gamma}
\left(\log u_T+\log\log u_T+4\right)^2,$ and $\bar T:=T+m_T$.
{Define} $M_T:=\frac{2(S+1)(\bar T+SA)\sqrt{\bar T}}{SA\,\delta\,\ln 2}.$
Then:
\begin{enumerate}
    \item
    \begin{align}
BR(T) &\le \frac{\gamma m_T}{1-\gamma}\Bigg( \sqrt{16\,SA\,\bar T\,\ln\frac{1}{1-\underline{\alpha}}} + \sqrt{16\,SA\,(\bar T+S^2A)\left(\ln\left(1+\frac{SA\,M_T}{\bar T+S^2A}\right)+2\right)} \notag\\
&\qquad+ \sqrt{16\,SA\,\bar T\,\ln(2SA\,T^3m_T)} \Bigg) + \frac{\gamma SA}{1-\gamma}\,m_T^2\left\lceil \log_2 m_T\right\rceil \notag\\
&\qquad+ \frac{\pi^2}{6(1-\gamma)} + \mathfrak T_T + \frac{2\delta\,SA\,\sqrt{T}\ln(2T)}{(1-\gamma)^2}. \label{eq:main-order-proof-1}
    \end{align}

    \item
    \begin{align}
BR\text{-}R(T) &\le \frac{\gamma m_T}{1-\gamma}\Bigg( \sqrt{16\,SA\,\bar T\,\ln\frac{1}{1-\underline{\alpha}}} + \sqrt{16\,SA\,(\bar T+S^2A)\left(\ln\left(1+\frac{SA\,M_T}{\bar T+S^2A}\right)+2\right)} \notag\\
&\qquad+ \sqrt{16\,SA\,\bar T\,\ln\frac{1}{\underline{\alpha}}} + 2\sqrt{16\,SA\,\bar T\,\ln(2SA\,T^3m_T)} \Bigg) \notag\\
&\qquad+ \frac{2\gamma SA}{1-\gamma}\,m_T^2\left\lceil\log_2 m_T\right\rceil + \frac{\pi^2}{3(1-\gamma)} + 2\mathfrak T_T. \label{eq:main-order-proof-2}
    \end{align}
\end{enumerate}
\end{theorem}

\begin{proof}
For notational convenience, extend the Bernoulli restart process and the corresponding trajectory beyond time $T$ only for the purpose of the proof. Let $I_k^T:=\mathbb I\{k\le K_T\}$. Since every pseudo-episode has length at least one, $K_T\le T$, and sums over pseudo-episodes started by time $T$ can be written as sums over $k=1,\ldots,T$ multiplied by $I_k^T$. 

 {Using the upper bound in \eqref{ieq:BRT},}
\begin{align*}
{BR(T)} {\le} {\mathbb E\!\left[ \sum_{k=1}^{T}I_k^T\sum_{i=1}^{L_k} \Bigl(V^*(s_{k,i})-V^{\pi_k}(s_{k,i})\Bigr) \right].}
\end{align*}
 {Under the same no-truncation convention for the robust-optimal benchmark,}
\begin{align*}
{BR\text{-}R(T)} {\le} {\mathbb E\!\left[ \sum_{k=1}^{T}I_k^T\sum_{i=1}^{L_k} \Bigl(V^*_{\phi_k,\underline{\alpha}}(s_{k,i})-V^{\pi_k}_{\phi_k,\underline{\alpha}}(s_{k,i})\Bigr) \right].}
\end{align*}
For each started pseudo-episode $k$, recall  $\mathcal G_k=\bigcap_{s,a}\mathcal G_k(s,a)$
is the optimism event from Lemma~\ref{lem:bayes_optimism}. Since
$1-\alpha_k(s,a)
\le
\delta\frac{N_k^+(s,a)}{\bar N_k^+}\frac{\ln(2k)}{\sqrt{k}}$,
we have
\begin{align*}
\mathbb P(\mathcal G_k^c\mid {\mathcal F_{t_k}}) \le \sum_{s,a}(1-\alpha_k(s,a)) \le \delta\,\frac{\ln(2k)}{\sqrt{k}}\sum_{s,a}\frac{N_k^+(s,a)}{\bar N_k^+} = \delta\,SA\,\frac{\ln(2k)}{\sqrt{k}}.
\end{align*}
Moreover, on $\mathcal G_k$, $V^*(s)\le {V_k(s)}=V_{\phi_k,\alpha_k}^*(s),$ $ \forall s\in\mathcal S.$

For part~(i), define $\Delta_{k,i}:={V_k(s_{k,i})}-V^{\pi_k}(s_{k,i})$. Then, for $s_{k,i}$,
\begin{align*}
V^*(s_{k,i})-V^{\pi_k}(s_{k,i}) = \Delta_{k,i} +\Big(V^*(s_{k,i})-{V_k(s_{k,i})}\Big) \le \Delta_{k,i}+\frac{1}{1-\gamma}\,\mathbb I\{\mathcal G_k^c\}.
\end{align*}
Therefore,
\begin{align*}
{BR(T)} \le \mathbb E\!\left[ \sum_{k=1}^{T}I_k^T\sum_{i=1}^{L_k}\Delta_{k,i} \right] + \frac{1}{1-\gamma} \sum_{k=1}^{T} \mathbb E\!\left[I_k^T L_k\,\mathbb I\{\mathcal G_k^c\}\right].
\end{align*}
{For the  non-optimistic-event term, we first use $I_k^T\le 1$. The event $\mathcal G_k^c$ is determined by $(\mathcal F_{t_k},P^c)$, whereas the pseudo-episode length $L_k$ is generated by the independent restart randomness within pseudo-episode $k$. Hence $L_k$ is independent of $\mathcal G_k^c$ and satisfies $\mathbb E[L_k]=1/(1-\gamma)$.} Thus,
\begin{align*}
\frac{1}{1-\gamma}\sum_{k=1}^T \mathbb E\!\left[I_k^T L_k\,\mathbb I\{\mathcal G_k^c\}\right] &\le {\frac{1}{1-\gamma}\sum_{k=1}^T \mathbb E\!\left[L_k\,\mathbb I\{\mathcal G_k^c\}\right]}
{=} {\frac{1}{(1-\gamma)^2}\sum_{k=1}^T \mathbb P(\mathcal G_k^c)}
{=} {\frac{1}{(1-\gamma)^2}\sum_{k=1}^T \mathbb E\!\left[\mathbb P(\mathcal G_k^c\mid \mathcal F_{t_k})\right]}\\
&\le \frac{\delta SA}{(1-\gamma)^2}\sum_{k=1}^T\frac{\ln(2k)}{\sqrt{k}}
\le \frac{2\delta SA\,\sqrt{T}\ln(2T)}{(1-\gamma)^2}.
\end{align*}
We now bound $\mathbb E[\sum_{k=1}^{T}I_k^T\sum_{i=1}^{L_k}\Delta_{k,i}]$. For each started pseudo-episode $k$, define $\mathcal P_k^1$ as the event that, for all $(s,a)\in\mathcal S\times\mathcal A$,
\begin{align*}
\Big| (\bar P_k(\cdot\mid s,a)-P^c_{s,a})^\top {V_k} \Big| \le \frac{1}{1-\gamma} \sqrt{\frac{2\ln(2SA\,T\,m_T\,t_k^2)}{N_{t_k}(s,a)+S}}.
\end{align*}
Conditioned on ${\mathcal F_{t_k}}$, the vector {$V_k$} is deterministic and $P^c_{s,a}\mid {\mathcal F_{t_k}}\sim \Dir(\phi_k(s,a))$. Hence, by \eqref{eq:dirichlet_upper_tail_V_cited}, \eqref{eq:dirichlet_lower_tail_V_cited}, and a union bound,
\begin{align*}
\mathbb P((\mathcal P_k^1)^c\mid {\mathcal F_{t_k}})\le \frac{1}{T\,m_T\,t_k^2}.
\end{align*}
Using $\Delta_{k,i}\le (1-\gamma)^{-1}$, we split according to $\mathcal P_k^1$:
\begin{align*}
\mathbb E\!\left[ \sum_{k=1}^{T}I_k^T\sum_{i=1}^{L_k}\Delta_{k,i} \right] &\le \mathbb E\!\left[ \sum_{k=1}^{T}I_k^T\sum_{i=1}^{L_k} \Delta_{k,i}\mathbb I\{ {\mathcal P_k^1}\} \right]
+ \frac{1}{1-\gamma} \sum_{k=1}^{T} \mathbb E\!\left[ I_k^T L_k\mathbb I\{ {(\mathcal P_k^1)^c}\} \right].
\end{align*}
Since $\mathcal P_k^1\in {\mathcal F_{t_k}\vee\sigma(P^c)}$, Lemma~\ref{lem:value-decomposition-sum} gives
\begin{align*}
\mathbb E\!\left[ \sum_{k=1}^{T}I_k^T\sum_{i=1}^{L_k} \Delta_{k,i}\mathbb I\{ {\mathcal P_k^1}\} \right]= \mathbb E\!\left[ \sum_{k=1}^{T}I_k^T\sum_{i=1}^{L_k} i\,\gamma E_{k,i}\mathbb I\{ {\mathcal P_k^1}\} \right],
\end{align*}
where
$E_{k,i}:=\rho^{\alpha_k}_{\phi_k(s_{k,i},a_{k,i})}(P^\top {V_k})-(P^c_{s_{k,i},a_{k,i}})^\top {V_k}$.
Moreover, {by the same independence between $L_k$ and $\mathbb I\{(\mathcal P_k^1)^c\}$, and} since $t_k\ge k$ and $m_T\ge (1-\gamma)^{-1}$,
\begin{align*}
\frac{1}{1-\gamma} \sum_{k=1}^{T} \mathbb E\!\left[ I_k^T L_k\mathbb I\{ {(\mathcal P_k^1)^c}\} \right] \le \frac{\pi^2}{6(1-\gamma)}.
\end{align*}
Let $N'_t(s,a):=N_t(s,a)+S$, and define
\begin{align*}
\mathcal{B}_k := \left\{ N'_{t_{k+1}-1}(s,a)+1\le 2N'_{t_k}(s,a) \ \text{for all } (s,a)\in\mathcal S\times\mathcal A \right\}.
\end{align*}
For the last started pseudo-episode, $t_{K_T+1}$ is interpreted in the proof-only continuation. Since $E_{k,i}\le (1-\gamma)^{-1}$, the preceding Bellman-error sum is bounded by
\begin{align*}
&\underbrace{
\frac{1}{1-\gamma}
\sum_{k=1}^{T}I_k^T\sum_{i=1}^{L_k}
\gamma i\,\mathbb I\{L_k>m_T\}
}_{(I)}
+
\underbrace{
\frac{1}{1-\gamma}
\sum_{k=1}^{T}I_k^T\sum_{i=1}^{L_k}
\gamma i\,\mathbb I\{L_k\le m_T\}\mathbb I\{\mathcal B_k^c\}
}_{(II)}\\
&\qquad+
\underbrace{
\sum_{k=1}^{T}I_k^T\sum_{i=1}^{L_k}
\gamma i\,E_{k,i}
\mathbb I\{L_k\le m_T\}\mathbb I\{\mathcal B_k\}\mathbb I\{ {\mathcal P_k^1}\}
}_{(III)} .
\end{align*}
By Lemma~\ref{lem:mk-constant} with $K=T$, $\mathbb E[(I)]\le \mathfrak T_T$. For term $(II)$, the witness-pair argument gives
\begin{align*}
\sum_{k=1}^{T}I_k^T\mathbb I\{L_k\le m_T\}\mathbb I\{\mathcal B_k^c\} \le SA\lceil\log_2 m_T\rceil,
\end{align*}
and hence $\mathbb E[(II)]\le \frac{\gamma SA}{1-\gamma}\,m_T^2\lceil\log_2 m_T\rceil$.

It remains to bound $(III)$. On $\{L_k\le m_T\}\cap\mathcal B_k\cap {\mathcal P_k^1}$, we have $i\le m_T$. By Lemma~\ref{lem:dirichlet_quantile_BR-MDP}, the definition of $\mathcal P_k^1$, and
\begin{align*}
\ln\frac{1}{1-\alpha_k(s,a)} \le \ln\frac{1}{1-\underline{\alpha}} + \left( \ln\frac{\bar N_k^+\sqrt{k}}{\delta\,N_k^+(s,a)\ln(2k)} \right)_+,
\end{align*}
where $(x)_+:=\max\{x,0\}$, we have
\begin{align*}
E_{k,i} &\le \frac{1}{1-\gamma} \sqrt{\frac{2\ln\frac{1}{1-\underline{\alpha}}}{N_{t_k}(s_{k,i},a_{k,i})+S}}
+ \frac{1}{1-\gamma} \frac{ \sqrt{ 2\left( \ln\frac{\bar N_k^+\sqrt{k}} {\delta\,N_k^+(s_{k,i},a_{k,i})\ln(2k)} \right)_+ } }{\sqrt{N_{t_k}(s_{k,i},a_{k,i})+S}}
+ \frac{1}{1-\gamma} \sqrt{\frac{2\ln(2SA\,T\,m_T\,t_k^2)}{N_{t_k}(s_{k,i},a_{k,i})+S}} .
\end{align*}
Since $k\le T$, $t_k\le T$, $\bar N_k^+\le(\bar T+SA)/SA$, $\ln(2k)\ge\ln2$, and $N_k^+(s,a)\ge (N_k(s,a)+S)/(S+1)$,
\begin{align*}
\left( \ln\frac{\bar N_k^+\sqrt{k}}{\delta\,N_k^+(s,a)\ln(2k)} \right)_+ \le \left( \ln\frac{(S+1)(\bar T+SA)\sqrt{\bar T}} {SA\,\delta\,\ln2\,(N_k(s,a)+S)} \right)_+.
\end{align*}
Also, $\ln(2SA\,T\,m_T\,t_k^2)\le \ln(2SA\,T^3m_T)$. Therefore,
\begin{align*}
(III)
&\le
\frac{\gamma m_T}{1-\gamma}\sqrt{2\ln\frac{1}{1-\underline{\alpha}}}
\sum_{k=1}^{T}I_k^T\sum_{i=1}^{L_k}
\frac{\mathbb I\{L_k\le m_T\}\mathbb I\{\mathcal B_k\}}{\sqrt{N_{t_k}(s_{k,i},a_{k,i})+S}}\\
&\quad+
\frac{\gamma m_T}{1-\gamma}\sqrt{2}
\sum_{k=1}^{T}I_k^T\sum_{i=1}^{L_k}
\frac{
\mathbb I\{L_k\le m_T\}\mathbb I\{\mathcal B_k\}
\sqrt{\left(
\ln\frac{(S+1)(\bar T+SA)\sqrt{\bar T}}
{SA\,\delta\,\ln2\,(N_{t_k}(s_{k,i},a_{k,i})+S)}
\right)_+}
}{\sqrt{N_{t_k}(s_{k,i},a_{k,i})+S}}\\
&\quad+
\frac{\gamma m_T}{1-\gamma}\sqrt{2\ln(2SA\,T^3m_T)}
\sum_{k=1}^{T}I_k^T\sum_{i=1}^{L_k}
\frac{\mathbb I\{L_k\le m_T\}\mathbb I\{\mathcal B_k\}}{\sqrt{N_{t_k}(s_{k,i},a_{k,i})+S}} .
\end{align*}
The short augmented steps in $(III)$ consist of the first $T$ real interactions plus, only if the last pseudo-episode is short, at most $m_T$ additional proof-only steps. Hence their total number is at most $\bar T=T+m_T$. On $\mathcal B_k$, for time within pseudo-episode $k$,
\begin{align*}
N_t(s_t,a_t)+S\le 2\bigl(N_{t_k}(s_t,a_t)+S\bigr).
\end{align*}
Thus the two count-sums without logarithmic weights satisfy
\begin{align*}
\sum_{k=1}^{T}I_k^T\sum_{i=1}^{L_k} \frac{\mathbb I\{L_k\le m_T\}\mathbb I\{\mathcal B_k\}}{\sqrt{N_{t_k}(s_{k,i},a_{k,i})+S}} \le \sqrt{2}\sum_{t=1}^{\bar T} \frac{1}{\sqrt{N_t(s_t,a_t)+S}}\le \sqrt{8SA\bar T}.
\end{align*}
For the logarithmically weighted sum, using the definition of $M_T$, the same argument gives
\begin{align*}
&\sum_{k=1}^{T}I_k^T\sum_{i=1}^{L_k}
\frac{
\mathbb I\{L_k\le m_T\}\mathbb I\{\mathcal B_k\}
\sqrt{\left(
\ln\frac{(S+1)(\bar T+SA)\sqrt{\bar T}}
{SA\,\delta\,\ln2\,(N_{t_k}(s_{k,i},a_{k,i})+S)}
\right)_+}
}{\sqrt{N_{t_k}(s_{k,i},a_{k,i})+S}}\\
&\qquad\le
\sqrt{
8SA(\bar T+S^2A)
\left(
\ln\left(1+\frac{SA\,M_T}{\bar T+S^2A}\right)+2
\right)
}.
\end{align*}
Combining these bounds yields
\begin{align*}
\mathbb E[(III)] &\le \frac{\gamma m_T}{1-\gamma}\Bigg( \sqrt{16SA\bar T\ln\frac{1}{1-\underline{\alpha}}} + \sqrt{ 16SA(\bar T+S^2A) \left( \ln\left(1+\frac{SA\,M_T}{\bar T+S^2A}\right)+2 \right) }\\
&\qquad+ \sqrt{16SA\bar T\ln(2SA\,T^3m_T)} \Bigg).
\end{align*}
Combining the previous estimates proves \eqref{eq:main-order-proof-1}.

For part~(ii), define $\widetilde\Delta_{k,i}:= {V_k}(s_{k,i}) - V_{\phi_k,\underline{\alpha}}^{\pi_k}(s_{k,i}).$
By Lemma~\ref{lem:bayes_optimism}(ii), $V_{\phi_k,\underline{\alpha}}^*(s)\le {V_k}(s)$ for all $s\in\mathcal S$, so $V_{\phi_k,\underline{\alpha}}^*(s_{k,i})-V_{\phi_k,\underline{\alpha}}^{\pi_k}(s_{k,i}) \le \widetilde\Delta_{k,i}.$
Hence $ {BR\text{-}R(T)}\le
\mathbb E[
\sum_{k=1}^{T}I_k^T\sum_{i=1}^{L_k}\widetilde\Delta_{k,i}]$.

Define $\mathcal P_k^2$ as the event that, for all $(s,a)\in\mathcal S\times\mathcal A$,
\begin{align*}
\Big| (\bar P_k(\cdot\mid s,a)-P^c_{s,a})^\top V_{\phi_k,\underline{\alpha}}^{\pi_k} \Big| \le \frac{1}{1-\gamma} \sqrt{\frac{2\ln(2SA\,T\,m_T\,t_k^2)}{N_{t_k}(s,a)+S}}.
\end{align*}
Then $\mathbb P((\mathcal P_k^2)^c\mid {\mathcal F_{t_k}})\le 1/(T\,m_T\,t_k^2)$, and therefore
$\mathbb P((\mathcal P_k^1\cap\mathcal P_k^2)^c\mid {\mathcal F_{t_k}})\le 2/(T\,m_T\,t_k^2)$.

 {On $\mathcal P_k^2$, Lemma~\ref{lem:dirichlet_quantile_BR-MDP} gives, for each $i\ge1$,}
\begin{align*}
{E^-_{k,i}} {\le \frac{1}{1-\gamma} \sqrt{\frac{2\ln\frac{1}{\underline{\alpha}}}{N_{t_k}(s_{k,i},a_{k,i})+S}} + \frac{1}{1-\gamma} \sqrt{\frac{2\ln(2SA\,T\,m_T\,t_k^2)}{N_{t_k}(s_{k,i},a_{k,i})+S}}.}
\end{align*}
As in part~(i), splitting only according to $\mathcal P_k^1\cap\mathcal P_k^2$ and using Lemma~\ref{lemma:value-decomposition-sum-robust}, the same argument gives
\begin{align*}
{BR\text{-}R(T)} &\le \frac{\gamma m_T}{1-\gamma}\Bigg( \sqrt{16SA\bar T\ln\frac{1}{1-\underline{\alpha}}} + \sqrt{ 16SA(\bar T+S^2A) \left( \ln\left(1+\frac{SA\,M_T}{\bar T+S^2A}\right)+2 \right) }\\
&\qquad+ \sqrt{16SA\bar T\ln\frac{1}{\underline{\alpha}}} + 2\sqrt{16SA\bar T\ln(2SA\,T^3m_T)} \Bigg)\\
&\qquad+ \frac{2\gamma SA}{1-\gamma}\,m_T^2\left\lceil\log_2m_T\right\rceil + \frac{\pi^2}{3(1-\gamma)} + 2\mathfrak T_T.
\end{align*}
This proves \eqref{eq:main-order-proof-2}.
\end{proof}

\begin{theorem}[Restatement of Theorem \ref{thm:main-order}]
For $\delta>0$, $\underline{\alpha}\in(0,1)$, and $T\ge S^2A$,
\begin{align*}
BR(T) \le \widetilde O\!\left( \frac{\gamma\sqrt{SA\,T\,\ln\frac{e}{1-\underline{\alpha}}}}{(1-\gamma)^2} + \frac{\delta\,SA\sqrt{T}}{(1-\gamma)^2} + \frac{SA}{(1-\gamma)^3} \right),
\end{align*}
and
\begin{align*}
BR\text{-}R(T) \le \widetilde O\!\left( \frac{\gamma\sqrt{SA\,T\,\ln\frac{1}{\min\{1-\underline{\alpha},\underline{\alpha}\}}}}{(1-\gamma)^2} + \frac{SA}{(1-\gamma)^3} \right),
\end{align*}
where $\widetilde O(\cdot)$ omits polylogarithmic factors in $S,A,T,1/\delta$, and logarithms of $1/(1-\gamma)$.
In particular, if $\delta=\frac{1}{\sqrt{SA}}$ in AQ-BRMDP, then
\begin{align*}
BR(T) {\le} \widetilde O\!\left( \frac{\gamma\sqrt{SA\,T\ln\frac{e}{1-\underline{\alpha}}}}{(1-\gamma)^2} + \frac{\sqrt{SA\,T}}{(1-\gamma)^2} + \frac{SA}{(1-\gamma)^3} \right).
\end{align*}
\end{theorem}

\begin{proof}
We first prove the bounds for arbitrary $\delta>0$. By Theorem~\ref{thm:regret},
\begin{align*}
m_T=\widetilde O\!\left(\frac{1}{1-\gamma}\right), \qquad m_T^2\left\lceil\log_2m_T\right\rceil = \widetilde O\!\left(\frac{1}{(1-\gamma)^2}\right).
\end{align*}
Therefore,
\begin{align*}
\frac{\gamma SA}{1-\gamma}\, m_T^2\left\lceil\log_2m_T\right\rceil = \widetilde O\!\left(\frac{SA}{(1-\gamma)^3}\right).
\end{align*}
Also $\mathfrak T_T=\widetilde O((1-\gamma)^{-1})$, which is absorbed by $\widetilde O(SA/(1-\gamma)^3)$.

Since $\bar T=T+m_T$, the square-root terms involving $\bar T$ are bounded by the corresponding $T$-terms plus lower-order terms that are absorbed by $\widetilde O(SA/(1-\gamma)^3)$. Moreover, the logarithmic factor involving $M_T$ contributes only polylogarithmic dependence on $S,A,T,1/\delta$, and $1/(1-\gamma)$. Since $T\ge S^2A$, the term $\sqrt{SA(\bar T+S^2A)}$ is absorbed, up to lower-order terms, by $\sqrt{SA\,T}$. For the true-optimal benchmark, this pure $\sqrt{SA\,T}$ term is absorbed into $\sqrt{SA\,T\ln\frac{e}{1-\underline{\alpha}}}$ because $\ln\frac{e}{1-\underline{\alpha}}\ge 1$. Substituting these estimates into \eqref{eq:main-order-proof-1} yields
\begin{align*}
BR(T) {\le} \widetilde O\!\left( \frac{\gamma\sqrt{SA\,T\,\ln\frac{e}{1-\underline{\alpha}}}}{(1-\gamma)^2} + \frac{\delta\,SA\sqrt{T}}{(1-\gamma)^2} + \frac{SA}{(1-\gamma)^3} \right).
\end{align*}
For the robust-optimal benchmark, since $\min\{1-\underline{\alpha},\underline{\alpha}\}\le 1/2$, the pure $\sqrt{SA\,T}$ term is absorbed by
$\sqrt{SA\,T\ln\frac{1}{\min\{1-\underline{\alpha},\underline{\alpha}\}}}$. Substituting the same estimates into \eqref{eq:main-order-proof-2} gives
\begin{align*}
BR\text{-}R(T) {\le} \widetilde O\!\left( \frac{\gamma\sqrt{SA\,T\,\ln\frac{1}{\min\{1-\underline{\alpha},\underline{\alpha}\}}}}{(1-\gamma)^2} + \frac{SA}{(1-\gamma)^3} \right).
\end{align*}
Finally, if $\delta=1/\sqrt{SA}$, then
\begin{align*}
\frac{\delta SA\sqrt T}{(1-\gamma)^2} = \frac{\sqrt{SA\,T}}{(1-\gamma)^2}.
\end{align*}
Substituting this into the arbitrary-$\delta$ bound for $BR(T)$ gives the stated special case; {the $BR\text{-}R(T)$ bound does not contain the $\delta SA\sqrt T/(1-\gamma)^2$ term and is therefore unchanged up to hidden polylogarithmic factors.}
\end{proof}

\section{Implementation Details}
\label{app:implementation-details}

This appendix gives the implementation details for the experiments in Section~\ref{sec:experiments}. 
{For the finite-state experiments, the per-transition reward function $r(s,a,s')$ is known and deterministic, while the entire transition kernel is treated as unknown. When the reward depends on the next state, as in FrozenLake, expected one-step rewards are computed under the sampled posterior transition model.}
The agent maintains an independent Dirichlet posterior over each transition vector $P^c_{s,a}$ and updates the posterior parameters using the observed transition counts. 
{For the continuous-state FrozenLake experiment, the transition kernel is parameterized by some stochastic components; details are given in Appendix~\ref{app:continuous-frozenlake}.}

\subsection{Value Iteration for the Approximate $\alpha_k$-quantile BR-MDP}
\label{sec:value-iteration}

At the beginning of pseudo-episode $k$, AQ-BRMDP first updates the posterior parameters $\phi_k$ using the history data and then computes the adaptive quantile schedule $\alpha_k$.
These two quantities define the $\alpha_k$-quantile BR-MDP to be solved in the current pseudo-episode.
The exact posterior quantiles in Bellman backups are generally not available in closed form.
We therefore replace the exact posterior quantile by an empirical quantile computed from posterior samples, leading to the approximate $\alpha_k$-quantile BR-MDP used in implementation.
{For a value vector $V$, state-action pair $(s,a)$, and quantile level $\alpha_k(s,a)$, we draw posterior transition samples $P_{s,a}^1,\ldots,P_{s,a}^M\stackrel{\mathrm{i.i.d.}}{\sim}\Dir(\phi_k(s,a))$ and compute the sampled one-step Bellman targets $Z_j=\sum_{s'\in\mathcal S}P_{s,a}^j(s')\bigl[r(s,a,s')+\gamma V(s')\bigr]$. We then sort $Z_{(1)}\le\cdots\le Z_{(M)}$ and use $Z_{(\lceil M\alpha_k(s,a)\rceil)}$ as the empirical posterior quantile.}

\paragraph{Monte Carlo budget for quantile estimation.}
The posterior-sampling budget used in each Bellman backup controls the Monte Carlo error of the empirical quantile estimator.
To choose this budget across different quantile levels, we use the classical asymptotic normal approximation for sample quantiles.
If $\widehat q_\alpha$ is the empirical $\alpha$-quantile computed from $m$ independent samples of a scalar random variable with density $g$ positive at its $\alpha$-quantile $q_\alpha$, then the leading variance term is proportional to
\begin{align*}
\frac{\alpha(1-\alpha)}{m\,g(q_\alpha)^2}.
\end{align*}
Since the exact density of the sampled Bellman target depends on both the posterior parameter and the current value vector, we use the standard normal distribution as a reference distribution to choose how the sample size varies across quantile levels.
This choice allocates more posterior samples to tail quantiles.

For AQ-BRMDP, at pseudo-episode $k$ and state-action pair $(s,a)$, let
$q_{k,s,a}:=\Phi^{-1}\!\bigl(\alpha_k(s,a)\bigr),$
and let $\varphi_{\mathrm{std}}$ denote the standard normal density.
The number of posterior samples used to estimate the quantile from $(s,a)$ is chosen as
\begin{equation}
\label{eq:app-nsa} { n_{k,s,a} = \min\left\{ 2048,\, \left\lceil c_{n_{\mathrm{samples}}}\, \frac{\alpha_k(s,a)\bigl(1-\alpha_k(s,a)\bigr)} { \varphi_{\mathrm{std}}(q_{k,s,a})^2 } \right\rceil \right\}.}
\end{equation}
For the fixed-level baselines BR-MDP-$\alpha$ with $\alpha\in\{0.1,0.3,0.5\}$, we use the same rule with $\alpha_k(s,a)$ replaced by the corresponding fixed value $\alpha$.

\begin{algorithm}[htbp]
\caption{Value Iteration for the Approximate $\alpha_k$-quantile BR-MDP}
\label{alg:BR-MDP-vi-k-approx}
\small
\begin{algorithmic}[1]
\State \textbf{Input:} Posterior parameters $\phi_k$, approximate optimal value in pseudo-episode $k-1$ $\widehat V_{k-1}$, quantile schedule $\alpha_k$, per-transition reward function $r$, discount factor $\gamma$, Monte Carlo budget coefficient $c_{n_{\mathrm{samples}}}$, tolerance $\varepsilon_{\mathrm{VI}}$, maximum number of iterations $M_{\mathrm{VI}}$

\ForAll{$(s,a)\in\mathcal S\times\mathcal A$}
    \State {Compute $n_{k,s,a}$ from \eqref{eq:app-nsa}} and {set $\ell_{k,s,a}\gets \lceil n_{k,s,a}\alpha_k(s,a)\rceil$}
    \State {Draw and fix posterior samples $P^1_{s,a},\ldots,P^{n_{k,s,a}}_{s,a}\stackrel{\mathrm{i.i.d.}}{\sim}\Dir(\phi_k(s,a))$}
\EndFor

\State Initialize $V^{(0)}\leftarrow\widehat V_{k-1}$
\For{$m=0,1,\ldots,M_{\mathrm{VI}}-1$}
    \ForAll{$(s,a)\in\mathcal S\times\mathcal A$}
        \State {Using the fixed posterior samples, compute}
        \begingroup
        \setlength{\abovedisplayskip}{1pt plus 1pt minus 1pt}
        \setlength{\belowdisplayskip}{1pt plus 1pt minus 1pt}
        \setlength{\abovedisplayshortskip}{0pt plus 1pt}
        \setlength{\belowdisplayshortskip}{1pt plus 1pt minus 1pt}
        \begin{align*}
Z_j^{(m)}(s,a) \gets \sum_{s'\in\mathcal S} P_{s,a}^j(s') \left[ r(s,a,s')+\gamma V^{(m)}(s') \right], \qquad j=1,\ldots,n_{k,s,a}.
        \end{align*}
        \endgroup
        \State Sort $Z_{(1)}^{(m)}(s,a)\le\cdots\le Z_{(n_{k,s,a})}^{(m)}(s,a)$
        \State Set $Q_k^{(m+1)}(s,a)\gets Z_{(\ell_{k,s,a})}^{(m)}(s,a)$
    \EndFor
    \State Set $V^{(m+1)}(s)\gets \max_{a\in\mathcal A}Q_k^{(m+1)}(s,a)$ for all $s\in\mathcal S$
    \If{$\|V^{(m+1)}-V^{(m)}\|_\infty\le \varepsilon_{\mathrm{VI}}$}
        \State \textbf{break}
    \EndIf
\EndFor
\State Set $\widehat V_k\gets V^{(m+1)}$
\State For each $s\in\mathcal S$, choose $\widehat\pi_k(s)\in\arg\max_{a\in\mathcal A}{Q_k^{(m+1)}(s,a)}$
\State \Return $\widehat V_k$ and $\widehat\pi_k$
\end{algorithmic}
\end{algorithm}

{The posterior transition samples in Algorithm~\ref{alg:BR-MDP-vi-k-approx} are drawn once at the beginning of pseudo-episode $k$ and held fixed throughout the value-iteration loop. Hence the empirical-quantile Bellman operator is deterministic within the pseudo-episode, and the stopping criterion is applied to a fixed sample-average approximation of the $\alpha_k$-quantile BR-MDP.}
The fixed-level baselines BRMDP-0.1, BRMDP-0.3, and BRMDP-0.5 use Algorithm~\ref{alg:BR-MDP-vi-k-approx} with $\alpha_k(s,a)\equiv 0.1$, $0.3$, and $0.5$, respectively.

\subsection{Implementations of Continuing PSRL}
Continuing PSRL uses the same pseudo-episode mechanism as AQ-BRMDP. At the beginning of pseudo-episode $k$, it samples one transition kernel $\widetilde P_k$ from the current posterior by drawing $\widetilde P_{k,s,a}\sim\Dir(\phi_k(s,a))$ independently for all $(s,a)$. It then solves the sampled discounted MDP using standard value iteration and executes a greedy policy $\widetilde\pi_k$ throughout pseudo-episode $k$. 
{When rewards are transition-dependent, the one-step reward in the sampled MDP is evaluated as $\sum_{s'}\widetilde P_{k,s,a}(s')r(s,a,s')$, so Continuing PSRL also avoids using the true expected reward under $P^c$ during learning.}

\subsection{Finite-State Experiment Schematics}
\label{app:experiment-schematics}

Figure~\ref{fig:finite-state-schematics} provides schematic diagrams for the two finite-state experiment families used in Section~\ref{sec:experiments}.

\begin{figure}[htbp]
\centering
\begin{minipage}[t]{0.62\textwidth}
\centering
\includegraphics[width=\linewidth]{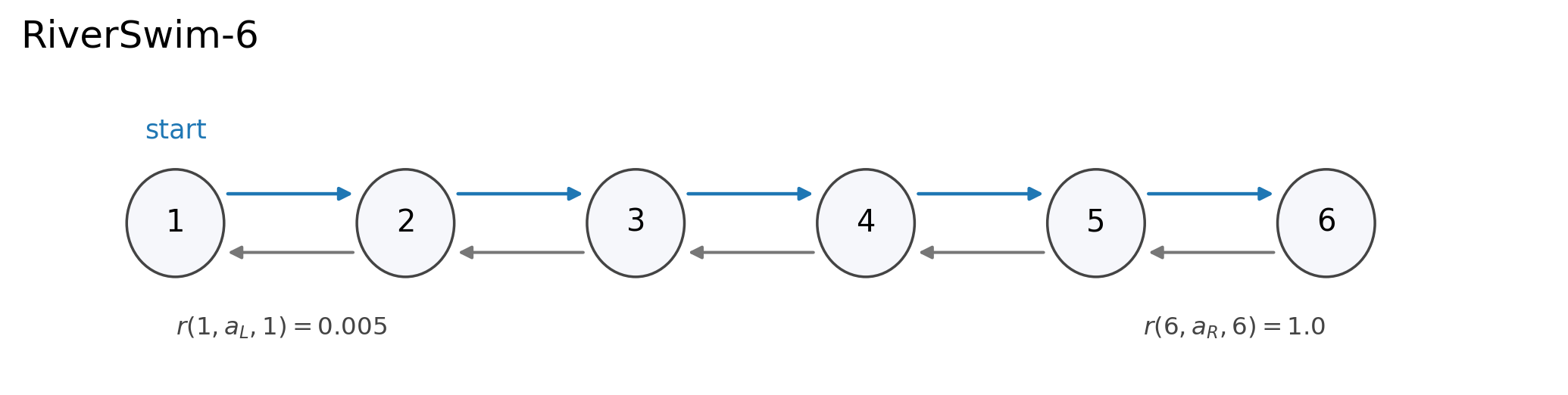}
\end{minipage}
\hfill
\begin{minipage}[t]{0.28\textwidth}
\centering
\includegraphics[width=\linewidth]{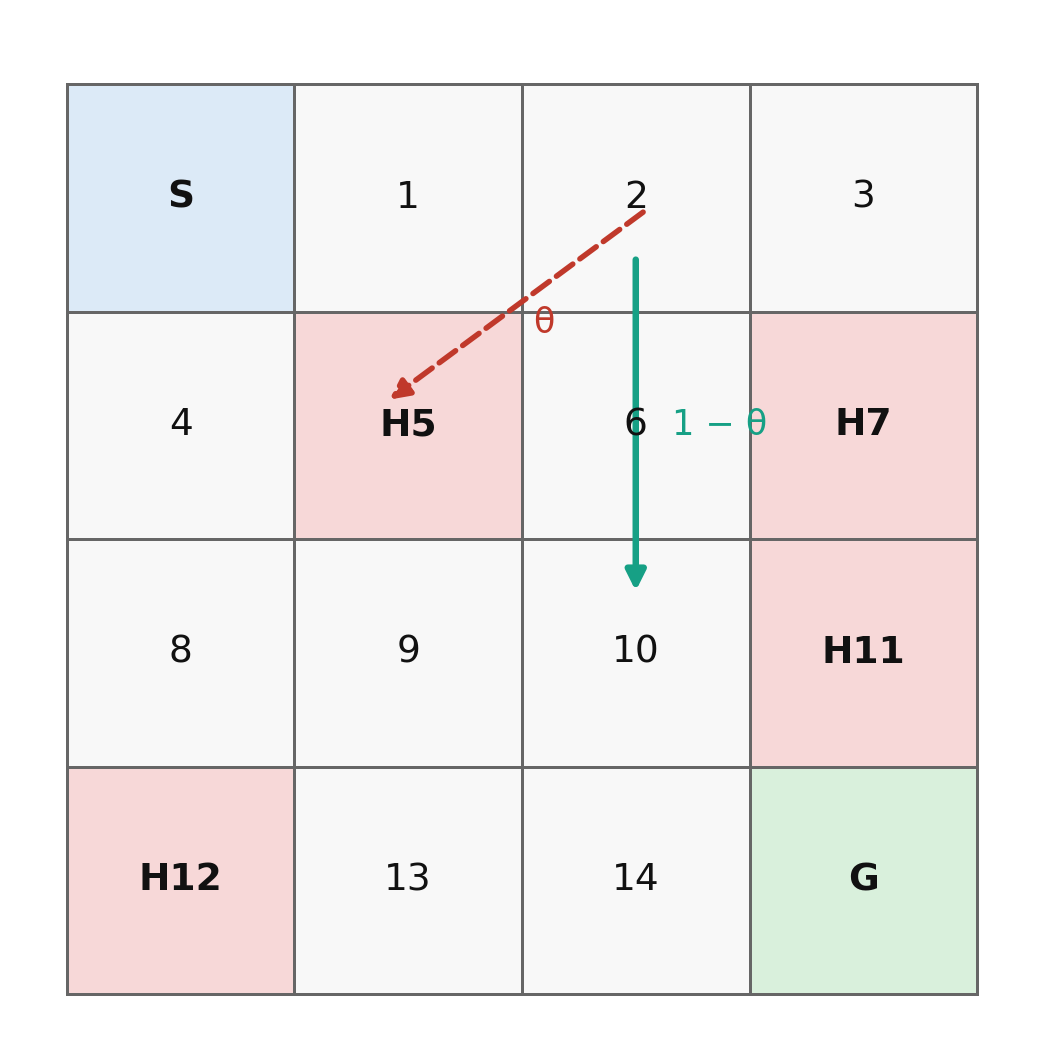}
\end{minipage}
\caption[Finite-state experiment schematics.]{Finite-state experiment schematics: RiverSwim-6 (left) and risky-branch FrozenLake (right). In the FrozenLake schematic, the shortcut at $(2,\mathrm{Down})$ moves the agent to state $10$ with probability $1-\theta$ and to the sticky hole at state $5$ with probability $\theta$.}
\label{fig:finite-state-schematics}
\end{figure}

\subsection{Experiment-Specific Settings and Evaluation Details}
\label{app:experiment-specific-details}

{For AQ-BRMDP, the floor parameter in the adaptive schedule is set to $\underline{\alpha}=0.2$ in all finite-state experiments unless otherwise stated.}
{In all tabular BR-MDP planning calls, we use value-iteration tolerance $\varepsilon_{\mathrm{VI}}=10^{-8}$ and maximum iteration count $M_{\mathrm{VI}}=10{,}000$. The same stopping rule and iteration cap are used for Continuing PSRL, fixed-quantile BR-MDP, and AQ-BRMDP planning.}
{The finite-state RiverSwim and risky-branch FrozenLake experiments were run on a MacBook using CPU execution without GPU acceleration. These experiments were substantially less computationally demanding than the continuous-state FrozenLake experiment reported below.}

\paragraph{RiverSwim.}
The discount factor is $\gamma = 0.9$, and the total interaction horizon is $T = 4500$. We set the constant in \eqref{eq:app-nsa} and \eqref{eq:alphak} to $(c_{n_{\mathrm{samples}}},\delta)=(150, 5)$ and $(200,10)$ respectively for RiverSwim-6 and RiverSwim-10.
The plots for RiverSwim-6 and RiverSwim-10 are truncated to the first $2000$ and $4000$ time steps, respectively. The same truncation windows are used for the corresponding occupancy heatmaps.

\paragraph{Risky-branch FrozenLake.}
In risky-branch FrozenLake experiments, the discount factor $\gamma=0.8$, and total interaction horizon $T=4500$. We set the constant in \eqref{eq:app-nsa} and \eqref{eq:alphak} to $(c_{n_{\mathrm{samples}}},\delta)=(250, 10)$.

\subsection{Estimation of the Posterior $\alpha$-Quantile Value}
\label{app:posterior-lb-estimation}

We estimate the posterior $\alpha$-quantile value by evaluating the current policy under posterior samples of the transition kernel. For each independent run, each algorithm, and diagnostic time $t$, we take the current posterior parameter $\phi_t$ and the current policy $\pi_t$. We then draw $M=147$ independent transition kernels by drawing $P^{(m)}_{s,a}\sim\Dir(\phi_t(s,a))$, $m=1,\ldots,M$, independently for all $(s,a)$. For each sampled transition kernel $P^{(m)}$, we compute the value of $\pi_t$ by exact policy evaluation. In particular, letting $P^{(m),\pi_t}(s,s'):=P^{(m)}_{s,\pi_t(s)}(s')$ and $r^{(m),\pi_t}(s):=\sum_{s'\in\mathcal S}P^{(m)}_{s,\pi_t(s)}(s')r(s,\pi_t(s),s')$, we compute $V^{\pi_t}_{P^{(m)}}=(I-\gamma P^{(m),\pi_t})^{-1}r^{(m),\pi_t}$. We then form the $M$ sampled values at the initial state, $Y_m:=V^{\pi_t}_{P^{(m)}}(s_0)$ for $m=1,\ldots,M$, and let $Y_{(1)}\le\cdots\le Y_{(M)}$ denote their order statistics. The empirical estimate of the posterior $0.1$-quantile value is
\begin{equation}
\label{eq:app-posterior-lb-estimator} \widehat V^{\pi_t,\mathrm q}_{\phi_t,0.1}(s_0) := Y_{(\lceil 0.1M\rceil)}.
\end{equation}
We set $M=147$ according to \eqref{eq:app-nsa} with $c_{n_{\mathrm{samples}}}=50$. This metric is computed every $200$ time steps. For each diagnostic time, we average the estimates over $100$ independent runs, and the $95\%$ confidence bands are computed across these independent runs.

\subsection{Sensitivity Analysis for the Schedule Parameter}
\label{app:sensitivity-analysis}

{We examine sensitivity to the schedule parameter $\delta$ on standard FrozenLake without the risky branch, holding the remaining settings fixed as described in Section~\ref{sec:exp-costly}. We compare $\delta\in\{5,10,15\}$ over $100$ independent runs per setting. 
These small differences in Figure~\ref{fig:delta-sensitivity} and Table~\ref{tab:delta-sensitivity} suggest that AQ-BRMDP is not materially sensitive to $\delta$ over the range tested.}

\begin{figure}[htbp]
\centering
\includegraphics[width=0.42\textwidth]{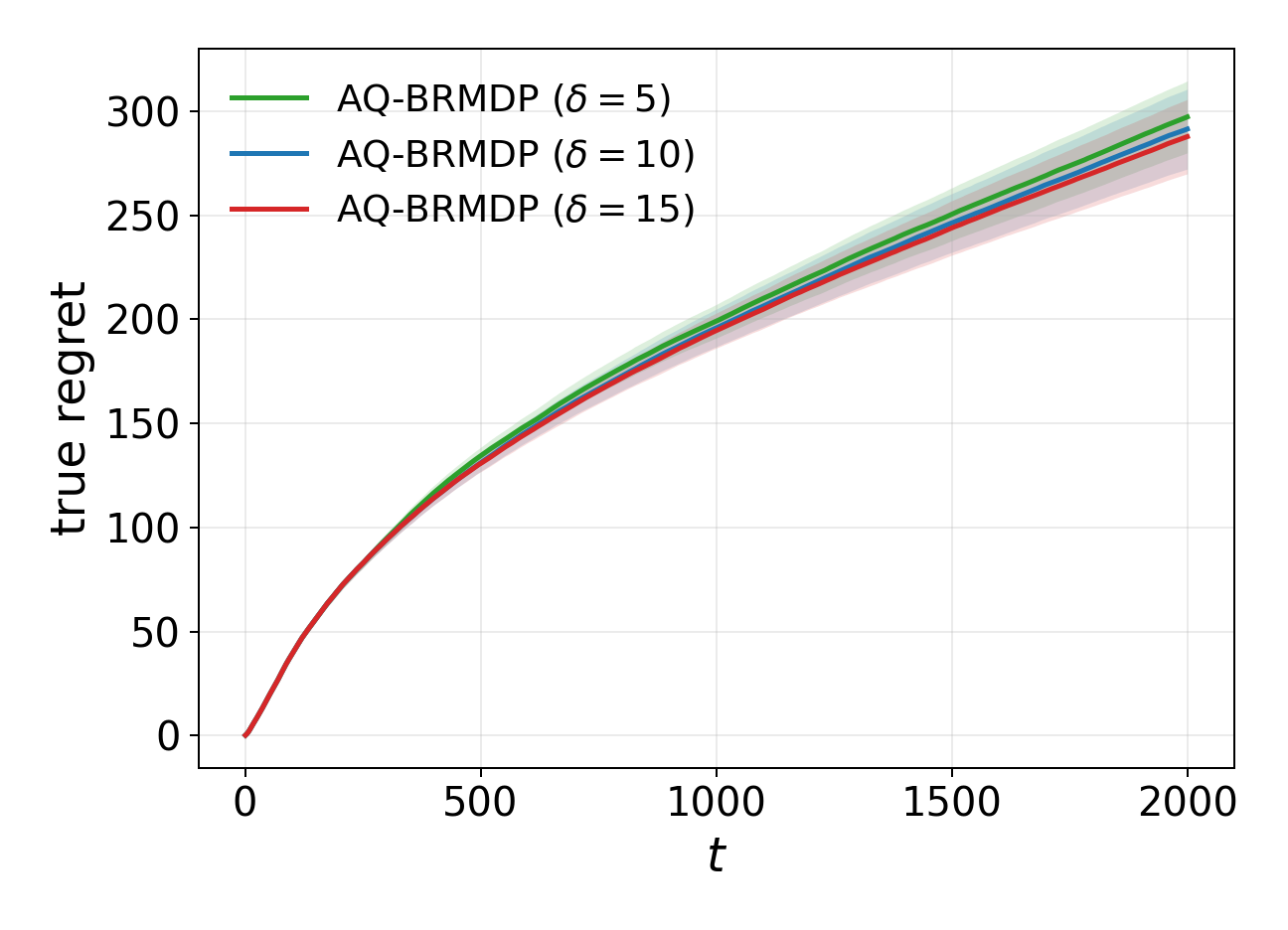}
\caption[Sensitivity of AQ-BRMDP to the schedule parameter $\delta$.]{Sensitivity of AQ-BRMDP to the schedule parameter $\delta$ on FrozenLake without the risky branch. The plot shows cumulative true regret.}
\label{fig:delta-sensitivity}
\end{figure}

\begin{table}[htbp]
    \small
\centering
\caption[Final cumulative true regret in the sensitivity study.]{Final cumulative true regret in the sensitivity study. Each entry reports the mean over $100$ runs. Percentages are reductions relative to the $\delta=5$ setting.}
\label{tab:delta-sensitivity}
\begin{tabular}{lcc}
\hline
$\delta$ & Final cumulative regret & Reduction \\
\hline
$5$  & $297.2396$ & $0.00\%$ \\
$10$ & $291.4018$ & $1.96\%$ \\
$15$ & $287.7603$ & $3.19\%$ \\
\hline
\end{tabular}
\end{table}

\subsection{Continuous-State FrozenLake Experiment}
\label{app:continuous-frozenlake}

{We further evaluate AQ-BRMDP in a continuous-state version of FrozenLake. In this experiment, the state space is continuous, the posterior is placed on a parametric transition model, and the resulting Bellman backups are approximated using neural fitted Q iteration.}

\paragraph{Environment.}
{The state space is $[0,4]^2$, partitioned into the same $4\times4$ cells as the discrete FrozenLake layout. The start state is $(0.5,0.5)$, the hole cells are $\{5,7,11,12\}$, and the goal cell is $15$. The action space remains $\{\mathrm{Left},\mathrm{Right},\mathrm{Up},\mathrm{Down}\}$. A transition receives reward $1$ if the next state enters the goal cell and reward $0$ otherwise. From the goal cell, the next state is sampled from the known uniform reset distribution over the non-hole, non-goal cells.}

{For ordinary cells, the realized movement direction is the intended direction with probability $0.50$ and one of the two perpendicular directions with probability $0.25$ each. For hole cells, the agent moves in the intended direction with probability $p_h=0.20$ and otherwise remains in place. The continuous movement length is $L=\ell_0+U$, where $\ell_0=0.75$ and $U\sim\mathrm{Unif}(0,\theta_L)$ with true value $\theta_L=0.50$. If a candidate next state leaves $[0,4]^2$, the agent remains in the current state.}

\paragraph{Posterior model.}
{The unknown transition parameters are $(p_{\mathrm{slip}},p_h,\theta_L)$. We use the conjugate priors $p_{\mathrm{slip}}\sim\Dir(1,1,1)$, $p_h\sim\mathrm{Beta}(1,1)$, and $\theta_L\sim\mathrm{Pareto}(a_0,x_0)$ with $a_0=2$ and $x_0=0.25$. The simulator records latent primitives: $M_t$ for the ordinary-cell movement mode, $E_t$ for the hole mobility indicator, and $L_t$ for the attempted movement length when a movement is activated. The Dirichlet posterior is updated from counts of $M_t$, the Beta posterior is updated from counts of $E_t$, and the Pareto posterior is updated from observations of $U_t=L_t-\ell_0$. Goal reset transitions do not update the posterior.}

\paragraph{Neural fitted Q approximation.}
{The Q-function is represented by a neural network $Q_\omega(s,a)$ with architecture $20\to64\to64\to4$, where the four outputs correspond to the four actions. The $20$ input features consist of the normalized two-dimensional coordinates, a $16$-dimensional cell one-hot vector, and two hole/goal indicators. At the start of each pseudo-episode, the network is warm-started from the previous pseudo-episode; it is not reinitialized.}

For a sampled posterior model $\theta^b$ and a Q-function $Q$, define the posterior-model Bellman target $G^b(s,a;Q) := \mathbb E_{s'\sim P_{\theta^b}(\cdot\mid s,a)} \left[R(s,a,s')+\gamma\max_{a'}Q(s',a')\right]$.

{For AQ-BRMDP, the Bellman target is the empirical $\alpha_k(s,a)$-quantile of $G^b(s,a;Q)$ over posterior samples. For PSRL, a single posterior model is sampled and the risk-neutral Bellman target under that sampled model is used. The expectations in these targets are estimated by Monte Carlo simulation from the corresponding model.}

\paragraph{Adaptive quantile schedule in the continuous setting.}
{The normalization set consists of the 15 cell centers excluding the goal center $(3.5,3.5)$. For each queried $(s,a)$, the posterior predictive uncertainty $u_k(s,a)$ is computed as the standard deviation of $G^b(s,a;Q_{\omega_{k-1}})$ across posterior samples, and $c_k(s,a)=1/(u_k(s,a)+\varepsilon)$. Let $\bar c_k$ be the average of $c_k(\tilde s,a)$ over the non-goal center states and actions. For non-goal states, $\alpha_k(s,a) = \mathrm{clip}\left(1-\delta\frac{c_k(s,a)}{\bar c_k}g_k,\,\underline{\alpha},\,0.95\right)$, and $g_k=\frac{\log(2k)}{\sqrt{k}}$.}

{For goal states, all actions are equivalent and we set $\alpha_k(s,a)=\mathrm{clip}(1-\delta g_k,\underline{\alpha},0.95)$.}
{The posterior sample budget used to estimate each quantile is selected by the same rule as in \eqref{eq:app-nsa} and capped at $800$ posterior samples.}

\begin{algorithm}[htbp]
\caption{Continuous-State AQ-BRMDP with Neural Fitted Q Approximation}
\label{alg:continuous-aq-brmdp}
\footnotesize
\begin{algorithmic}[1]
\State \textbf{Input:} initial posterior, initial Q-network parameter $\omega_0$, schedule parameters, and replay buffer $\mathcal B\leftarrow\emptyset$
\State Initialize pseudo-episode index $k\leftarrow0$ and restart indicator $X_1\leftarrow0$
\For{$t=1,\ldots,T$}
    \If{$X_t=0$}
        \State $k\leftarrow k+1$
        \State {Update the posterior using the observed latent primitives}
        \State Draw posterior model samples $\{\theta_k^b\}$ {from the current posterior}
        \State Compute $\bar c_k$ over non-goal center states and fix {the rule defining} $\alpha_k(s,a)$ for this pseudo-episode
        \State $\omega_k\leftarrow \mathrm{FitQ}(\omega_{k-1},\alpha_k,\{\theta_k^b\},\mathcal B)$ using fitted-Q regression
    \EndIf
    \State Take action $a_t\in\arg\max_{a\in\mathcal A}Q_{\omega_k}(s_t,a)$
    \State Observe $s_{t+1}$, reward $r_t$, and latent primitives when applicable
    \State Append the transition and latent primitives to $\mathcal B$
    \State Sample $X_{t+1}\sim\mathrm{Bernoulli}(\gamma)$ if $t<T$
\EndFor
\end{algorithmic}
\end{algorithm}

{The fitted-Q subroutine initializes at $\omega_{k-1}$ and repeats a standard fitted-Q loop: sample training states, construct clipped Monte Carlo Bellman targets for each state-action pair, and regress $Q_\omega$ onto those targets. For AQ-BRMDP, the target is the empirical $\alpha_k(s,a)$-quantile over posterior models; for PSRL, it is the risk-neutral target under one sampled model.}

{The continuous-state experiment compares AQ-BRMDP with a continuous-state PSRL baseline using the same pseudo-episode mechanism, posterior model, Q-network architecture, and fitted-Q solver.}

\begin{figure}[htbp]
\centering
\includegraphics[width=0.52\textwidth]{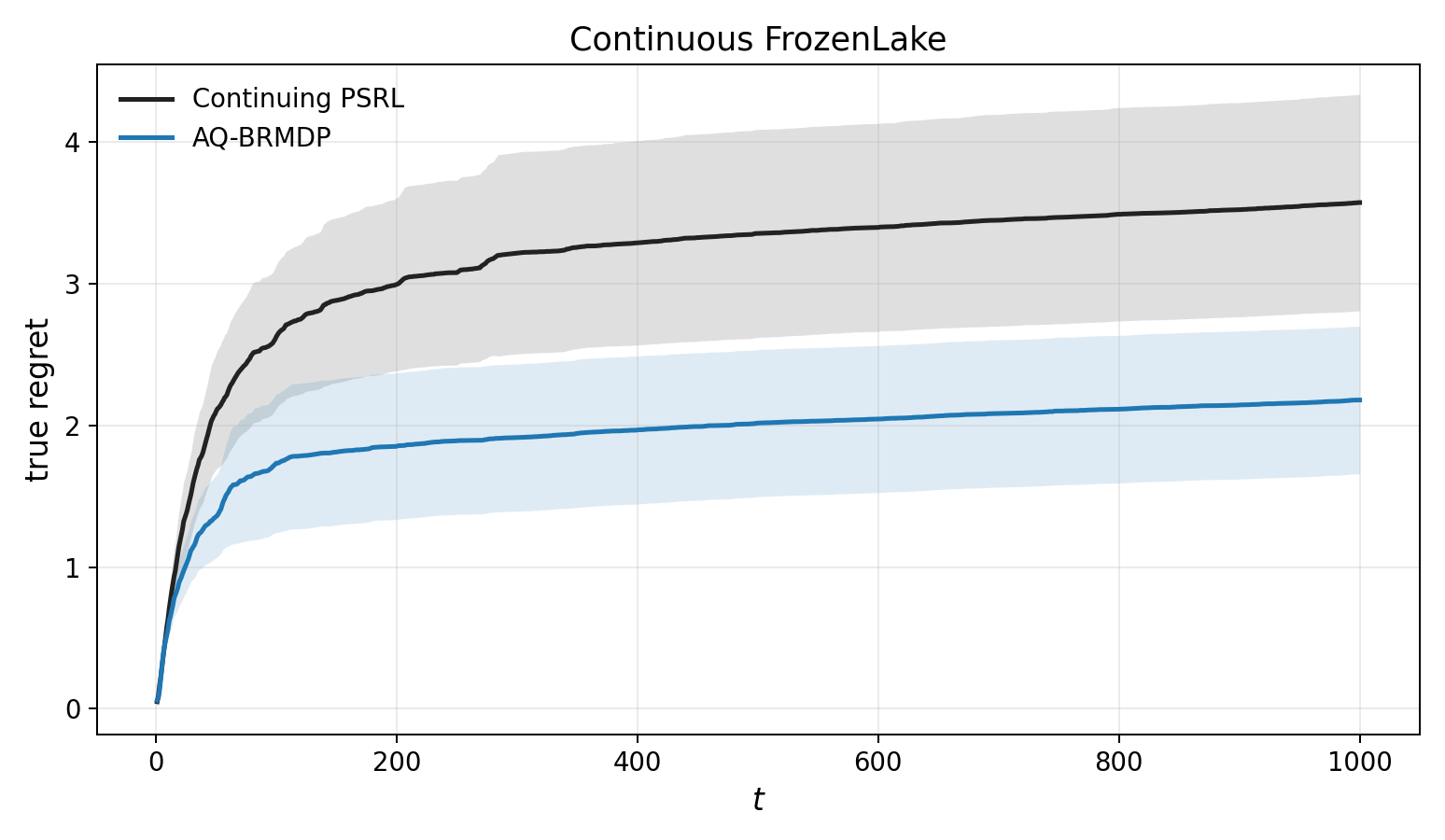}
\caption{{Cumulative true regret in continuous-state FrozenLake.}}
\label{fig:continuous-fl-true-regret}
\end{figure}

{The continuous-state results show that AQ-BRMDP can be implemented with a parametric posterior model and neural fitted-Q approximation in a continuous-state space. Figure~\ref{fig:continuous-fl-true-regret} shows that AQ-BRMDP outperforms the continuous-state PSRL baseline with smaller cumulative true regret while retaining the adaptive quantile mechanism used in the finite-state experiments.}

{The continuous-state FrozenLake experiments were run on a single NVIDIA GeForce RTX 4090 GPU with 24 GB GPU memory. The mean runtime per independent run was $21.2$ minutes for AQ-BRMDP and $5.5$ minutes for Continuing PSRL, with the difference mainly due to posterior quantile estimation.}
\end{document}